%% file: main.tex
\documentclass[runningheads]{llncs}

\usepackage{tcolorbox}
\usepackage{graphicx}
\usepackage{svg}
\usepackage{amsmath}
\DeclareMathOperator*{\argmax}{arg\,max}
\usepackage{booktabs}
\usepackage[table]{xcolor}
\usepackage{multirow}
\usepackage{adjustbox}
\usepackage{wrapfig}
\usepackage{algorithm}
\usepackage{algorithmic}
\usepackage{pifont}
\usepackage{enumitem}
\usepackage{booktabs}
\usepackage{xspace}
\usepackage{colortbl}
\usepackage{makecell}
\usepackage[skip=3pt]{caption}
\newcommand{\cmark}{\ding{51}}
\newcommand{\xmark}{\ding{55}}
\usepackage[accsupp]{axessibility}  
\newcommand{\method}{\textbf{VERDICT}\xspace}

\usepackage[pagebackref,breaklinks,colorlinks]{hyperref}
\usepackage{cleveref}
\usepackage{orcidlink}

\begin{document}

\title{\method: Training-Free Step-Wise Verification of Multimodal Reasoning via Disagreement-Aware Consensus} 

\titlerunning{VERDICT - Training-Free Step-Wise Verification}

\author{Rohit Sinha*\inst{1} \and
Kunal Tilaganji*\inst{1} \and
 Tanuja Ganu \inst{2} \and 
 Nagarajan Natarajan\inst{2} \and
 Amit Sharma\inst{2} \and
Vineeth Balasubramanian\inst{1,2}}

\authorrunning{Sinha et al.}

\institute{Indian Institute of Technology, Hyderabad\\ 
\url{https://www.iith.ac.in/}\and
Microsoft Research\\ 
 \url{https://www.microsoft.com/en-us/research/lab/microsoft-research-india/}}

\maketitle

\begin{abstract}

Multimodal large language models often generate reasoning chains containing subtle errors that lead to incorrect answers. Current verification approaches have notable limitations. Existing approaches either require expensive labelled supervision with inconsistent cross-task performance or aggregate scores from multiple sources by simple aggregations, missing a key insight: when these scores disagree, that disagreement itself carries important information about whether a reasoning step is truly valid or not. We formalise this as a coupled scoring problem among disparate, frozen verifiers, interpretable as a coordination game with a unique closed-form equilibrium where agreement signals valid steps while disagreement reveals instability. Towards this end, we propose a training-free domain-agnostic step-wise verification approach we call \method: \textbf{VER}ification via \textbf{D}isagreement-\textbf{I}nformed \textbf{C}oupled \textbf{T}hresholding. To our knowledge, \method is the first training-free verifier that makes the structure of cross-modal disagreement explicit and actionable. It computes consensus scores through a closed-form solution, enabling both disagreement-aware filtering and stability-conscious ranking of reasoning steps. Evaluated across six benchmarks, \method consistently improves over the base model by up to +5.95\%, and performs competitively with domain-specific critics that demand extensive supervision, demonstrating that cross-modal agreement provides robust verification signals without task-specific adaptation.
\end{abstract}

\input{sections/1_2_introduction}

\input{sections/2_2_related_works}

\input{sections/3_preliminaries}

\input{sections/4_Method}
\input{sections/5_6_Exp_Results}
\input{sections/7_2_Ablation}

\input{sections/8_9_2_discuss_concl}

\clearpage

\textbf{Table of Contents}
\begin{itemize}
  \item S1. Hyperparameters
  \item S2. Implementation Details: VERDICT
  \item S3. Detailed Experimental Setup
  \item S4. Additional Results
  \begin{itemize}
    \item S4.1 Comparison with Test-Time Scaling Methods
    \item S4.2 Cross-Benchmark Dispersion Tolerance Sensitivity
    \item S4.3 Cross-Benchmark Confidence Threshold Analysis
    \item S4.4 Threshold Sensitivity: Extended Analysis
    \item S4.5 Joint Threshold Interaction Analysis
    \item S4.6 Cross-Benchmark Judge Scale Analysis
    \item S4.7 Cross-Benchmark Stubbornness Sensitivity
    \item S4.8 Cross-Benchmark Stubbornness Assignment Analysis
    \item S4.9 Dispersion as a Diagnostic Signal
    \item S4.10 Confidence--Dispersion Landscape
    \item S4.11 Rejection and Selection: Detailed Analysis
    \item S4.12 Fallback Mechanism Analysis
    \item S4.13 Generalization Across Model Families
  \end{itemize}
  \item S5 Detailed Computational Cost Analysis
  \item S6 Closed-Form Derivation
  \begin{itemize}
    \item S6.1 Matrix Construction
    \item S6.2 Invertibility: Formal Proof
    \item S6.3 Scalar Reduction and Explicit Solution
    \item S6.4 Stubbornness-Weighted Sum Invariant
    \item S6.5 Extension to General $m$
    \item S6.6 Numerical Verification
  \end{itemize}
  \item S7 Existence and Uniqueness of the Consensus Fixed Point
  \item S8 Proposition 1: Detailed Proof and Worked Examples
  \item S9 Consensus Dispersion vs.\ Weighted Averaging: Illustrative Scenarios
  \item S10 Connections to Social Choice Theory and Belief Aggregation
  \item S11 Qualitative Samples
  \item S12 Prompt Templates
\end{itemize}

\clearpage

\input{suppl_sections/takeaways}
\input{suppl_sections/hyperparameters}
\input{suppl_sections/app_implementation}
\input{suppl_sections/app_experimental_setup}

\input{suppl_sections/additional_results/additional_result_v2}

\input{suppl_sections/compute_cost_v3}
\input{suppl_sections/closed_form_derivation}
\input{suppl_sections/game_proof}

\input{suppl_sections/prop_proof_v2}
\input{suppl_sections/nash_not_avg}

\input{suppl_sections/social_choice_v2}


\input{suppl_sections/qualitative_samples_v2}
\clearpage
\input{suppl_sections/prompt_template}
\clearpage

%

\bibliographystyle{splncs04}
\bibliography{main}
\end{document}

%% file: sections/1_2_introduction.tex
\section{Introduction}
\vspace{-7pt}

Multimodal large language models (MLLMs)~\cite{bai2025qwen3vltechnicalreport,NEURIPS2023_6dcf277e} have demonstrated impressive multi-step reasoning capabilities over images and text. By decomposing complex questions into chains of intermediate steps, these models can tackle compositional visual reasoning tasks that require both perceptual grounding and logical inference ~\cite{zhang2024multimodalchainofthoughtreasoninglanguage,chen2024m3cotnovelbenchmarkmultidomain}. Yet their reasoning chains often contain subtle errors, unsupported visual claims, logical gaps, and hallucinated spatial relationships that are difficult to detect precisely because each step \emph{appears} locally plausible~\cite{li2023evaluating,liu2024hallucination}. These errors propagate: a single flawed step early in the chain can cascade through subsequent reasoning, ultimately leading to an incorrect final answer. This motivates \emph{step-level} verification, catching errors where they arise rather than after they have shaped the trajectory of the reasoning chain.

\begin{figure}[!t]
    \centering
    \includegraphics[width=\linewidth]{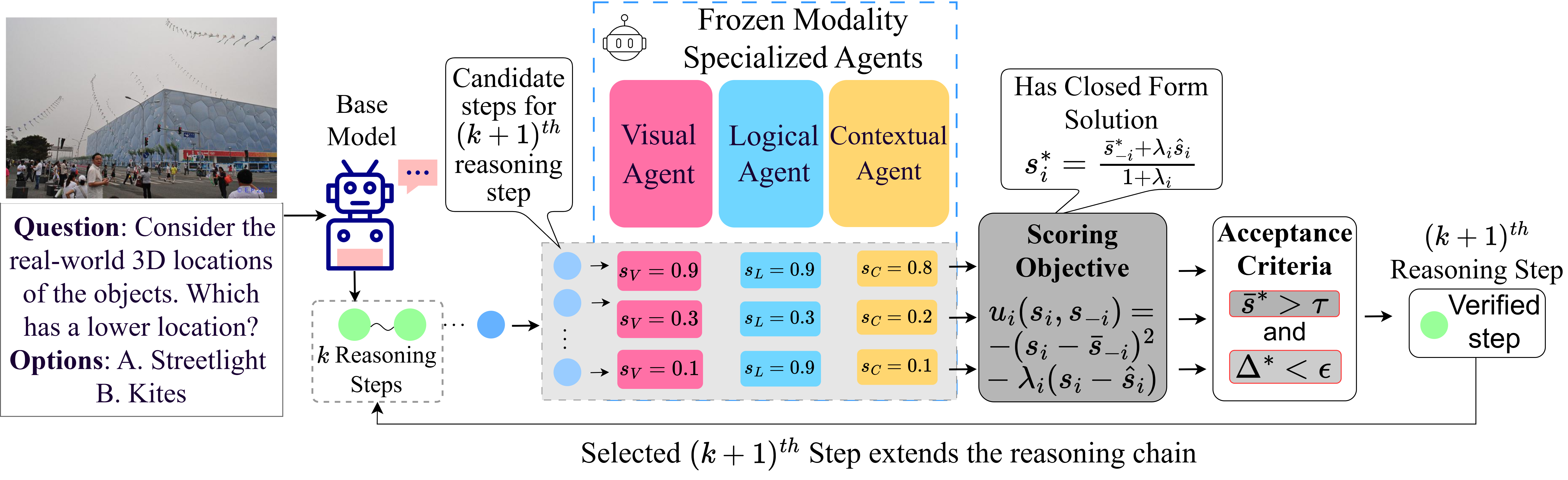}
    \caption{An illustrative overview of \method.Given a multimodal question and partial reasoning trace, the base model generates candidate continuations that are independently scored by three frozen, modality-specialized agents. A closed-form consensus computation transforms the raw scores into disagreement-aware consensus scores; a dual acceptance criterion then filters by mean confidence and consensus dispersion, selecting the most stable step to extend the reasoning chain.}
    \label{fig:methods_illustration}
\end{figure}

\noindent
Existing verification approaches fall into two broad categories, each with notable limitations. \textbf{Domain Specific critics}, including process reward models (PRMs)~\cite{lightman2023prm800k,wang2023mathshepherd,luo2024improvemathematicalreasoninglanguage} and their multimodal extensions such as VisualPRM~\cite{visualprm}, Sherlock~\cite{sherlock}, LLaVA-Critic~\cite{wang2025llavacriticr1criticmodelsecretly}, and MM-Verify~\cite{sun2025mmverify}, score individual reasoning steps using trained models. However, these require carefully labeled supervision, often obtained through expensive Monte Carlo rollouts or human annotations~\cite{luo2024improvemathematicalreasoninglanguage,wang2024visualprm}. More critically, they show inconsistent cross-task performance: while they may improve accuracy on one benchmark, they frequently degrade on others. This limitation is especially acute in data-scarce domains, where the labeled supervision these critics require is prohibitively expensive or simply unavailable.\textbf {Training-free aggregation methods} like Weaver~\cite{weaver} and reward model ensembles~\cite{helpingherdingrewardmodel,coste2024rewardmodelensembleshelp} combine scores from multiple sources, but they miss a critical insight. Consider two scoring patterns: unanimous confidence $(0.7, 0.7, 0.7)$ and conflicting evidence $(0.9, 0.3, 0.9)$. Simple averaging treats both as equivalent, yet they reflect fundamentally different verification states. This problem is compounded by well-documented capability tensions within MLLMs themselves: visual grounding can disrupt reasoning~\cite{kang2025vgentvisualgroundingmodular,zhang2023llavagrounding}, visual instruction tuning degrades language performance~\cite{shen2024multimodal}, and extended reasoning chains cause forgetting of earlier context~\cite{tian2025thoughtaccuracydualnature}. These tensions mean that different evaluation perspectives can legitimately disagree about the same reasoning step and that disagreement carries diagnostic value.

\noindent
Our key insight is that when a reasoning step is truly valid, disparate judges evaluating it from different perspectives, such as visual grounding, logical consistency, and contextual relevance, should be able to converge. When they \emph{cannot} agree, even after accounting for each other's views, the step is likely unstable. Rather than treating verification as a classification or regression task, we treat it as a \emph{coupled scoring problem} among disparate sources of evidence. We formalize this as a coupled scoring framework among frozen, modality-specialized judges, whose consensus yields disagreement-aware scores. The consensus captures how each judge would revise their belief in light of what others think, producing scores that naturally balance collective confidence against residual disagreement. Unlike variance-based approaches (which ignore belief strength) or simple averaging (which buries disagreement), \method is a disagreement-aware consensus framework, which makes the structure of conflict explicit and actionable.

\noindent
Concretely, three frozen agents, Visual, Logical, and Contextual, disparately score each candidate's reasoning step. A closed-form consensus computation then transforms these raw scores into consensus-adjusted scores, without any iterative optimization or learning. A dual acceptance criterion combines a mean confidence check (ensuring sufficient collective endorsement) with a dispersion check (requiring cross-modal agreement). Among accepted candidates, ranking by consensus confidence selects the most stable continuation. The entire framework operates as a plug-in verifier requiring no modification to the base model, no training data, and no task-specific adaptation. This formulation admits a natural game-theoretic interpretation: each agent is a player in a coordination game, and the closed-form solution corresponds to its unique Nash equilibrium \cite{nash1950equilibrium}, grounding the influence structure in equilibrium theory rather than ad hoc design choices.

\noindent
Evaluated across six benchmarks spanning spatial reasoning, visual grounding, and multimodal abstraction, \method achieves consistent improvements of up to $+5.95\%$ over baseline models and performs competitively with or surpasses domain-specific critics that require extensive labeled training data. Crucially, our method never degrades below baseline on any tested benchmark, a property that none of the domain-specific critics we evaluate can claim, as they show significant performance variability across tasks.

\noindent
It is worth noting that Averaging~\cite{coste2024rewardmodelensembleshelp,wangmath} is symmetric as every judge's deviation counts equally, and the pattern of who disagrees with whom is lost. Variance-based filtering~\cite{NIPS2017_9ef2ed4b} treats a stubborn visual agent and a flexible contextual agent identically. Our disagreement-aware consensus formulation introduces \emph{coupling} where each agent's adjusted score depends on
all others, so disagreement from a stubborn agent weighs more heavily: an asymmetry no fixed weighting can recover. This coupling is what distinguishes \method. Our contributions are summarized below:

\vspace{-5pt}
\begin{itemize}

    \item We formalize step-wise verification as a coupled scoring problem among disparate, frozen verifiers, establishing disagreement structure as a first-class verification signal and proposing disagreement-aware verification as a design principle.

    \item We introduce \method, a consensus verifier with a closed-form solution (the unique equilibrium of a coordination game among modality-specialized judges) that models cross-modal agreement, enabling both disagreement-aware filtering and stability-conscious ranking of candidate reasoning steps.

    \item We provide a domain-agnostic, training-free, plug-in implementation that integrates directly into MLLM reasoning without base model modification and demonstrates consistent improvements across six benchmarks spanning spatial reasoning, visual grounding, and multimodal abstraction.

    \item We conduct comprehensive ablation studies revealing that the framework operates through two complementary mechanisms, rejection and selection, and that consensus scores function optimally as ranking tools rather than binary classifiers.

\end{itemize}
\vspace{-15pt}

%% file: sections/2_2_related_works.tex
\section{Related Work}
\vspace{-7pt}

\begin{figure}[!t]
    \centering
    \includegraphics[width=\linewidth]{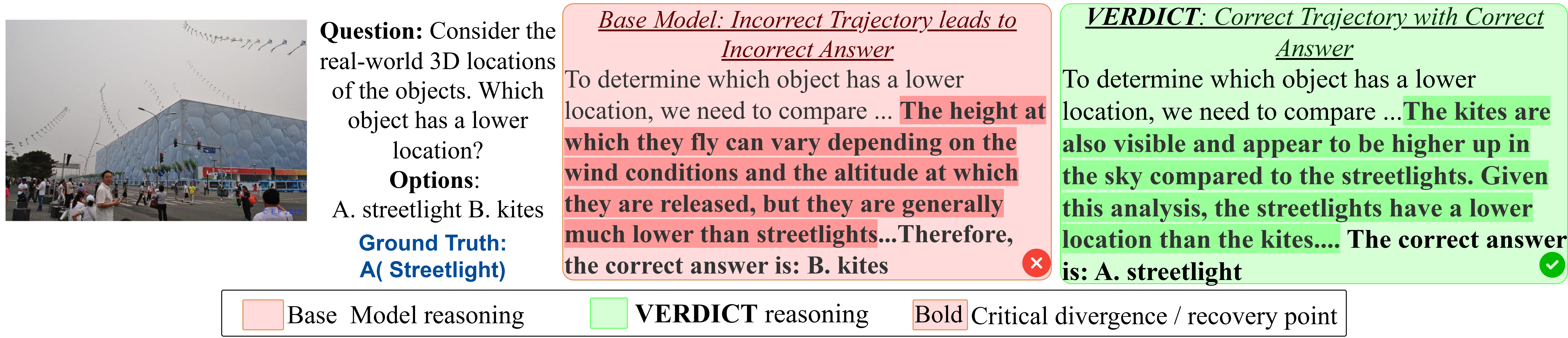}
    \caption{Qualitative examples showing base model reasoning (orange) and 
    VERDICT-verified reasoning (blue). Bold highlighted text marks the critical 
    sentence where the base model commits to an incorrect reasoning trajectory. 
    VERDICT recovers by selecting an alternative candidate step that corrects 
    the judgment, yielding the correct answer. Full reasoning chains are provided 
    in the supplementary material.}
    \label{fig:quali_main}
\vspace{-15pt}
\end{figure}

Step-level verification has largely developed along two tracks. Process reward 
models from PRM800K~\cite{lightman2023prm800k} through 
Math-Shepherd~\cite{wang2023mathshepherd} and OmegaPRM~\cite{luo2024improvemathematicalreasoninglanguage} 
to multimodal extensions like MM-PRM~\cite{mmprm} and 
VisualPRM~\cite{visualprm} demonstrate that intermediate supervision improves 
reasoning reliability, but at the cost of expensive human annotation or Monte 
Carlo rollouts that our approach requires none of. Domain-specific critic 
models~\cite{sherlock,criticv,sun2025mmverify,wang2025llavacriticr1criticmodelsecretly,li2025selfrewardingvisionlanguagemodelreasoning} achieve strong results on specific benchmarks yet implicitly collapse 
heterogeneous evidence into a single score, which, as our experiments show, 
leads to fragile cross-task generalization; a critic that learns to reward one 
reasoning style can quietly penalize another, and the scalar it produces carries 
no record of which modality drove the judgment.

\begin{wraptable}{r}{0.65\linewidth}
\centering
\vspace{-10pt}
\setlength{\tabcolsep}{4pt}
\renewcommand{\arraystretch}{1.2}
\scriptsize
\caption{\scriptsize VERDICT vs.\ prior methods across three design axes (Section~3).}
\begin{tabular}{lccc}
\toprule
\textbf{Method} 
    & \makecell{\textbf{Training}\\\textbf{Free}} 
    & \makecell{\textbf{Disagree.}\\\textbf{Aware}} 
    & \makecell{\textbf{Plug-in}\\\textbf{Compat.}} \\
\midrule
Domain Specific Critics  
    & \xmark & \xmark & \xmark \\
    ~\cite{lightman2023prm800k,visualprm,sherlock,wang2025llavacriticr1criticmodelsecretly,cao2025dreamprm} & & & \\
Reward Ensembles~\cite{helpingherdingrewardmodel}
    & \cmark & \xmark & \cmark \\
Weaver~\cite{weaver}                            
    & \cmark & \xmark & \cmark \\
Variance Filtering~\cite{NIPS2017_9ef2ed4b}     
    & \cmark & \xmark & \cmark \\
\midrule
\rowcolor{green!10}
\textbf{VERDICT (Ours)}                         
    & \cmark & \cmark & \cmark \\
\bottomrule
\end{tabular}

\vspace{-8pt}
\label{tab:positioning}
\end{wraptable}

\noindent This opacity is not incidental; 
it is structural, and it is what motivates an explicit treatment of disagreement 
rather than a better-trained aggregator. Ensemble and multi-agent 
methods~\cite{helpingherdingrewardmodel,weaver,du2023improvingfactualityreasoninglanguage,NIPS2017_9ef2ed4b} 
take a step toward robustness by combining multiple sources of evidence, but 
aggregate through averaging or majority voting, treating all patterns of 
agreement and disagreement as equivalent; a three-way split and a near-unanimous 
vote are indistinguishable once reduced to a mean. The richer question- not 
whether agents disagree, but how and between whom is precisely what averaging 
discards and what our framework is designed to recover. Multi-agent formulations, 
prominent in GANs and adversarial training~\cite{roughgarden2016twenty}, have so 
far relied on adversarial zero-sum games where agents work against one another by 
design. Our formulation departs from all of these: we employ a coupled scoring 
framework in which frozen, modality-specialized judges seek consensus rather than 
oppose one another, and where the closed-form solution makes the structure of 
their disagreement explicit rather than averaging it away: a property that, to 
our knowledge, no existing training-free step-wise verification method provides. Table~\ref{tab:positioning} summarizes how VERDICT relates to prior 
work along the three axes that motivate its design.

%% file: sections/3_preliminaries.tex
\vspace{-15pt}
\section{Preliminaries}
\vspace{-15pt}

In this section, we formalize the verification setting and define the problem addressed by our approach.

\noindent\textbf{Verification Setting.}
We consider a setting where a base MLLM generates a reasoning chain step by step, and each step must be evaluated before extending the chain further. Formally, at reasoning step $t$, the base model is provided with an image $I$, a question $Q$, and a partial reasoning trace $r_{1:t-1}$ consisting of all previously accepted steps. From this context, the base model generates $n$ candidate continuations $\{r_t^{(1)}, \ldots, r_t^{(n)}\}$ via sampling. A verifier then evaluates each candidate and returns a binary decision: \emph{accept} the step and continue reasoning, or \emph{reject} it and resample. Crucially, evaluation is \emph{local} to the current step; it does not require access to future reasoning or to the final answer. This locality allows errors to be intercepted at the point where they arise, before they propagate through subsequent steps.

\noindent\textbf{Problem Definition.}
Given $n$ candidate steps at position $t$, $m$ specialized judge agents 
$\{A_1, \ldots, A_m\}$, and context $(I, Q, r_{1:t-1})$, the verification 
problem is to produce a single verified continuation $r_t^*$. Each agent $A_i$ 
independently scores every candidate $r_t^{(j)}$ with a scalar confidence 
$\hat{s}_i^{(j)} \in [0, 1]$ reflecting its modality-specific assessment. The 
verifier must aggregate these heterogeneous assessments into a selection decision, 
accounting for both collective confidence and inter-agent agreement.

\noindent
We impose three design constraints: (1)~\emph{training-free}, requiring no 
labeled data or parameter updates; (2)~\emph{disagreement-aware}, treating 
cross-modal conflict as a diagnostic signal rather than noise; and 
(3)~\emph{plug-in compatible}, integrating with any base MLLM without 
architectural modification. Domain-specific critics violate~(1); standard 
aggregation methods violate~(2). Constraint~(2) deserves emphasis: variance-based 
rejection treats all agents symmetrically, contributing equally regardless of 
their modality-specific reliability, and any two score vectors with identical 
means and per-agent deviations are deemed equivalent even when their disagreement 
patterns differ qualitatively. What is needed is a mechanism that \emph{couples} 
agents one in which each agent's adjusted assessment depends on where every 
other agent landed. As we show in Section~\ref{section:method} 
(Proposition~\ref{prop:nash_vs_wa}), this coupling cannot be replicated by any 
separable function of individual scores, and it is precisely what \method 
provides through a coordination-game formulation.

%% file: sections/4_Method.tex
\section{\method: Disagreement-Aware Consensus Verification}
\label{section:method}
\vspace{-7pt}
We now present \method, our verification framework. We first describe the verifier agents and their roles, then formulate their interaction as a coupled scoring problem, derive the closed-form consensus , and finally define the acceptance criterion used to select verified reasoning steps.

\noindent\textbf{Verifier Agents.}
Verification is carried out by three frozen MLLMs, each prompted to judge a candidate's reasoning step from a distinct modality-specific perspective. We instantiate three agents:
\begin{itemize}
    \item \textbf{Visual Agent (V):} assesses whether objects and spatial relationships mentioned in the reasoning step are visually verifiable in the provided image. It checks whether the step is grounded in what can actually be observed.
    \item \textbf{Logical Agent (L):} evaluates whether the reasoning step follows logically from the preceding steps and makes progress toward answering the question. It checks for coherence, valid inferences, and logical gaps.
    \item \textbf{Contextual Agent (C):} determines whether the step maintains focus on the original question and avoids introducing irrelevant or tangential information. It penalizes unnecessary speculation and off-topic claims.
\end{itemize}

\noindent
Each agent receives the image $I$, question $Q$, the partial reasoning trace $r_{1:t-1}$ (assumed correct), and the current candidate step $r_t^{(j)}$ to evaluate. It outputs a single scalar score $\hat{s}_i \in [0,1]$, interpreted as its subjective confidence that the step is valid given its modality-specific evidence. Agents operate in \emph{complete isolation}: they never see other agents' scores, their prompts contain no information about other agents' judgments, and no fine-tuning or calibration is performed.

\noindent
\textit{Number of agents:}
The formulation is defined for arbitrary~$m$ agents: Eqs.~\eqref{eq:payoff}-\eqref{eq:acceptance} and Proposition~\cref{prop:nash_vs_wa} hold for any~$m \geq 2$. We instantiate three because they map onto three largely orthogonal MLLM failure modes: unsupported visual claims~\cite{kang2025vgentvisualgroundingmodular,zhang2023llavagrounding}, logical gaps~\cite{tian2025thoughtaccuracydualnature}, and contextual drift~\cite{shen2024multimodal}. The coupled system remains closed-form for larger~$m$, and its cost is dominated by the~$m$ scoring calls rather than the~$m \times m$ linear solve. The binding constraint for adding agents is semantic, not computational: each must bring a genuinely disparate  evaluation perspective.

\noindent\textbf{Coupled Scoring Formulation.}
Rather than aggregating raw scores directly, we frame verification as a coupled scoring problem among the agents. The motivating intuition is straightforward: if a reasoning step is truly valid, disparate  judges evaluating it from different perspectives should be able to reach agreement. If they \emph{cannot} agree even after accounting for each other's views, the step is likely unstable.

\noindent
We model this interaction as a coupled system in which each agent $i$ selects a reported score $s_i \in [0,1]$, balancing agreement with the other agents against fidelity to its own modality-specific judgment. The  scoring objective to agent $i$ is defined as:
\begin{equation}
    u_i(s_i, s_{-i}) = -(s_i - \bar{s}_{-i})^2 - \lambda_i(s_i - \hat{s}_i)^2
    \label{eq:payoff}
\end{equation}
where $\bar{s}_{-i} = \frac{1}{m-1}\sum_{j \neq i} s_j$ denotes the mean reported score of the other agents, $\hat{s}_i$ is agent $i$'s raw confidence score, and $\lambda_i > 0$ controls how strongly agent $i$ weights its own evidence relative to group consensus. The first term encourages agreement: it penalizes the agent for deviating from what others report. The second term enforces self-consistency: it penalizes deviation from the agent's original belief, scaled by $\lambda_i$.
Under the coordination game interpretation, Eq. \cref{eq:payoff} is each agent's payoff: it is maximized when the agent finds a score that balances fidelity to its own signal against coordination with others, which is the defining tension of a coordination game.
\noindent
The stubbornness parameter $\lambda_i$ encodes how resistant each agent should be to consensus pressure. A higher $\lambda_i$ means agent $i$ trusts its own modality-specific evidence more strongly and will deviate less from its raw score even when others disagree. In our experiments, we set $\lambda_V = 1.5$, $\lambda_L = 1.0$, and $\lambda_C = 0.8$, reflecting the intuition that the visual agent should be most resistant to consensus pressure on perception-heavy tasks, while the contextual agent may be more flexible when visual or logical evidence is strong. These values are fixed across all datasets. Since the objective function in Eq.~\eqref{eq:payoff} is strictly concave in each agent's own score (the second derivative $\frac{\partial^2 u_i}{\partial s_i^2} = -2(1 + \lambda_i) < 0$ for all $\lambda_i > 0$), the coupled system admits a unique fixed point (equivalent to the unique Nash equilibrium of the induced coordination game). This follows from Rosen's uniqueness result for concave interaction systems,  ~\cite{rosen1965existence}.

\noindent\textbf{Closed-Form Solution.}
At the consensus solution, each agent's reported score satisfies the following:
\begin{equation}
    s_i^* = \frac{\bar{s}_{-i}^* + \lambda_i \hat{s}_i}{1 + \lambda_i}
    \label{eq:equilibrium}
\end{equation}

This system of equations can be rewritten as a linear system:
\begin{equation}
    (1 + \lambda_i)\, s_i^* - \frac{1}{m-1} \sum_{j \neq i} s_j^* = \lambda_i \hat{s}_i, \quad i = 1, \ldots, m
    \label{eq:linear_system}
\end{equation}
which admits a closed-form solution and can be solved directly from the raw scores $\{\hat{s}_i\}$ without any iterative optimization, learning, or approximation.

\noindent
A key property of the consensus solution is that it \emph{preserves the mean while dampening disagreement}: the mean consensus score equals the mean raw score, but individual scores converge as disagreement diminishes proportionally to inter-agent conflict. This property is important because it allows us to separate two distinct failure modes. The first is \emph{collective doubt}: steps where all agents assign low confidence, resulting in a low mean score. The second is \emph{conflicting evidence}: steps where the mean confidence is high but agents fundamentally disagree, producing high dispersion. Simple averaging conflates these two cases; for instance, it treats $(0.6, 0.6, 0.6)$ and $(0.9, 0.3, 0.6)$ identically since both yield the same mean. However, the consensus dispersion reveals that the second case reflects unresolved cross-modal instability, while the first reflects genuine (if modest) consensus.

\noindent
The consensus solution adjusts individual scores but not their collective
level: $\bar{s}^* = \frac{1}{m}\sum_i s_i^*
= \frac{1}{m}\sum_i \hat{s}_i$, a direct consequence of the
symmetry in Eq.~\eqref{eq:linear_system} that holds for any
$\{\lambda_i\}$. The coupled system cannot inflate confidence.
Its effect is confined to redistributing scores around a fixed mean,
isolating cross-modal conflict from collective doubt.

\noindent\textbf{Acceptance Criterion and Step Selection.}
Once consensus scores $\{s_i^*\}$ are computed for each candidate, we derive two summary statistics: the \emph{mean consensus confidence} $\bar{s}^* = \frac{1}{m}\sum_i s_i^*$, measuring collective endorsement of the step, and the \emph{consensus dispersion} $\Delta^* = \frac{1}{m}\sum_i |s_i^* - \bar{s}^*|$, measuring residual disagreement after consensus adjustment.

\noindent
A candidate step is accepted if and only if it satisfies a dual criterion: the mean confidence must exceed a threshold $\tau$ \emph{and} the dispersion must fall below a tolerance $\epsilon$:
\begin{equation}
    \text{accept}(r_t^{(j)}) \iff \bar{s}^{*(j)} > \tau \;\;\wedge\;\; \Delta^{*(j)} < \epsilon
    \label{eq:acceptance}
\end{equation}

\noindent 
This dual criterion enables two complementary verification mechanisms. The dispersion check ($\Delta^* < \epsilon$) filters candidates where cross-modal evidence fundamentally conflicts, while the confidence check ($\bar{s}^* > \tau$) ensures sufficient collective endorsement. Among accepted candidates, the step with the highest $\bar{s}^*$ is selected to extend the reasoning trace, prioritizing the most stable continuation. When no candidate satisfies both criteria, indicating that all options are suboptimal, the system falls back to continuous ranking by $\bar{s}^* - \Delta^*$, selecting the candidate that best balances confidence against remaining disagreement. This fallback ensures that the reasoning process can always continue while still preferring more stable steps.

\noindent\textbf{Numerical Illustration.}
Consider raw scores $\hat{\mathbf{s}} = (0.9,\, 0.3,\, 0.9)$ from the 
Visual, Logical, and Contextual agents with $\lambda_V{=}1.5$,
$\lambda_L{=}1.0$, $\lambda_C{=}0.8$. The linear system in
Eq.~\eqref{eq:linear_system} yields $\mathbf{s}^* \approx (0.80,\,
0.55,\, 0.77)$: the outlying Logical agent is pulled from $0.30$ to
$0.55$, the stubborn Visual agent shifts only $0.10$, and the flexible
Contextual agent shifts $0.13$: asymmetry driven directly by the
$\lambda_i$ coupling in Eq.~\eqref{eq:equilibrium}. The resulting
$\bar{s}^* \approx 0.71$ and $\Delta^* \approx 0.11$: the confidence
check passes ($0.71 > \tau{=}0.6$) but dispersion fails
($0.11 > \epsilon{=}0.1$), So the step is rejected as cross-modally
unstable. Contrast $\hat{\mathbf{s}} = (0.7,\, 0.7,\, 0.7)$: consensus
returns $\Delta^* = 0$ and the step is accepted: same raw mean,
opposite decision, a distinction simple averaging cannot make.

\noindent\textbf{Where the consensus formulation adds value.}\;
By construction, $\bar{s}^*$ equals the raw mean: the consensus does not
shift collective confidence. Its contribution lies entirely in the individual equilibrium scores $\{s_i^*\}$ and the resulting dispersion~$\Delta^*$. One might argue that since Eq.~\eqref{eq:equilibrium} admits a closed-form solution, one could define the same asymmetric dispersion directly, without motivating the coupled system. This is correct computationally, but it misses what the formulation provides. Defining a weighted dispersion requires choosing how much each agent's deviation should count: a design decision with no principled grounding. The coupled scoring framework  \emph{derives} these weights from a single interpretable assumption: each agent trades off agreement against
fidelity to its own evidence, governed by stubbornness $\lambda_i$. This is precisely the assumption that defines a coordination game. The fixed point uniquely determines the rest.

\noindent
Crucially, the consensus solution introduces \emph{coupling} that no per-agent measure can replicate. Because agent~$i$'s adjusted score $s_i^*$ depends on all other agents' raw scores through the linear system in Eq.~\eqref{eq:linear_system}, its residual $|s_i^* - \bar{s}^*|$ reflects not just its own deviation but how that deviation interacts with where other agents land. We formalize this below.
\vspace{-2pt}
\begin{proposition}[Consensus Dispersion is Not Recoverable from Weighted Averages]
\label{prop:nash_vs_wa}
For any fixed weight vector $\mathbf{w}$, there exist score
vectors $\hat{\mathbf{s}} \neq \hat{\mathbf{s}}'$ such that
$\mathbf{w}^\top \hat{\mathbf{s}} = \mathbf{w}^\top \hat{\mathbf{s}}'$
yet $\Delta^*(\hat{\mathbf{s}}) \neq \Delta^*(\hat{\mathbf{s}}')$.
As for any agent~$i$, the consensus residual $|s_i^* - \bar{s}^*|$ depends on
the full raw score vector $\hat{\mathbf{s}} = (\hat{s}_1, \ldots, \hat{s}_m)$, not solely on~$\hat{s}_i$. Consequently, the consensus dispersion
$\Delta^* = \frac{1}{m}\sum_{i}|s_i^* - \bar{s}^*|$ cannot be decomposed as a
separable function $\sum_i f_i(\hat{s}_i)$ for any choice of per agent functions~$\{f_i\}$. 
\end{proposition}

\vspace{-8pt}

\begin{proof}
From Eq.~\eqref{eq:equilibrium}, $s_i^*$ depends on $\bar{s}_{-i}^*$, which itself depends on all agents' raw scores through the coupled linear system in Eq.~\eqref{eq:linear_system}. Consequently, $\partial s_i^* / \partial \hat{s}_j \neq 0$ for $i \neq j$, so the residual $|s_i^* - \bar{s}^*|$ is not a function of $\hat{s}_i$ alone and $\Delta^*$ is not separable. As a concrete illustration: $\hat{\mathbf{s}} = (0.9, 0.2, 0.9)$ and $\hat{\mathbf{s}}' = (0.7, 0.6, 0.7)$ share the same mean $\tfrac{2}{3}$, yet the consensus solution yields $\Delta^*(\hat{\mathbf{s}}) \approx 0.13 > \epsilon$ and $\Delta^*(\hat{\mathbf{s}}') \approx 0.04 < \epsilon$, producing opposite acceptance decisions under Eq.~\eqref{eq:acceptance}, a distinction no separable measure can replicate, since the difference arises from the consensus solution coupling agent~2's residual to agents~1 and~3's raw scores.
\end{proof}

\vspace{-3pt}

\noindent
\textbf{\textit{Two properties follow.}} First, disagreement from a high-$\lambda$ agent
weighs more heavily in~$\Delta^*$ than equivalent disagreement from a
flexible one, an asymmetry that raw variance treats identically. Second, this
coupling is what lets the dual criterion in Eq.~\eqref{eq:acceptance}
distinguish genuine cross-modal conflict from ordinary confidence fluctuations,
even when two score vectors share identical means and per-agent deviations.


\noindent
\textbf{Game-theoretic interpretation:} The coupled scoring formulation admits a natural reading as a coordination game in which each agent is a player, its reported score is a strategy, and .~\cref{eq:payoff} defines the payoff. The fixed point in ~\cref{eq:equilibrium} is then the unique Nash equilibrium, guaranteed by Rosen's theorem for concave $n$-person games~\cite{rosen1965existence}. 
This connection is more than cosmetic: it supplies a principled rationale for why the stubbornness parameters $\lambda_i$ govern the influence structure, as each agent trades off fidelity to its own evidence against pressure to coordinate exactly the tension a coordination game formalizes. The analogy has limits: our agents score in isolation and never observe one another; the equilibrium is computed in closed form from one-shot scores rather than reached through iterated play, and should be understood as interpretive scaffolding that motivates the formulation rather than a claim about strategic interaction.

%% file: sections/5_6_Exp_Results.tex
\section{Experiments and Results}
\input{sections/main_table}
\label{sec:results}
\vspace{-7pt}
\noindent\textbf{Experimental Setup.}
We evaluate our verification framework on six benchmarks that collectively assess different aspects of visual and multimodal reasoning. \textbf{3DSRBench}~\cite{3dsrbench}, 
\textbf{CV-Bench}~\cite{cvbench} split into CV-Bench-2D (spatial relationships and object counting) and CV-Bench-3D (depth ordering and relative distance). 
\textbf{BLINK}~\cite{blink}, 
\textbf{MMStar}~\cite{mmstar}, 
\textbf{AI2D}~\cite{ai2d}.

\noindent
Our approach employs three verification agents: Visual~(V), Logical~(L), and Contextual~(C), all built on Qwen2.5-VL-7B-Instruct~\cite{bai2025qwen3vltechnicalreport}, the same model used as the base reasoner. At each reasoning step, the base model generates $n{=}3$ candidate continuations via temperature sampling ($T{=}0.8$, top-$p{=}0.6$). Each agent disparately scores every candidate, and consensus scores are computed using agent-specific stubbornness parameters ($\lambda_V{=}1.5$, $\lambda_L{=}1.0$, $\lambda_C{=}0.8$). A candidate is accepted if its consensus dispersion satisfies $\Delta^* < \epsilon{=}0.1$ and mean confidence satisfies $\bar{s}^* > \tau{=}0.6$. Among accepted steps, we select the one with the highest $\bar{s}^*$. These hyperparameters are fixed across all six benchmarks with no task-specific tuning. The final answer is extracted from the complete reasoning trace using (Gemma-3:12B via Ollama). The use of Ollama was only for answer extraction, and is not for any part of the verification pipeline. All agent prompts are provided in the supplementary. 


\noindent\textbf{Baselines.}
We compare \method against two categories of baselines. The first category consists of \emph{domain-specific critics}(so called because each requires labeled data from the target domain to train): LLaVA-Critic~\cite{wang2025llavacriticr1criticmodelsecretly}, 
Sherlock~\cite{sherlock}, 
VisionSR1~\cite{li2025selfrewardingvisionlanguagemodelreasoning}, 
DreamPRM \cite{cao2025dreamprm},
Critic-V \cite{zhang2025criticvvlmcriticshelp}
These represent the current state of the art in step-level verification, but each requires task-specific training data or learned reward signals. All domain-specific critics are evaluated under comparable conditions. The second category consists of \emph{domain-agnostic baselines}(which require no domain-specific training data, instead relying on modality-specialized agents with task-general rubrics) that use the same three frozen agents as our method but differ in how they combine the raw scores. These include: \textbf{Mean} (select the candidate with the highest average raw score), \textbf{Max} (select by highest maximum score across agents), \textbf{Min} (select by highest minimum score), \textbf{Majority} (accept candidates where a majority of agents score above 0.5, then rank by mean), and \textbf{Variance} (reject candidates with high score variance, then rank by mean among the rest). These baselines isolate the contribution of our consensus formulation by controlling for the agent architecture and scoring procedure.


\noindent ~\cref{tab:main_results} presents accuracy across six benchmarks under all verification strategies. We organize the discussion around \textbf{\textit{three findings.}}

\noindent
\textbf{\method consistently improves over the base model, without degradation on any benchmark.} Our method improves over the unverified base model on every benchmark, with gains ranging from $+2.90$ on 3DSRBench to $+5.95$ on CV-Bench-3D. This consistency holds across qualitatively different tasks: spatial reasoning (3DSRBench, CV-Bench-3D), visual grounding (CV-Bench-2D, AI2D), and multimodal abstraction (BLINK, MMStar). 


\noindent
\textbf{Domain-specific critics show cross-task fragility.}Every domain-specific critic we evaluate degrades below baseline on at least two benchmarks. Sherlock~\cite{sherlock} drops $18.26$ points on CV-Bench-3D and $8.01$ points on 3DSRBench. VisionSR1~\cite{li2025selfrewardingvisionlanguagemodelreasoning} 
collapses by $22.21$ points on CV-Bench-3D and $18.22$ points on BLINK. Even LLaVA-Critic~\cite{wang2025llavacriticr1criticmodelsecretly}, which achieves the highest single-benchmark score among domain-specific methods on CV-Bench-3D ($81.58$), simultaneously degrades by $6.75$ points on CV-Bench-2D, two splits of the \emph{same} benchmark family. This pattern is consistent with capability tensions documented in prior work~\cite{kang2025vgentvisualgroundingmodular,shen2024multimodal,zhang2023llavagrounding}: a verifier trained to reward one reasoning style can actively penalize another. Our method sidesteps this entirely by aggregating frozen, modality-specialized judgments through consensus computation rather than 
learning a single score function that must generalize across heterogeneous evaluation criteria.

\noindent
\textbf{\method adds value beyond domain-agnostic aggregation.} To isolate the contribution of our consensus formulation, we compare against five aggregation baselines that use identical agents and scoring procedures. The only difference is how raw scores are combined into a selection decision. Among these, Mean scoring emerges as the strongest simple baseline, followed by Max and Variance. However, \method consistently outperforms Mean by $+0.68$ to $+2.57$ across all benchmarks. This margin confirms that the disagreement structure captured by consensus computation, specifically, the coupling between agents introduced through the coupled scoring framework, provides a verification signal that score averaging discards. It is worth noting that the Variance baseline, which explicitly attempts to incorporate disagreement by rejecting high-variance candidates, still underperforms our method. This indicates that treating all variance equally, without distinguishing genuine cross-modal conflict from ordinary confidence fluctuations, is insufficient (a finding consistent with Proposition~1).
\vspace{-15pt}

%% file: sections/main_table.tex
\begin{table}[t]
\renewcommand{\arraystretch}{1.05}
\setlength{\tabcolsep}{4pt}
\centering
\caption{Accuracy (\%) across multimodal reasoning benchmarks under different verification strategies. \textbf{Avg.} is the unweighted mean across all six benchmarks. Values in \textcolor{blue}{blue} indicate improvement of our method relative to the Base Model. Our detailed evaluation setup is given in the supplementary.}
\vspace{-5pt}
\begin{adjustbox}{width=\textwidth}
\begin{tabular}{lccccccc}
\hline
\textbf{Dataset} 
& \textbf{3DSRBench}
& \textbf{CV-Bench-3D} 
& \textbf{CV-Bench-2D} 
& \textbf{BLINK} 
& \textbf{MMStar} 
& \textbf{AI2D} 
& \textbf{Avg.} \\
\hline
\rowcolor{gray!15}
Base
& 56.12 
& 76.39 
& 74.27 
& 48.31 
& 61.25 
& 81.52 
& 66.31 \\
\hline
\rowcolor{blue!15}
\multicolumn{8}{l}{\textit{Domain-Specific Critics}} \\
LLaVA Critic
& 52.71 $_{\pm 0.83}$
& 81.58 $_{\pm 0.91}$
& 67.52 $_{\pm 0.82}$
& 49.22 $_{\pm 0.77}$
& 64.07 $_{\pm 0.61}$
& 81.81 $_{\pm 0.44}$
& 66.15 \\
Critic-V$_{\textcolor{gray!70}{\text{CVPR, 25}}}$
& 53.25 $_{\pm 0.73}$
& 77.66 $_{\pm 0.81}$
& 75.38 $_{\pm 0.79}$
& 46.17 $_{\pm 0.86}$
& 55.83 $_{\pm 0.59}$
& 80.17 $_{\pm 0.75}$
& 64.74 \\
Sherlock$_{\textcolor{gray!70}{\text{NeurIPS, 25}}}$
& 48.11 $_{\pm 0.93}$
& 58.13 $_{\pm 1.10}$
& 68.78 $_{\pm 1.98}$
& 49.07 $_{\pm 0.89}$
& 57.26 $_{\pm 0.96}$
& 82.77 $_{\pm 0.52}$
& 60.69 \\
VisionSR1$_{\textcolor{gray!70}{\text{ICLR, 26}}}$
& 53.05 $_{\pm 1.82}$
& 54.18 $_{\pm 1.05}$
& 73.10 $_{\pm 1.91}$
& 30.09 $_{\pm 1.45}$
& 57.20 $_{\pm 2.73}$
& 80.25 $_{\pm 0.97}$
& 57.98 \\
DreamPRM$_{\textcolor{gray!70}{\text{NeurIPS, 25}}}$
& 53.08 $_{\pm 1.88}$
& 63.39 $_{\pm 0.98}$
& 75.58 $_{\pm 1.61}$
& 49.89 $_{\pm 0.82}$
& 61.23 $_{\pm 0.59}$
& 81.21 $_{\pm 0.49}$
& 64.06 \\
\hline
\rowcolor{blue!15}
\multicolumn{8}{l}{\textit{Domain-Agnostic Baselines}} \\
Variance
& 57.21 $_{\pm 0.57}$
& 78.16 $_{\pm 0.68}$
& 76.43 $_{\pm 0.53}$
& 49.61 $_{\pm 0.49}$
& 63.09 $_{\pm 0.47}$
& 81.41 $_{\pm 0.34}$
& 67.65 \\
Mean
& 58.34 $_{\pm 0.53}$
& 79.77 $_{\pm 0.61}$
& 77.57 $_{\pm 0.51}$
& 50.17 $_{\pm 0.58}$
& 64.14 $_{\pm 0.44}$
& 82.18 $_{\pm 0.36}$
& 68.70 \\
Min
& 56.11 $_{\pm 0.58}$
& 77.21 $_{\pm 0.63}$
& 75.17 $_{\pm 0.59}$
& 49.39 $_{\pm 0.55}$
& 62.41 $_{\pm 0.46}$
& 81.09 $_{\pm 0.42}$
& 66.90 \\
Majority
& 57.58 $_{\pm 0.62}$
& 77.37 $_{\pm 0.73}$
& 75.53 $_{\pm 0.68}$
& 48.43 $_{\pm 0.66}$
& 62.10 $_{\pm 0.52}$
& 81.78 $_{\pm 0.48}$
& 67.13 \\
Max
& 57.37 $_{\pm 0.55}$
& 78.16 $_{\pm 0.56}$
& 76.46 $_{\pm 0.63}$
& 49.17 $_{\pm 0.43}$
& 63.42 $_{\pm 0.58}$
& 82.12 $_{\pm 0.54}$
& 67.78 \\
\rowcolor{green!15}
\textbf{\method (Ours)}
& \textbf{59.02} $_{\pm 0.45}$
& \textbf{82.34} $_{\pm 0.51}$
& \textbf{79.22} $_{\pm 0.53}$
& \textbf{51.32} $_{\pm 0.48}$
& \textbf{65.88} $_{\pm 0.37}$
& \textbf{83.14} $_{\pm 0.32}$
& \textbf{70.15} \\
\hline
Gains
& \textcolor{blue}{+2.90}
& \textcolor{blue}{+5.95}
& \textcolor{blue}{+4.95}
& \textcolor{blue}{+3.01}
& \textcolor{blue}{+4.63}
& \textcolor{blue}{+1.62}
& \textcolor{blue}{+3.84} \\
\hline
\end{tabular}
\end{adjustbox}

\label{tab:main_results}
\vspace{-20pt}
\end{table}


%% file: sections/7_2_Ablation.tex
\section{Discussion and Analysis}
\vspace{-7pt}

\begin{figure*}[h]
\centering
\begin{minipage}{0.48\linewidth}
\centering
\includegraphics[width=\linewidth]{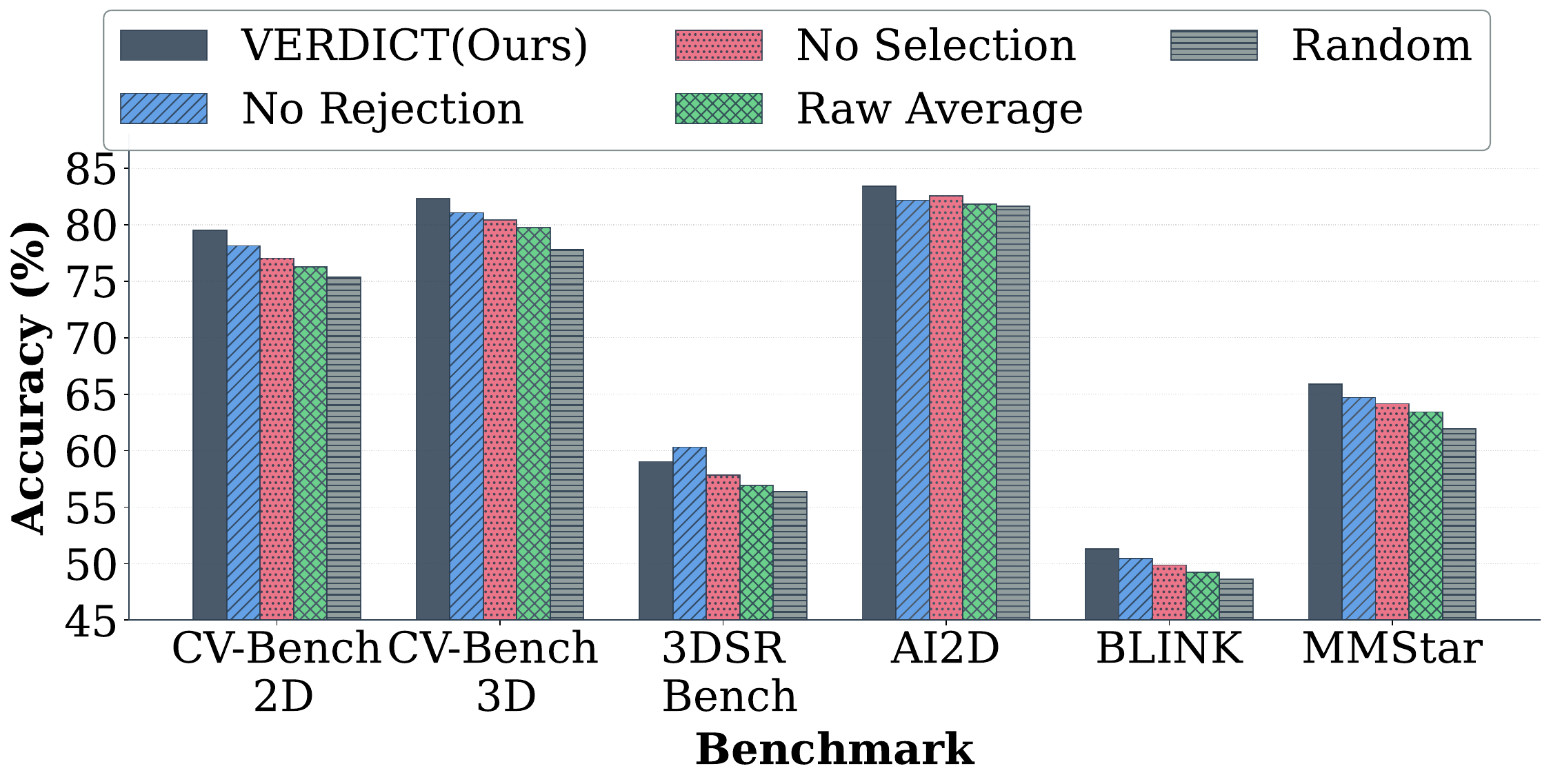}
\caption{\method achieves the highest accuracy on five of six benchmarks, with each component (rejection and selection) contributing disparately.}
\label{fig:rej_sampling_ablation}
\end{minipage}
\hfill
\begin{minipage}{0.48\linewidth}
\centering
\includegraphics[width=\linewidth]{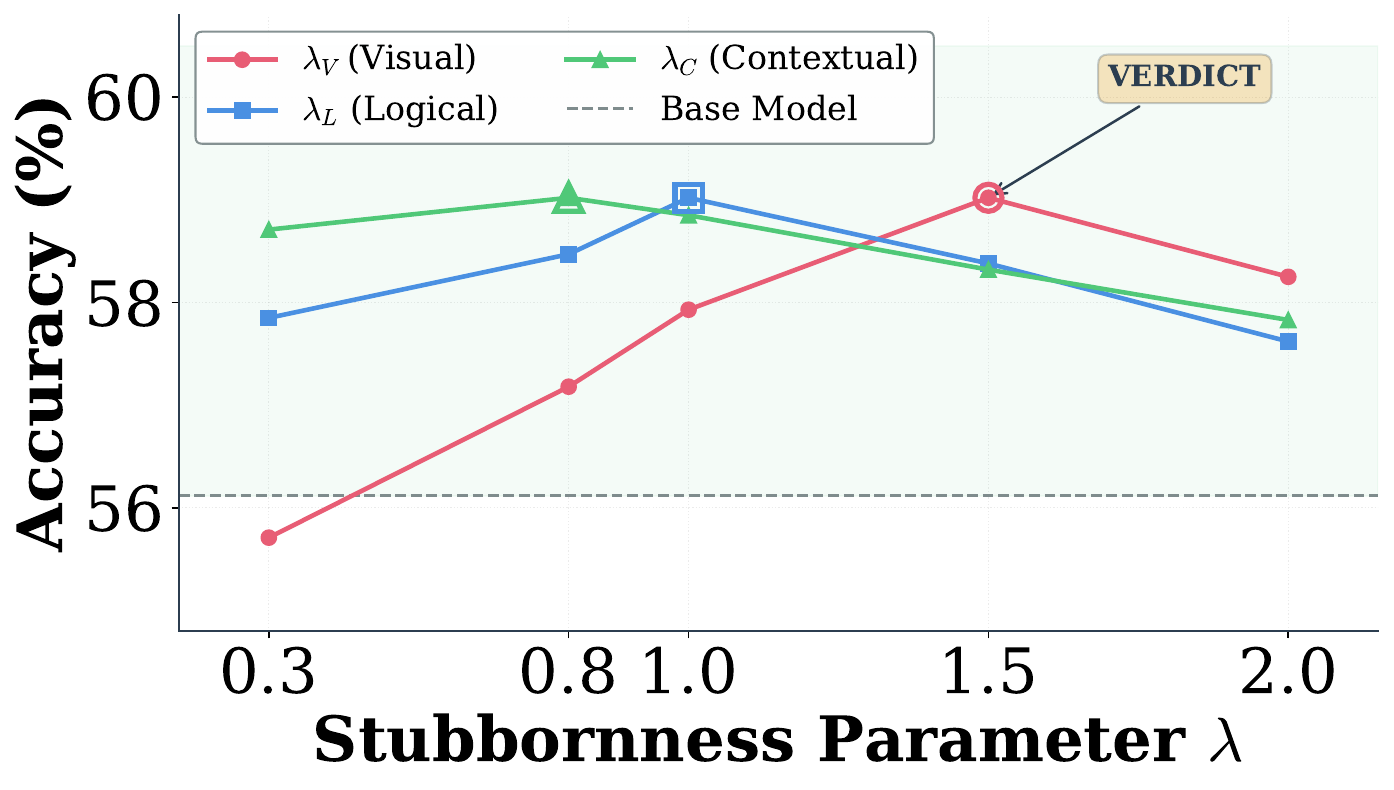}
\caption{Stubbornness parameter sensitivity on 3DSRBench. Each curve varies one parameter while holding the others fixed at our configuration. Open markers indicate our chosen values.}
\label{fig:stub_sensitivity}
\end{minipage}
\vspace{-10pt}
\end{figure*}

We conduct a series of ablation studies to understand the contribution of each component in our framework. 
\vspace{-10pt}
\subsection{Rejection vs.\ Selection}

To understand whether improvements come primarily from filtering unreliable candidates or from better ranking among viable options, we decompose our method into its constituent mechanisms. We evaluate five strategies: \textbf{\method} (our complete method: filter by acceptance criterion, then rank by $\bar{s}^*$), \textbf{No Rejection} (skip filtering, rank all candidates by $\bar{s}^* - \Delta^*$), \textbf{No Selection} (keep filtering, randomly select among accepted), \textbf{Raw Average} (replace consensus scores with simple means and dispersions, keeping the same dual-criterion structure), and \textbf{Random} (no filtering, no ranking). \cref{fig:rej_sampling_ablation} presents the results. \textbf{Three findings emerge:} First, rejection and selection contribute differently: removing either degrades accuracy, but removing
intelligent ranking (No Selection, $-1.17$ to $-2.17$) costs more than removing filtering (No Rejection, $-0.26$ to $-1.08$),
indicating that consensus-adjusted ranking is the primary driver. Second, Raw Average consistently underperforms \method by
$2.59$-$2.95$ points despite using the same dual-criterion architecture, confirming that the coupled scoring framework produces better-calibrated scores than naive aggregation independent of the filtering and ranking structure built on top. Third, Random selection performs only marginally above the base model, establishing that verification-aware scoring, not candidate diversity alone drives the gains.
\vspace{-5pt}
\subsection{Stubbornness Sensitivity}
\vspace{-5pt}
We ablate each stubbornness parameter separately on 3DSRBench, varying it over $\{0.3, 0.8, 1.0, 1.5, 2.0\}$ while holding the other two fixed. \cref{fig:stub_sensitivity} reveals a clear sensitivity hierarchy: $\lambda_V$ has the widest performance range (2.90 percentage points), $\lambda_L$ an intermediate range (1.40), and $\lambda_C$ the narrowest (1.19). This ordering directly mirrors our asymmetric design where the parameter with the most impact on spatial reasoning is set highest, while the least impactful one is given the most flexibility. Critically, no setting degrades below the base model, and the three curves converge at their respective optimal values to the same peak accuracy (59.02\%), confirming that the chosen configuration sits at a jointly optimal point rather than a compromise. The range of  \cref{fig:stub_sensitivity} $\lambda \in [0.3, 2.0]$ spans the full coupling transition: at $\lambda =  0.3$ an agent places 77\% weight on the group mean ( \cref{eq:equilibrium}), making it nearly consensus-driven; at $\lambda = 2.0$ it places 67\% on its own evidence, yet the remaining 33\% consensus pull still preserves meaningful inter-agent coupling which is the defining property of the framework.

\noindent\textbf{Stubbornness Assignment.}
The per-parameter sensitivity above establishes \emph{which} parameter matters most; it does not test whether the \emph{assignment} of values to agents matters.
We permute the stubbornness triple $(\lambda_V{=}1.5,\;\lambda_L{=}1.0,\;\lambda_C{=}0.8)$ via pairwise swaps on 3DSRBench.
 Swapping L$\leftrightarrow$C ($\lambda_V$ unchanged) costs only $-$0.41 points (58.61\%), while swapping V$\leftrightarrow$L costs $-$1.19 (57.83\%) and V$\leftrightarrow$C costs $-$1.76 (57.26\%). The degradation is proportional to the reduction in the visual agent's stubbornness. As long as $\lambda_V$ retains the highest value, \method still exceeds the Mean baseline (58.34\%); reducing it below either other agent drops performance below Mean. The assignment is not arbitrary: on perception-heavy tasks, visual evidence is the least negotiable signal, and the consensus formulation amplifies this only when the stubbornness structure reflects it.

\vspace{-8pt}
\subsection{Threshold Sensitivity}
\label{sec:threshold_sensitivity}

\begin{figure*}[t]
\centering
\begin{minipage}{0.48\linewidth}
\centering
\includegraphics[width=\linewidth]{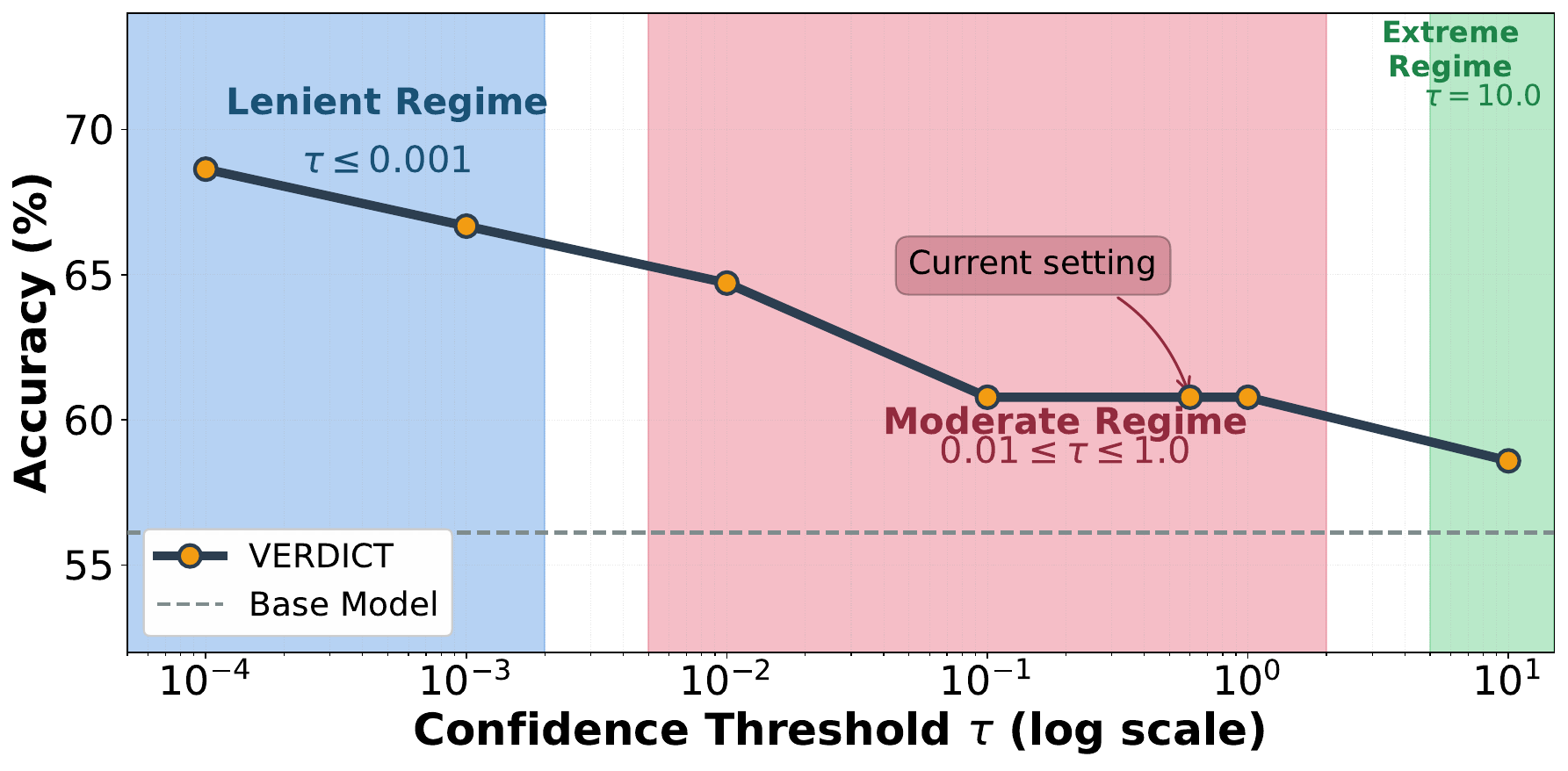}
\caption{Confidence threshold sensitivity on 3DSRBench with $\epsilon = 0.1$ fixed. Performance follows a descending gradient.}
\label{fig:tauab}
\end{minipage}
\hfill
\begin{minipage}{0.48\linewidth}
\centering
\includegraphics[width=\linewidth]{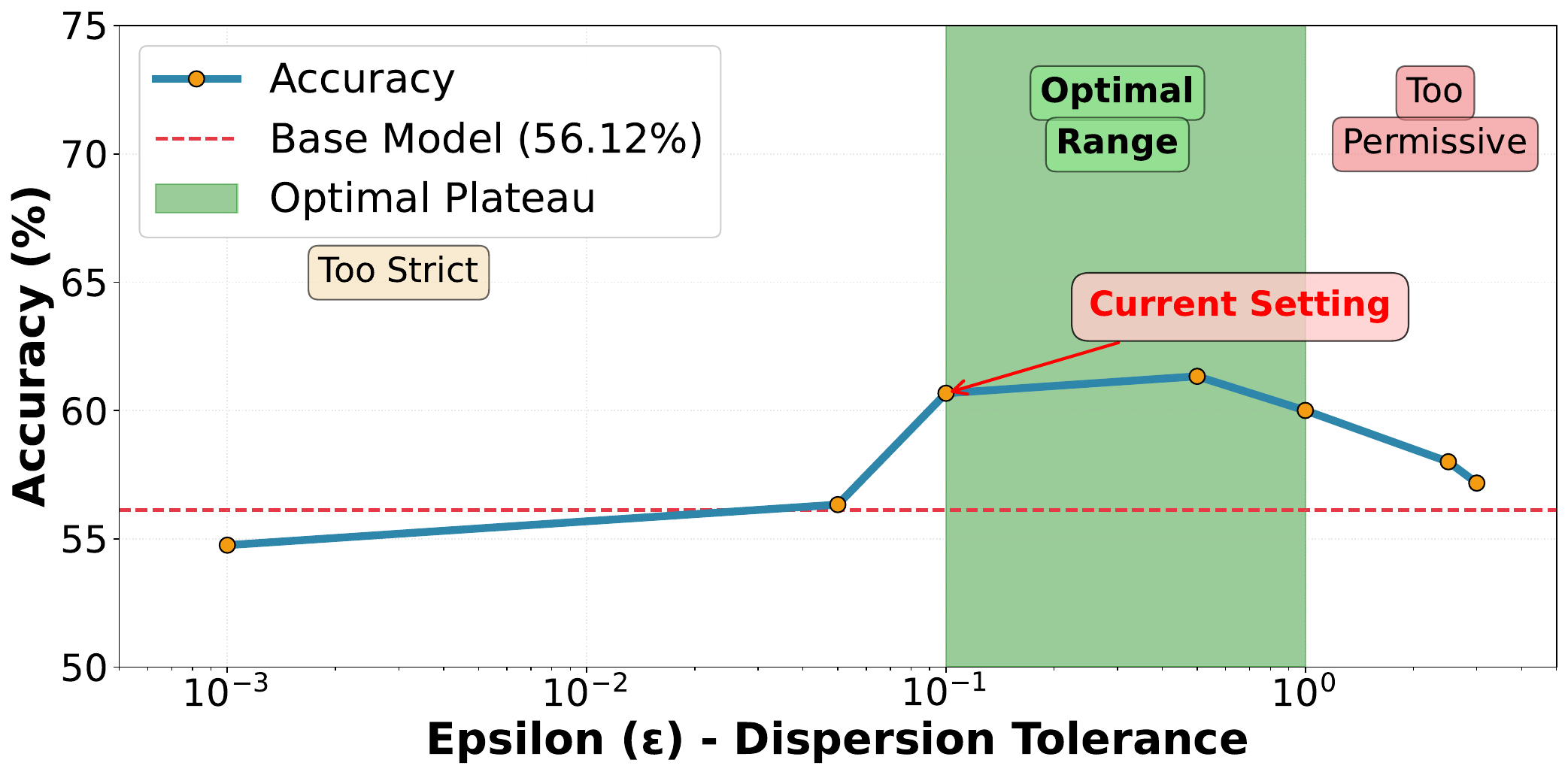}
\caption{Dispersion Tolerance Sensitivity on 3DSRBench with $\tau = 0.6$ fixed.}
\label{fig:epsab}
\end{minipage}
\end{figure*}

\noindent\textbf{Confidence threshold $\tau$.}
We ablate the confidence threshold $\tau$ while holding the dispersion tolerance fixed at $\epsilon = 0.1$ on 3DSRBench, and ~\cref{fig:tauab} present results across $\tau \in \{10^{-4}, 10^{-3}, 10^{-2}, 10^{-1}, 0.6, 1.0, 10.0\}$, extended analysis across operating regimes is provided in the supplementary.For~$\tau$ in ~\cref{fig:tauab}, performance follows a descending gradient across three regimes: permissive ($\tau \leq
0.01$, accuracy $65$--$69\%$), intermediate ($\tau = 0.1$--$0.6$, around $60$--$61\%$), and restrictive ($\tau \geq 1.0$, ${\sim}60.6\%$ via fallback ranking). The permissive regime is the most revealing:
with the confidence filter effectively disabled, the consensus-adjusted ranking alone substantially exceeds the Mean
baseline ($58.34\%$), confirming that the consensus formulation itself, not the filtering architecture, drives the improvement.

\noindent\textbf{Cross-benchmark threshold analysis.}
We extend the $\tau$ sensitivity analysis across all six benchmarks
(full results in supplementary). The permissive regime ($\tau \leq 0.01$) achieves higher accuracy on 3DSRBench
($68.71\%$ vs.\ $59.02\%$), but $\tau = 0.6$ is individually optimal on the remaining five benchmarks.

\noindent\textbf{Dispersion tolerance $\epsilon$.}
We ablate the dispersion tolerance while holding $\tau = 0.6$ fixed.  ~\cref{fig:epsab} presents accuracy across $\epsilon \in \{0.001, 0.05, 0.1, 0.5, 1.0, 2.5, 3.0\}$.
For~$\epsilon$ in ~\cref{fig:epsab} , performance peaks in an optimal plateau at $\epsilon \in [0.1, 1.0]$, with accuracy
around $60$--$61\%$. Too strict ($\epsilon < 0.1$) rejects legitimate steps with natural cross-modal tension; too permissive ($\epsilon > 1.0$) admits candidates where judges have not genuinely converged.
The decline is gradual in both directions because the fallback mechanism ($\bar{s}^* - \Delta^*$ ranking) and the confidence
threshold ($\bar{s}^* > \tau$) respectively provide safety nets, preventing catastrophic failure at any operating point.
\footnote{Across all six benchmarks at the operating point($\tau{=}0.6,\;\epsilon{=}0.1$), the fallback path is triggered on approximately 15\% of reasoning steps, confirming that the dual acceptance criterion admits at least one viable candidate on the large majority of steps}


\vspace{-7pt}
\subsection{Disentangling Judge Quality from Algorithmic Gain}


\begin{wrapfigure}{r}{0.46\linewidth}
\centering
\includegraphics[width=\linewidth]{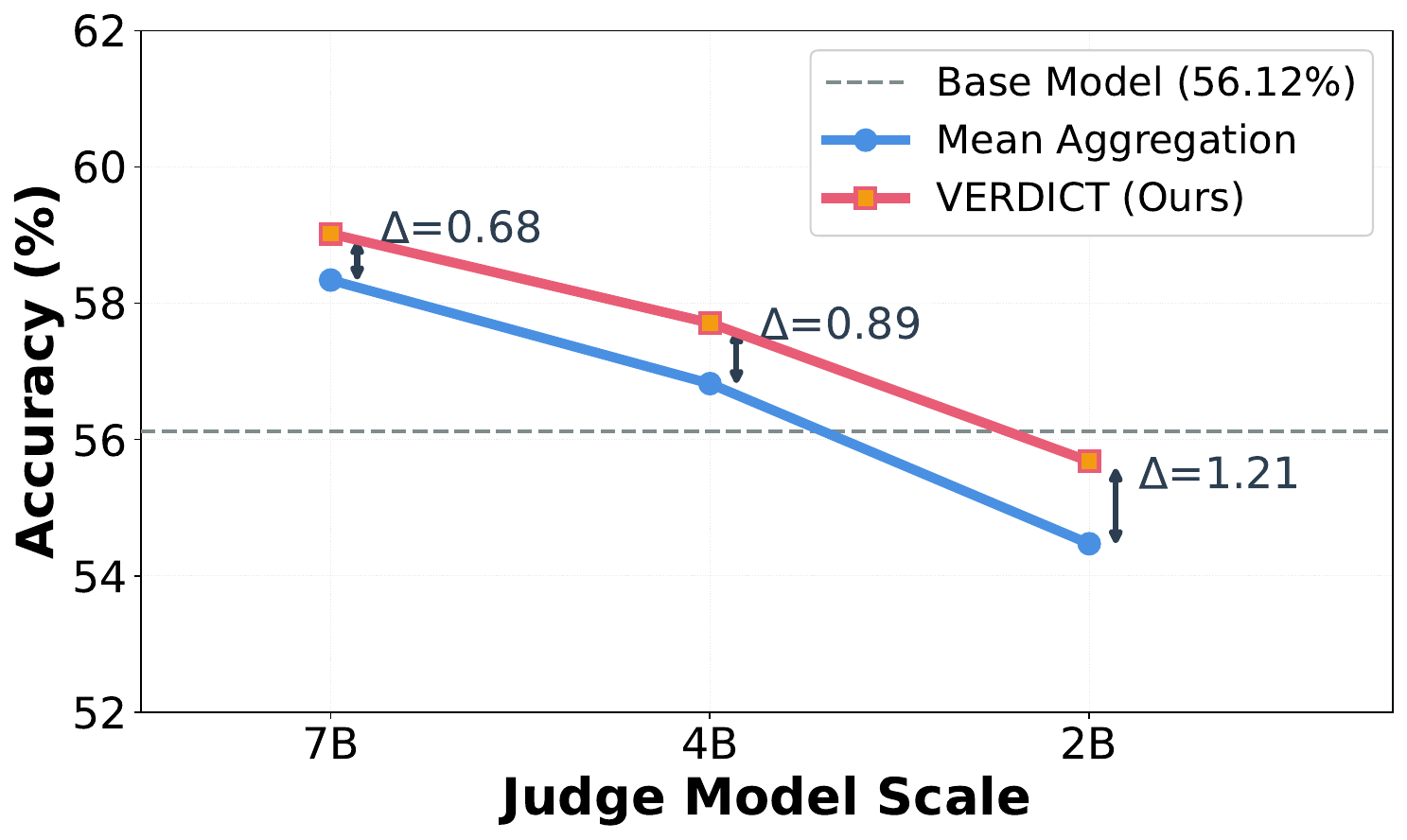}
\caption{Disentangling judge quality from algorithmic gain on 3DSRBench.Both methods degrade with weaker judges, but the margin between VERDICT and Mean widens.}
\label{fig:degcur}
\end{wrapfigure}

We replace the 7B verification agents with 4B and 2B variants while holding the base reasoner fixed, isolating scoring quality from candidate quality. \cref{fig:degcur} reports results on 3DSRBench. Absolute accuracy degrades monotonically with judge capability for both VERDICT and Mean, as expected. However, the margin between VERDICT and Mean widens from +0.68 at 7B to +0.89 at 4B to +1.21 at 2B, consistent with Proposition 1: the coupled scoring system surfaces noise-induced disagreement through the dispersion filter rather than passing it to the selection decision. At the 2B scale, Mean drops 1.65 points below the unverified baseline while VERDICT with the same judges limits this degradation to 0.44 points, a 1.21 point recovery attributable entirely to the consensus formulation, confirming that the reported gains reflect algorithmic structure rather than judge quality.Additional results can be found in the supplementary materials.


\vspace{-8pt}

%% file: sections/8_9_2_discuss_concl.tex
\section{Conclusion}
\vspace{-8pt}
Multimodal reasoning in MLLMs fails gradually as locally plausible steps drift 
toward incorrect answers. Our results across six benchmarks show that this drift 
leaves a detectable trace in the \emph{structure} of cross-modal disagreement. 
\method formalizes this signal through disagreement-aware consensus computation, 
achieving consistent improvements without task-specific tuning. Our \textbf{key 
takeaway} is that when disparate judges disagree about a reasoning step, the 
pattern of that disagreement is itself a verification signal that principled 
 aggregation can exploit. The framework has certain limitations. When all agents are confidently wrong, consensus 
converges on an incorrect assessment; and no filtering recovers from a base model 
that produces no viable candidates. Characterizing how often this failure mode 
occurs in practice is deferred to future work.  On the practical side, \method is fully parallelizable and dominated by generation cost, scoring requires 1–5 tokens per step versus 50–200 per reasoning step, yielding $3.80\times$ wall-clock overhead under sequential execution. This is non-trivial but fundamentally cheaper than regenerating entire trajectories, and reducible via adaptive verification at high-uncertainty steps, a direction we leave to future work. Per-agent scores further provide a modality-level diagnostic trace pinpointing failures in visual grounding, logical coherence, or contextual relevance. Unlike domain-specific critics, this overhead is the \emph{entire} cost. Because the coordination-game formulation requires only that heterogeneous evaluators trade off private evidence against consensus, 
\method generalizes beyond multimodal reasoning to reward model 
ensembles~\cite{helpingherdingrewardmodel,weaver}, multi-agent evaluation, and 
compositional code verification. 

%% file: suppl_sections/takeaways.tex
\begin{tcolorbox}[
  colback=blue!3,
  colframe=blue!40!black,
  coltitle=white,
  title={\textbf{Supplementary at a Glance: Key Takeaways}},
  fonttitle=\bfseries,
  boxrule=0.8pt,
  arc=2pt,
  left=6pt, right=6pt, top=4pt, bottom=4pt
]
\small
This supplement provides extended derivations, ablations, and analyses supporting the main paper. We highlight the four most important findings:
\begin{enumerate}[leftmargin=*, itemsep=4pt, topsep=4pt, label=\textbf{\arabic*.}]
\item \textbf{Cross-benchmark robustness of the default configuration.}
All hyperparameters ($\lambda_V{=}1.5$, $\lambda_L{=}1.0$, $\lambda_C{=}0.8$, $\tau{=}0.6$, $\epsilon{=}0.1$) are \emph{fixed across all six benchmarks} with no per-task tuning.
The dispersion tolerance $\epsilon$ exhibits a stable optimal plateau spanning a full order of magnitude ($\epsilon \in [0.1,\, 1.0]$), and the confidence threshold $\tau{=}0.6$ is individually optimal on five of six benchmarks.
The joint threshold sweep (49 configurations) confirms that the operating point lies on a plateau, not an edge.
\item \textbf{Consensus adds value \emph{beyond} aggregation, especially with weaker judges.}
The margin between \textbf{VERDICT} and Mean aggregation \emph{widens} as the judge scale decreases from 7B to 2B on five of six benchmarks. At 2B scale, \textbf{VERDICT} recovers 1.2--4.6 points over na\"{\i}ve averaging.
This confirms that the gains are attributable to the consensus formulation's algorithmic structure, not to judge quality alone.
\item \textbf{Dispersion is a genuine, non-separable diagnostic signal.}
Consensus dispersion $\Delta^*$ is enriched 1.4--1.6$\times$ for incorrect chains (Section~S4.9); ROC analysis yields AUC $=$ 0.65--0.66.
The signal is moderate and diagnostic rather than oracular, which explains why \textbf{VERDICT} benefits from soft integration (ranking $+$ dual filtering) rather than hard binary classification.
\item \textbf{The framework generalises across model families at zero adaptation cost.}
Applying \textbf{VERDICT} to four different base model families yields consistent gains of $+2.45$ to $+4.00$ points. The consensus compresses base-model variability from a 1.33-point range to 0.22 points, confirming plug-in generality.
\end{enumerate}
\end{tcolorbox}
\medskip

%% file: suppl_sections/hyperparameters.tex
\section{Hyperparameters}

\begin{table}[!h]
\centering
\begin{tabular}{l c}
\hline
Hyperparameter & Value \\
\hline
$\lambda_{\text{V}}$        & $1.5$ \\
$\lambda_{\text{L}}$    & $1.0$ \\
$\lambda_{\text{C}}$    & $0.8$ \\
$n$                     & 3 \\
$\tau$                  & 0.6\\
$\epsilon$              &  0.1 \\
\hline
\end{tabular}
\caption{Hyperparameters used in all experiments.}
\label{tab:lambda_values}
\end{table}

\noindent
In our experiments, we used hyperparameters as shown in Table \ref{tab:lambda_values}. This configuration reflects the intuition that the visual verifier (V) should be most resistant to consensus pressure on perception-heavy steps, while the contextual verifier (C) may be more flexible when visual or linguistic evidence is strong. These values are fixed across all datasets and require no tuning.

%% file: suppl_sections/app_implementation.tex
\section{Implementation Details: \method}
\label{app:implementation}
\input{suppl_sections/algorithm}

\noindent
Here, we discuss how the \method's consensus-based verification is implemented in practice. 

\noindent
At each reasoning step, the base MLLM proposes several candidate continuations (typically 3 candidates per step). For every candidate, three frozen verifier agents are queried independently: a visual grounding agent (V), a logical consistency agent (L), and a contextual reasoning agent (C). Each produces a scalar confidence score in $[0,1]$, reflecting its modality-specific assessment of each step.


\noindent
If one verifier is confident while the others are uncertain, averaging may obscure this inconsistency. We want a mechanism that surfaces disagreement explicitly and allows agents to partially adjust their beliefs toward consensus without forcing uniformity when genuine divergence exists.

\noindent
The consensus formulation captures this trade-off naturally. Each verifier prefers agreement with others, but not at the cost of abandoning its own judgment entirely. This is encoded through a quadratic scoring objective, which admits a unique, closed-form. The consensus scores thus represent a principled resolution of inter-agent tension, rather than an ad-hoc blend.

\subsection*{Consensus formulation and computation}

The implementation uses a heterogeneous trade-off parameter formulation, where different verifiers are assigned different sensitivities to consensus pressure. This is motivated by the observation that certain verifiers should be more resistant to group influence depending on the nature of the reasoning step.

Each agent's consensus score $s_i^*$ satisfies:
\begin{equation}
s_i^* = \frac{\bar{s}_{-i}^* + \lambda_i \hat{s}_i}{1 + \lambda_i}
\end{equation}
where $\hat{s}_i$ is agent $i$'s raw confidence score, $\bar{s}_{-i}^* = \frac{1}{n-1} \sum_{j \neq i} s_j^*$ is the mean consensus score of all other agents, and $\lambda_i > 0$ controls how strongly agent $i$ weights its own judgment relative to group consensus.

This system of equations can be rewritten as a linear system and solved exactly:
\begin{equation}
\label{eq:linear}
\left(1 + \lambda_i\right) s_i^* - \frac{1}{n-1} \sum_{j \neq i} s_j^* = \lambda_i \hat{s}_i, \quad i = 1, \dots, n
\end{equation}
The system is solved using a standard linear solver. The resulting matrix is typically well-conditioned when $\lambda_i > 0$, but if numerical issues arise, the implementation falls back to the raw verifier scores. All consensus scores are clipped to $[0,1]$ to maintain valid confidence values.

\noindent
In our experiments, we set
$n=3$
$\lambda_{\text{V}} = 1.5$, $\lambda_{\text{L}} = 1.0$, and $\lambda_{\text{C}} = 0.8$. This configuration reflects the intuition that the visual verifier (V) should be most resistant to consensus pressure on perception-heavy steps, while the contextual verifier (C) may be more flexible when visual or linguistic evidence is strong. These values are fixed across all datasets and require no tuning.

\noindent
Importantly, \emph{no iterative optimization, learning, or approximation is used}. The consensus solution is computed via a direct linear solve at every reasoning step, adding negligible overhead compared with the cost of querying the verifier models themselves.

\subsection*{Step selection via consensus statistics}

Once consensus scores are obtained, the system computes two summary statistics:
\begin{itemize}
    \item \textbf{Mean consensus confidence} $\bar{s}^* = \frac{1}{n} \sum_i s_i^*$: measures collective endorsement of a step.
    \item \textbf{Consensus dispersion} $\Delta = \frac{1}{n} \sum_i |s_i^* - \bar{s}^*|$: measures residual disagreement after consensus adjustment.
\end{itemize}
Candidate steps are accepted only if they simultaneously achieve \textbf{high collective confidence} ($\bar{s}^* > \tau$) and \textbf{low inter-agent dispersion} ($\Delta < \epsilon$). In our experiments, we set $\tau = 0.6$ and $\epsilon = 0.1$. This dual criterion is stricter than confidence alone: a step with high average confidence but high dispersion signals unresolved conflict and is rejected.

\noindent
When multiple candidate steps are evaluated at the same reasoning position, rejected steps are discarded immediately. Among accepted steps, the one with the highest $\bar{s}^*$ is selected to extend the reasoning trace. If \emph{no} candidate step is accepted (i.e., all steps have either low confidence or high dispersion), the implementation selects the step that maximizes $\bar{s}^* - \Delta$, prioritizing the best available balance between confidence and agreement. This fallback ensures the reasoning process can continue even when all candidates are suboptimal, while still preferring more stable steps.

\subsection*{Why this matters}

\method serves as a lightweight consensus protocol over frozen verifiers. It makes disagreement \emph{explicit and actionable}, rather than burying it in an average. It requires no training or calibration beyond setting three hyperparameters ($\lambda_i$ values, $\tau$, and $\epsilon$), all of which remain fixed across datasets, and integrates naturally into step-wise reasoning by filtering unstable steps before they can compound downstream errors.

\noindent
Crucially, the consensus is not a heuristic approximation. It is the exact fixed point of a well-defined coupled scoring system. This gives the filtering process a principled justification and makes the system's behavior more interpretable: rejected steps are precisely those where verifiers could not reach a stable agreement, even after accounting for consensus pressure. The heterogeneous $\lambda_i$ values allow the system to implicitly adapt to different reasoning regimes without explicit step-type classification.

%% file: suppl_sections/algorithm.tex
\begin{algorithm}[!h]
\caption{\method: Step-Wise Consensus Verification}
\label{alg:nash_verification}
\begin{algorithmic}[1]

\REQUIRE Base model $\mathcal{M}_{\text{base}}$, Verifier agents $\{\mathcal{V}, \mathcal{L}, \mathcal{C}\}$, Image $I$, Question $Q$
\REQUIRE Number of candidates $n=3$, Stubbornness parameters $\{\lambda_V=1.5, \lambda_L=1.0, \lambda_C=0.8\}$
\REQUIRE Acceptance thresholds: dispersion $\epsilon=0.1$, confidence $\tau=0.6$
\ENSURE Complete reasoning trace $r_{1:T}$ and final answer

\STATE Initialize reasoning trace $r \leftarrow \emptyset$
\STATE Generate initial step: $r_0 \sim \mathcal{M}_{\text{base}}(I, Q)$
\STATE Append $r_0$ to $r$

\WHILE{$r$ does not contain end-of-sequence token}
    \STATE $\textsc{Candidates} \leftarrow \emptyset$
    \STATE $\textsc{RawScores} \leftarrow \emptyset$
    
    \FOR{$i = 1$ to $n$}
        \STATE Generate candidate: $c_i \sim \mathcal{M}_{\text{base}}(I, Q, r)$ with temperature $T=0.8$
        \STATE Query Visual Agent: $\hat{s}_V^{(i)} \leftarrow \mathcal{V}(I, Q, r, c_i)$ \COMMENT{Score $\in [0,1]$}
        \STATE Query Logical Agent: $\hat{s}_L^{(i)} \leftarrow \mathcal{L}(I, Q, r, c_i)$
        \STATE Query Contextual Agent: $\hat{s}_C^{(i)} \leftarrow \mathcal{C}(I, Q, r, c_i)$
        \STATE Store: $\textsc{Candidates}[i] \leftarrow c_i$, $\textsc{RawScores}[i] \leftarrow \{\hat{s}_V^{(i)}, \hat{s}_L^{(i)}, \hat{s}_C^{(i)}\}$
    \ENDFOR
    
    \STATE \COMMENT{\textbf{Consensus Computation \& Selection}}
    \STATE $\textsc{Results} \leftarrow \emptyset$
    
    \FOR{each candidate $i$ in $\{1, \ldots, n\}$}
        \STATE Solve $\mathbf{s}^{*(i)}$ per Eq. \ref{eq:linear}
        \STATE Compute mean: $\bar{s}^{*(i)} \leftarrow \frac{1}{3}\sum_{j \in \{V,L,C\}} s_j^{*(i)}$
        \STATE Compute dispersion: $\Delta^{(i)} \leftarrow \frac{1}{3}\sum_{j \in \{V,L,C\}} |s_j^{*(i)} - \bar{s}^{*(i)}|$
        \STATE Check acceptance: $\text{accepted}^{(i)} \leftarrow (\Delta^{(i)} < \epsilon) \land (\bar{s}^{*(i)} > \tau)$
        \STATE Store: $\textsc{Results}[i] \leftarrow \{\mathbf{s}^{*(i)}, \bar{s}^{*(i)}, \Delta^{(i)}, \text{accepted}^{(i)}\}$
    \ENDFOR
    
    \STATE $\textsc{ValidSteps} \leftarrow \{i : \textsc{Results}[i].\text{accepted} = \text{True}\}$
    
    \IF{$\textsc{ValidSteps} \neq \emptyset$}
        \STATE $i^* \leftarrow \argmax_{i \in \textsc{ValidSteps}} \bar{s}^{*(i)}$ \COMMENT{Highest confidence among stable steps \textbf{Normal Model}}
    \ELSE
        \STATE $i^* \leftarrow \argmax_{i=1,\ldots,n} (\bar{s}^{*(i)} - \Delta^{(i)})$ \COMMENT{Best balance if none accepted \textbf{Fallback Model}}
    \ENDIF
    
    \STATE Append selected step: $r \leftarrow r \cup \{\textsc{Candidates}[i^*]\}$
\ENDWHILE

\STATE Extract final answer from $r$
\RETURN $r$, final answer
\end{algorithmic}
\end{algorithm}

%% file: suppl_sections/app_experimental_setup.tex
\section{Detailed Experimental Setup}
\label{app:experimental_setup}

\paragraph{Base Model Configuration}

We use Qwen2.5-VL-7B-Instruct as our primary reasoning model. The model generates candidate reasoning steps through chain-of-thought prompting with the instruction: ``\textit{[Question]. Reason step by step.}'' We employ temperature sampling with $T=0.8$ and top-$p=0.6$ to encourage diversity among the three candidate continuations generated at each step. Generation stops when the model produces either an end-of-sequence token or a newline character, with a maximum of 1000 new tokens per step.

\paragraph{Verification Agent Architecture}

All three verification agents are implemented using Qwen2.5-VL-7B-Instruct as the backbone model. Each agent receives the same visual input (the question image) along with the question text, previous reasoning steps (assumed correct), and the current candidate step to verify. Agents operate independently and output scores in $[0, 1]$.

\paragraph{Visual Agent (V)}

The Visual Agent evaluates whether objects and spatial relationships mentioned in the reasoning step are visually verifiable. The agent is specifically instructed to maintain a balance between strictness and common-sense spatial reasoning, and to output only a single number between 0.0 and 1.0.

\paragraph{Logical Agent (L)}

The Logical Agent assesses whether the reasoning step follows logically from previous steps and makes progress toward answering the question. It evaluates:
\begin{itemize}
    \item Whether the step builds coherently on facts
    \item Whether logical inferences are valid
    \item Whether the step moves closer to resolving the question
\end{itemize}

\paragraph{Contextual Agent (C)}

The Contextual Agent determines whether the step maintains focus on the original question and avoids introducing irrelevant or tangential information. It penalizes steps that:
\begin{itemize}
    \item Describe details unrelated to the question
    \item Make definitive claims about obscured or cropped-out objects
    \item Introduce unnecessary speculation
\end{itemize}

\subsection{Dataset and Evaluation}

We evaluate our approach across six diverse vision-language benchmarks that test different aspects of multimodal reasoning:

\textbf{3DSRBench} is a comprehensive 3D spatial reasoning benchmark comprising 2,772 manually annotated visual question-answer pairs across 12 question types. The benchmark evaluates four main categories of 3D awareness: height, location, orientation, and multi-object reasoning. It includes questions based on both natural images from MS-COCO and multi-view synthetic images, with particular emphasis on testing robustness across common and uncommon camera viewpoints. The benchmark employs careful design to avoid trivial answers and uses novel evaluation strategies like FlipEval to ensure robust assessment.

\textbf{CV-Bench} addresses vision-centric evaluation through 2,638 manually inspected examples repurposed from standard vision benchmarks, including ADE20k, COCO, and OMNI3D. The benchmark is divided into two components: CV-Bench-2D evaluates spatial relationships and object counting, while CV-Bench-3D assesses depth order and relative distance understanding. By formulating natural language questions that probe fundamental visual understanding, the benchmark tests whether models can perform classic computer vision tasks within a multimodal context.

\textbf{BLINK} focuses on core visual perception abilities that humans can solve “within a blink”, reformatting 14 classic computer vision tasks into 3,807 multiple-choice questions. The benchmark spans pixel-level to image-level perception tasks, including relative depth estimation, visual correspondence, forensics detection, multi-view reasoning, and visual similarity. A key feature is the incorporation of diverse visual prompts such as circles, boxes, and image masks alongside textual questions, deliberately designed to resist solutions based purely on language mediation.

\textbf{MMStar} is an elite vision-indispensable benchmark comprising 1,500 challenge samples meticulously selected by humans to address issues of visual dependency and data leakage in existing benchmarks. The benchmark evaluates six core capabilities and 18 detailed axes, with each sample undergoing strict human review to ensure visual dependency, minimal data leakage, and requirements for advanced multi-modal capabilities. Beyond traditional accuracy metrics, MMStar introduces two novel metrics to measure multi-modal gain and multi-modal leakage in model training.

\textbf{AI2D} consists of 4,817 illustrative diagrams for research on diagram understanding and associated question answering. The dataset represents topics in primary school natural sciences such as food webs, life cycles, moon phases, and human physiology. Each diagram has been densely annotated with object segmentations, diagrammatic elements like arrows and lines, and text elements. The benchmark requires understanding abstract visual representations and symbolic elements common in scientific illustrations, testing both visual comprehension and scientific reasoning abilities.

\noindent
We process each question by iteratively generating and verifying reasoning steps until the model produces an end-of-sequence token. The final answer is extracted from the complete reasoning trace using (Gemma-3:12B\cite{gemmateam2025gemma3technicalreport} via Ollama) and evaluated against the ground truth using string matching. The use of Ollama was only for answer extraction, and is not for any part of the verification pipeline.

\noindent
To isolate the contribution of the verification mechanism itself, we decouple generation variance from verification variance by fixing the base model seed across all methods, ensuring that every verifier evaluates the same set of candidate steps. The standard deviations reported (across 5 runs) therefore reflect verification noise alone. An end-to-end variance analysis, while informative for deployment, would compound sampling stochasticity across multi-step chains, conflating candidate quality with verifier quality and obscuring the comparison we aim to make: whether consensus-based aggregation selects better steps than alternatives \emph{given the same options}.
    
\paragraph{Computational Resources: }
All experiments are conducted on NVIDIA A100  and 3 NVIDIA A6000 GPUs. 
\noindent To ensure full reproducibility, all code and evaluation scripts will be publicly released upon acceptance.


%% file: suppl_sections/additional_results/additional_result_v2.tex
\section{Additional Results}

This section presents extended empirical analyses that complement the main paper's experiments. It is organized along three axes: (i) comparisons with alternative inference-time strategies \ref{sec:tts}, (ii) systematic sensitivity analyses of every hyperparameter — individually \ref{sec:epsilon_cross_bench} - \ref{sec:supp_threshold}, jointly \ref{sec:threshold_interaction}, and across model configurations \ref{sec:judge_scale}–\ref{sec:cross_bench_swaps} and (iii) diagnostic analyses that characterize why the consensus formulation works, including the dispersion signal's predictive value \ref{sec:dispersion_diagnostic}–\ref{sec:conf_disp_landscape}, the interplay between filtering and ranking \ref{sec:supp_rejection_selection}–\ref{sec:fallback_analysis}, and generalization across model families \ref{sec:cross_model}.

\input{suppl_sections/additional_results/additional_baselines}
\input{suppl_sections/additional_results/eps_analysis}
\input{suppl_sections/additional_results/tau_analysis}

\input{suppl_sections/additional_results/threshold_analysis_ext}
\input{suppl_sections/additional_results/heatmap_analysis}
\input{suppl_sections/additional_results/judge_scale}
\input{suppl_sections/additional_results/cb_stubborness_sens}
\input{suppl_sections/additional_results/stubborness_permute}
\input{suppl_sections/additional_results/dispersion_as_diagonistic}
\input{suppl_sections/additional_results/confidence_dispersion}
\input{suppl_sections/additional_results/rejection_v_sampling}
\input{suppl_sections/additional_results/fallback_analysis}
\input{suppl_sections/additional_results/cross_model_familiy}

%% file: suppl_sections/additional_results/additional_baselines.tex
\subsection{Comparison with Test-Time Scaling Methods}
\label{sec:tts}

\method and test-time scaling (TTS) methods both spend additional inference
compute to improve output quality.  They differ in \emph{where} that
computation is spent.  TTS methods generate multiple candidate responses and
aggregate them at the answer or token level.  \method generates multiple
candidate reasoning \emph{steps} and verifies each through cross-modal
consensus before the chain proceeds.  The two paradigms address different
failure modes: \method intercepts errors within a reasoning chain, while TTS methods
improve the quality of complete responses after generation.  A controlled
comparison clarifies how much each strategy contributes.

\noindent
We compare \method against five representative test-time scaling methods.
(1)~\textbf{Self-Consistency} aggregates candidate answers via majority
voting across multiple sampled outputs~\cite{wang2022selfconsistency}.
(2)~\textbf{Self-Selector} uses the base model itself as a verifier to
select one response among the candidates~\cite{chen2024universal}.
(3)~\textbf{Self-Synthesizer} combines potentially correct parts from
different responses.  We use the base model to aggregate responses into one
coherent final answer~\cite{li2025reasoningaslogicunits}.
(4)~\textbf{Self-Refinement}~\cite{selfrefinement}.
(5)~\textbf{TTAug}~\cite{kaya2026efficient}.
All methods except TTAug use temperature sampling ($T{=}0.8$) for diversity.
\method uses its standard configuration.

\begin{table}[t]
\centering
\caption{\method vs.\ test-time scaling baselines across six benchmarks.
Best per benchmark in \textbf{bold}.}
\label{tab:tts_comparison}
\small
\setlength{\tabcolsep}{3pt}
\begin{tabular}{l c c c c c c c}
\toprule
Method & 3DSR & CV-3D & CV-2D & BLINK & MMStar & AI2D & Avg. \\
\midrule
Base Model
  & 56.12 & 76.39 & 74.27 & 48.31 & 61.25 & 81.52 & 66.31 \\
\midrule
Self-Consistency$_{\textcolor{gray!70}{\text{ICLR, 23}}}$
  & 57.08 & 78.12 & 76.03 & 42.66 & 59.16 & 74.63 & 64.61 \\
Self-Selector$_{\textcolor{gray!70}{\text{ICML, 24}}}$
  & 55.98 & 72.77 & 68.04 & 35.16 & 57.31 & 77.90 & 61.19 \\
Self-Refinement$_{\textcolor{gray!70}{\text{-, -}}}$
  & 54.94 & 73.01 & 66.05 & 40.15 & 55.08 & 73.29 & 60.42 \\
Self-Synthesizer$_{\textcolor{gray!70}{\text{ICML, 25}}}$
  & 54.68 & 77.33 & 77.79 & 49.40 & 62.32 & 81.14 & 67.11 \\
TTAug$_{\textcolor{gray!70}{\text{ICLR, 26}}}$
  & 52.12 & 73.19 & 70.47 & 50.26 & 60.88 & 69.75 & 62.78 \\
\midrule
\textbf{VERDICT}
  & \textbf{59.02} & \textbf{82.34} & \textbf{79.22}
  & \textbf{51.32} & \textbf{65.88} & \textbf{83.14}
  & \textbf{70.15} \\
\bottomrule
\end{tabular}
\end{table}

\paragraph{Results.}
Table~\ref{tab:tts_comparison} reports accuracy across all six benchmarks. \method leads on average by $+2.71$ points over the strongest TTS method (Self-Synthesizer, 67.44\%). The discussion is structured around a single question: where should additional inference computation be allocated: at the response level or at the reasoning-step level?  Three structural differences between the paradigms bear on the answer.

\noindent \textit{Granularity of error correction.}
Self-Consistency, Self-Selector, Self-refinement, and Self-Synthesizer evaluate or combine complete responses.  If a reasoning chain drifts at step~3 of~8, these methods can only discard the entire trajectory and hope
that a different sample avoids the same error.  \method intervenes at step~3 directly, replacing the flawed continuation before downstream errors compound.

\noindent \textit{Nature of the verification signal.}
TTS methods derive selection or aggregation signals from the base model's own
outputs: majority agreement (Self-Consistency), self-evaluation (Self-Selector), base-model-driven refinement (Self-Refinement), response
synthesis (Self-Synthesizer), or across augmented views (TTAug).  These are single-perspective signals.  \method derives its signal from cross-modal disagreement among three independently prompted, modality-specialized judges.  Proposition~1 in the main paper establishes that this coupled disagreement structure carries information that no separable
function of individual scores can replicate.  The dispersion~$\Delta^{*}$, which is absent in all TTS baselines, provides a second acceptance axis that catches candidates with high average confidence but unresolved cross-modal
tension (Section~\ref{sec:fallback_analysis}).

\noindent \textit{Cross-task consistency.}
Four of the five TTS methods fall \emph{below} the unverified base model on average: Self-Consistency (64.61\%), Self-Selector (61.19\%),
Self-Refinement (60.42\%), and TTAug (62.78\%) each underperform the base model's 66.31\%.  Only Self-Synthesizer (67.44\%) exceeds it.  More
critically, every TTS method degrades on at least one benchmark. Self-Consistency drops 5.65~points on BLINK and 6.89~points on AI2D.
Self-Selector degrades on five of six benchmarks, with a 13.15 point collapse on BLINK.  Self-Refinement falls below baseline on all six. TTAug loses 11.77~points on AI2D.  Even Self-Synthesizer, the strongest TTS baseline, drops 1.44~points on 3DSRBench.  \method improves on every benchmark.

\noindent
\textit{Where the gap is largest.}
The advantage is most pronounced on benchmarks requiring tight step-level visual grounding.  On CV-Bench-3D, \method leads Self-Synthesizer by
5.01~points; on 3DSRBench, it leads Self-Consistency (the best TTS method on
that benchmark) by 1.94~points.  These are depth-ordering and 3D spatial reasoning tasks where a single misjudged step early in the chain propagates through all subsequent comparisons: precisely the failure mode that response-level aggregation cannot intercept.

%% file: suppl_sections/additional_results/eps_analysis.tex
\subsection{Cross-Benchmark Dispersion Tolerance Sensitivity}
\label{sec:epsilon_cross_bench}

Section~6.3 of the main paper analyzes the dispersion tolerance $\epsilon$ on 3DSRBench. Here we extend this analysis to all six benchmarks, sweeping $\epsilon \in \{0.001, 0.05, 0.1, \allowbreak\ 0.5, 1.0, 2.5, 3.0\}$ with $\tau = 0.6$ fixed and reporting accuracy with standard errors across five seeds.

\begin{figure*}[t]
\centering
\begin{minipage}{0.48\linewidth}
\centering
\includegraphics[width=\linewidth]{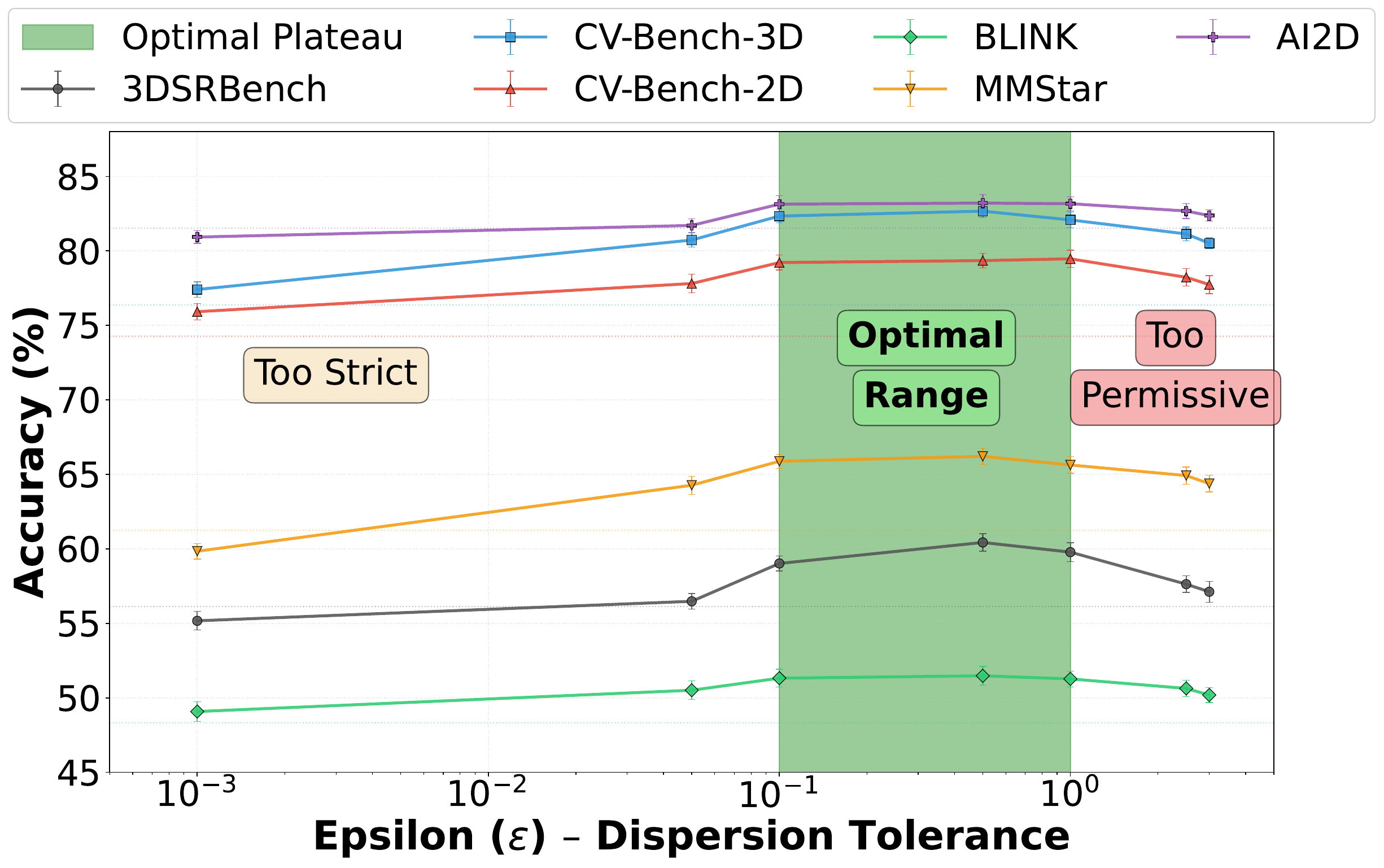}
\caption{Dispersion tolerance sensitivity across all six benchmarks ($\tau = 0.6$ fixed). All benchmarks exhibit the same three-regime structure: too-strict rejection at low $\epsilon$, an optimal plateau at $\epsilon \in [0.1, 1.0]$, and gradual degradation at high $\epsilon$. Dotted lines indicate base model performance.}
\label{fig:eps_all}
\end{minipage}
\hfill
\begin{minipage}{0.48\linewidth}
\centering
\includegraphics[width=\linewidth]{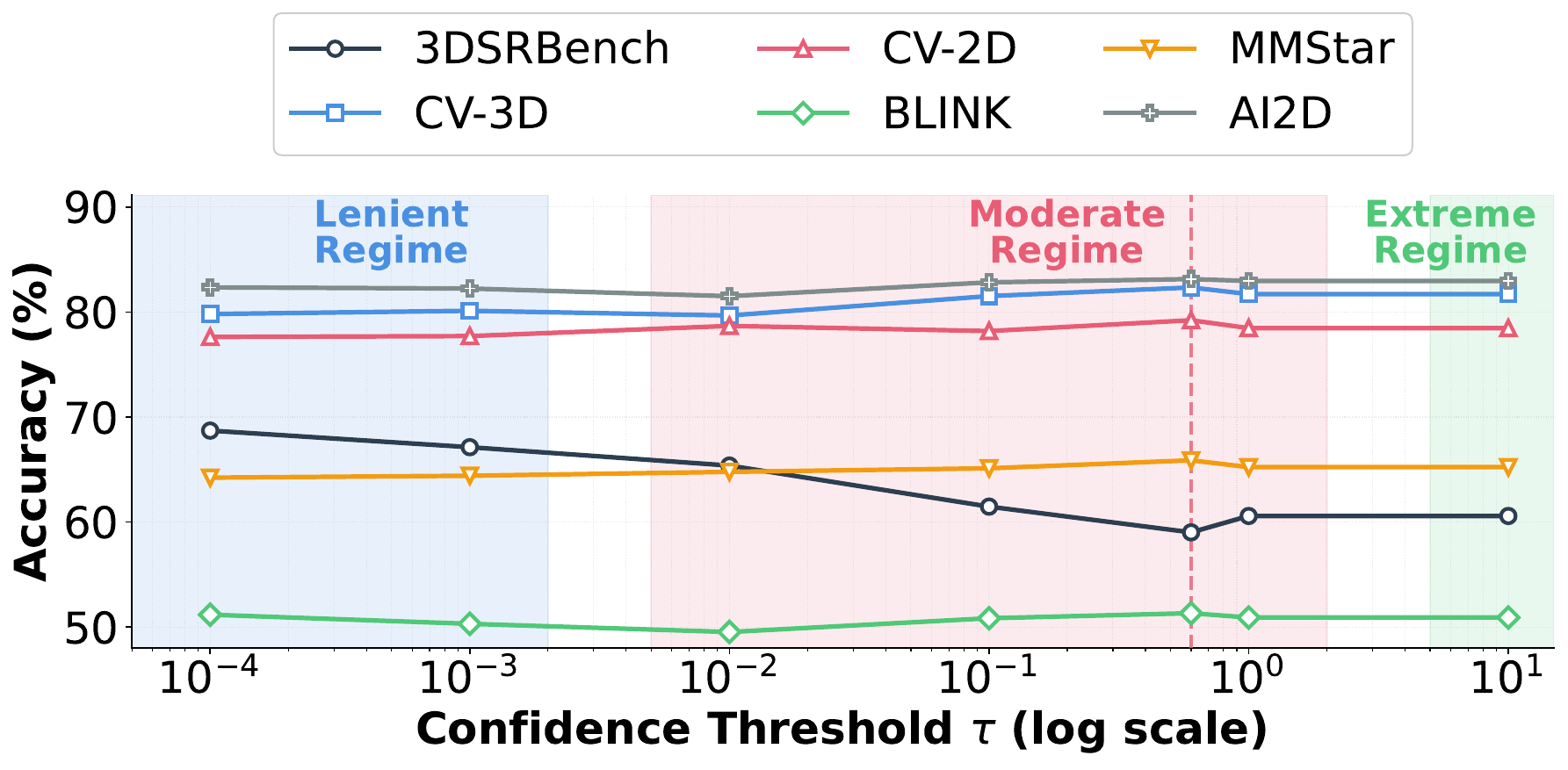}
\caption{Accuracy vs.\ confidence threshold $\tau$ across all six benchmarks ($\epsilon=0.1$ fixed). Five benchmarks peak at $\tau=0.6$; 3DSRBench is the sole exception, where the permissive regime dominates. Horizontal dashed lines indicate base model performance.}
\label{fig:tau_cross_supp}
\end{minipage}
\vspace{-10pt}
\end{figure*}

\begin{wraptable}{R}{0.48\textwidth}
\centering
\caption{Worst-case improvement over the base model (minimum accuracy gain across benchmarks) at each $\epsilon$.}
\label{tab:epsilon_worstcase}
\small
\begin{tabular}{lc c }
\toprule
$\epsilon$ & Min gain (\%) & Benchmark \\
\midrule
0.001 & $-$0.95 & 3DSRBench \\
0.05  & $+$0.36 & 3DSRBench \\
\rowcolor{gray!15}
\textbf{0.1} & $+$\textbf{1.62} & AI2D \\
0.5   & $+$1.69 & AI2D \\
1.0   & $+$1.65 & AI2D \\
2.5   & $+$1.16 & AI2D \\
3.0   & $+$0.85 & AI2D \\
\bottomrule
\end{tabular}
\end{wraptable}


\noindent Figure~\ref{fig:eps_all} and Table~\ref{tab:epsilon_worstcase} confirm that the three-regime pattern generalizes across benchmarks. Three findings emerge.

\paragraph{The three-regime structure is universal.}
Every benchmark follows the same qualitative pattern: performance rises from the strict regime ($\epsilon < 0.1$) into an optimal plateau ($\epsilon \in [0.1, 1.0]$), then declines in the permissive regime ($\epsilon > 1.0$). Degradation is gradual in both directions because two safety nets operate universally: the fallback ranking ($\bar{s}^* - \Delta^*$) catches failures when filtering is too aggressive, and the confidence threshold ($\bar{s}^* > \tau$) continues rejecting weak candidates when filtering is too lenient.

\paragraph{The optimal plateau spans an order of magnitude.}
Accuracy differences between $\epsilon = 0.1$ and $\epsilon = 1.0$ are small across all benchmarks: 0.76 points on 3DSRBench, 0.26 on CV-Bench-3D, 0.25 on CV-Bench-2D, 0.05 on BLINK, 0.25 on MMStar, and 0.03 on AI2D. This stability indicates that $\epsilon = 0.1$ does not require fine-tuning to be effective.

\paragraph{Sensitivity increases with task difficulty.}
3DSRBench exhibits the widest performance range (5.26 points), consistent with the inherent cross-modal tension in 3D spatial reasoning. AI2D shows the narrowest range (1.28 points), reflecting its high baseline accuracy and modest \textbf{VERDICT} gain (+1.62).

\noindent Table~\ref{tab:epsilon_worstcase} reports worst-case robustness: the minimum improvement over the base model at each $\epsilon$. The worst-case benchmark is 3DSRBench for $\epsilon \leq 0.05$ (where overly strict filtering removes informative candidates) and AI2D for $\epsilon \geq 0.1$ (reflecting its small overall gain). The chosen default $\epsilon = 0.1$ provides near-optimal worst-case robustness while sitting at the conservative end of the plateau, which is appropriate for deployment without prior knowledge of task distribution.

\noindent The complementary safety-net structure is confirmed across benchmarks: even at $\epsilon = 0.001$, five of six benchmarks remain above baseline; when $\epsilon$ is too permissive, the confidence threshold limits degradation to 1-2 points. This mutual reinforcement explains why no benchmark exhibits a sharp performance cliff at any $\epsilon$ value.

%% file: suppl_sections/additional_results/tau_analysis.tex
\subsection{Cross-Benchmark Confidence Threshold Analysis}
\label{sec:supp_tau_cross}

Section~6.4 of the main text presents the $\tau$ sensitivity analysis on 3DSRBench. Here we extend this analysis to all six benchmarks, addressing whether the chosen operating point $\tau=0.6$ is justified given that the permissive regime ($\tau \leq 0.01$) achieves substantially higher accuracy on 3DSRBench.

\noindent \textbf{Setup} We sweep $\tau \in \{10^{-4}, 10^{-3}, 10^{-2}, 10^{-1}, 0.6, 1.0, 10.0\}$ with $\epsilon = 0.1$ fixed, reporting accuracy with standard errors across five random seeds. Results are shown in Figure~\ref{fig:tau_cross_supp}.



\paragraph{Per-benchmark behavior.}
Five of six benchmarks follow a consistent pattern: accuracy increases as $\tau$ rises toward $0.6$. The gains from the permissive regime to $\tau=0.6$ range from $+0.90$ (AI2D) to $+1.65$ (MMStar), reflecting that the confidence threshold catches genuinely weak candidates that the dispersion filter alone does not intercept.

\begin{wraptable}{L}{0.48\textwidth}
\centering
\caption{Worst-case improvement over the base model (minimum accuracy gain across benchmarks) at each $\tau$.}
\label{tab:tau_worstcase}
\begin{tabular}{l c c}
\toprule
$\tau$ & Min gain(\%) & Benchmark \\
\midrule
$10^{-4}$ & $+0.72$ & AI2D \\
$10^{-3}$ & $+0.83$ & AI2D \\
$10^{-2}$ & $+0.99$ & AI2D \\
$10^{-1}$ & $+1.31$ & AI2D \\
\rowcolor{gray!15}
$0.6$ & $\mathbf{+1.62}$ & AI2D \\
$1.0$ & $+1.44$ & AI2D \\
$10.0$ & $+1.44$ & AI2D \\
\bottomrule
\end{tabular}
\end{wraptable}

\paragraph{The Permissive Regime as Evidence for the Consensus Formulation}
\label{sec:supp_permissive_evidence}

The permissive regime on 3DSRBench ($65$--$69\%$ accuracy) provides independent evidence for the consensus formulation. At $\tau \leq 0.01$, the confidence filter is effectively disabled: nearly all candidates pass, and the system relies almost entirely on the dispersion criterion $\Delta^* < \epsilon$. Since $\bar{s}^*$ equals the raw mean by construction, the ranking among accepted candidates is identical to what the Mean baseline would produce. The improvement over Mean ($58.34\% \to 65$--$69\%$) is therefore attributable entirely to the dispersion filter. The dispersion $\Delta^*$ identifies conflicted candidates that raw score variance misses, consistent with Proposition~1 and the Raw Average ablation in Section~6.2.


\noindent The consensus formulation thus contributes through complementary mechanisms at different operating points: dispersion filtering dominates in the permissive regime on high-tension tasks; acceptance calibration through the dual criterion dominates at moderate $\tau$ on lower-tension tasks. At $\tau=0.6$, both mechanisms are active, producing the $2.59$--$2.95$ point gap between \method and Raw Average reported in Section~6.2.


\noindent \textit{Single threshold across benchmarks.}
Per-task threshold selection would undermine the training-free, plug-in design principle that motivates our framework and require held-out validation sets for each benchmark, introducing a data dependency we specifically avoid. Reporting a single fixed configuration across all six benchmarks is a stricter evaluation protocol that demonstrates genuine cross-task generalization.

\noindent \textit{Performance trade-offs at the operating point.}
On 3DSRBench specifically, $\tau=0.6$ leaves performance on the table: $59.02\%$ versus $68.71\%$ at $\tau=10^{-4}$. This is the cost of optimizing for robustness under unknown task distributions rather than per-task maximization. If the deployment distribution is known to be similar to that of 3DSRBench, the threshold should be lowered accordingly. Adaptive threshold selection based on task characteristics remains a promising direction for future work.

%% file: suppl_sections/additional_results/threshold_analysis_ext.tex
\subsection{Threshold Sensitivity: Extended Analysis}
\label{sec:supp_threshold}

The main paper (Section~6.3) summarizes threshold sensitivity on 3DSRBench. Here we provide the full regime analysis for both parameters, contextualizing the cross-benchmark results from Sections~\ref{sec:epsilon_cross_bench} and~\ref{sec:threshold_interaction}.

\paragraph{Confidence threshold $\tau$.}
We sweep $\tau \in \{10^{-4}, 10^{-3}, 10^{-2}, 10^{-1}, 0.6, 1.0, 10.0\}$ with $\epsilon = 0.1$ fixed. Performance separates into three regimes. In the \emph{permissive regime} ($\tau \leq 0.01$), nearly all candidates pass the confidence check, and the 65--69\% accuracy depends primarily on how well consensus scores rank accepted candidates. Since $\bar{s}^*$ equals the raw mean by construction, ranking among accepted candidates is identical to Mean aggregation; the improvement over Mean (58.34\% $\to$ 65--69\%) is attributable entirely to the dispersion filter $\Delta^* < \epsilon$. This provides independent evidence that consensus dispersion captures structure that raw variance fails to capture (Proposition~1, main paper). In the \emph{intermediate regime} ($\tau = 0.1$--$0.6$), performance stabilizes around 60--61\%. Candidates filtered here where judges partially disagree but lack a strong consensus are precisely cases where cross-modal tension carries diagnostic value. On 3DSRBench, removing these borderline candidates is counterproductive because some contentious steps are correct; this explains why $\tau = 0.6$ is individually optimal on five of six benchmarks (Section~\ref{sec:epsilon_cross_bench}) but not on 3DSRBench. In the \emph{restrictive regime} ($\tau \geq 1.0$), no candidates pass (consensus scores are bounded in $[0,1]$), and the system falls back to $\bar{s}^* - \Delta^*$ ranking, achieving $\sim$60.6\%. That comparable accuracy is maintained under pure fallback confirms that the combined criterion provides a reasonable ranking signal even without dual-criterion filtering.

\paragraph{Dispersion tolerance $\epsilon$.}
We sweep $\epsilon \in \{0.001, 0.05, 0.1, 0.5, 1.0, 2.5, 3.0\}$ with $\tau = 0.6$ fixed. Performance peaks in an optimal plateau at $\epsilon \in [0.1, 1.0]$. When $\epsilon$ is \emph{too strict} ($< 0.1$), accuracy degrades to 54--57\% because near-unanimous agreement is unrealistic for spatial reasoning where some cross-modal tension is natural. Many legitimate steps are rejected for dispersion only slightly above the threshold. Performance does not collapse because the fallback mechanism remains active. In the \emph{optimal plateau} ($\epsilon \in [0.1, 1.0]$), accuracy stabilizes around 60--61\%; the filter removes candidates with genuine conflict while admitting those with task-appropriate disagreement. The plateau width ($\sim$1 order of magnitude) indicates insensitivity to exact threshold choice consistent with the cross-benchmark finding that accuracy differences between $\epsilon = 0.1$ and $\epsilon = 1.0$ are small on all benchmarks (Section~\ref{sec:epsilon_cross_bench}). When $\epsilon$ is \emph{too permissive} ($> 1.0$), performance gradually declines to 57--58\% as the filter loses discriminative power, though the confidence threshold continues catching weak candidates.

\paragraph{Interaction and safety nets.}
The two thresholds provide complementary safety nets, as confirmed by the joint analysis (Section~\ref{sec:threshold_interaction}). When $\epsilon$ is too strict, the fallback ranking prevents catastrophic failure; when $\epsilon$ is too permissive, the confidence threshold $\bar{s}^* > \tau$ continues filtering. This mutual reinforcement explains why the framework degrades gracefully across the full parameter range, rather than exhibiting sharp cliffs—a property that holds across all six benchmarks.

\paragraph{Key insight.}
Consensus scores function most naturally as a ranking tool rather than a binary classifier. The permissive regime succeeds by preserving the full candidate pool while leveraging dispersion for filtering; the restrictive regime maintains comparable performance because its fallback incorporates both confidence and disagreement. This explains the main paper's finding (Section~6.1) that removing intelligent ranking (No Selection ablation) costs more than removing filtering (No Rejection), and why the dispersion-reject quadrant in Section~\ref{sec:conf_disp_landscape} contains candidates that confidence-only methods cannot distinguish from accepts.

%% file: suppl_sections/additional_results/heatmap_analysis.tex
\subsection{Joint Threshold Interaction Analysis}
\label{sec:threshold_interaction}

Figures~5 and~6 of the main paper analyze sensitivity to $\tau$ and $\epsilon$ separately. Here we examine whether these thresholds interact by sweeping both jointly on 3DSRBench: all combinations of $\tau \in \{10^{-4}, 10^{-3}, 10^{-2}, 10^{-1}, 0.6, 1.0, 10.0\}$ and $\epsilon \in \{0.001, 0.01, 0.05, 0.1, 0.5, 1.0, 2.5\}$, yielding 49 configurations. Figure~\ref{fig:threshold_heatmap} presents heatmaps of accuracy and acceptance rate. Four findings emerge.

\begin{figure}[t]
\centering
\begin{minipage}{0.48\linewidth}
\centering
\includegraphics[width=\linewidth]{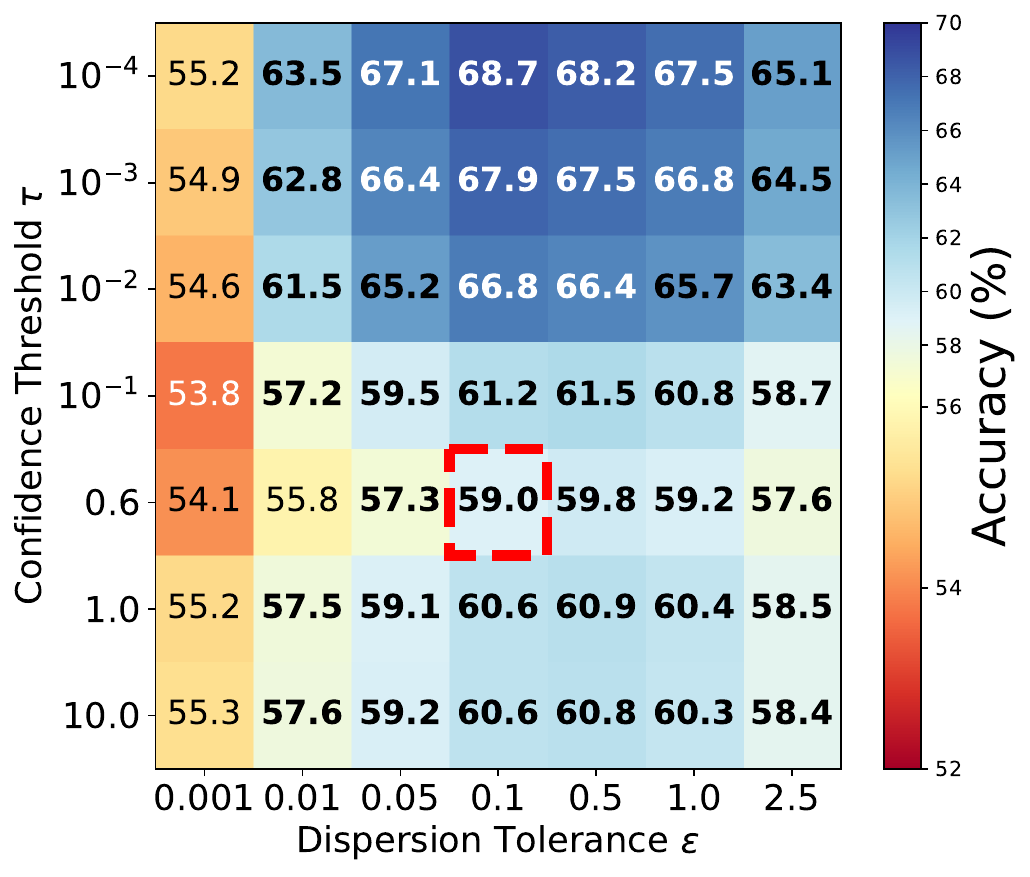}
\end{minipage}
\hfill
\begin{minipage}{0.48\linewidth}
\centering
\includegraphics[width=\linewidth]{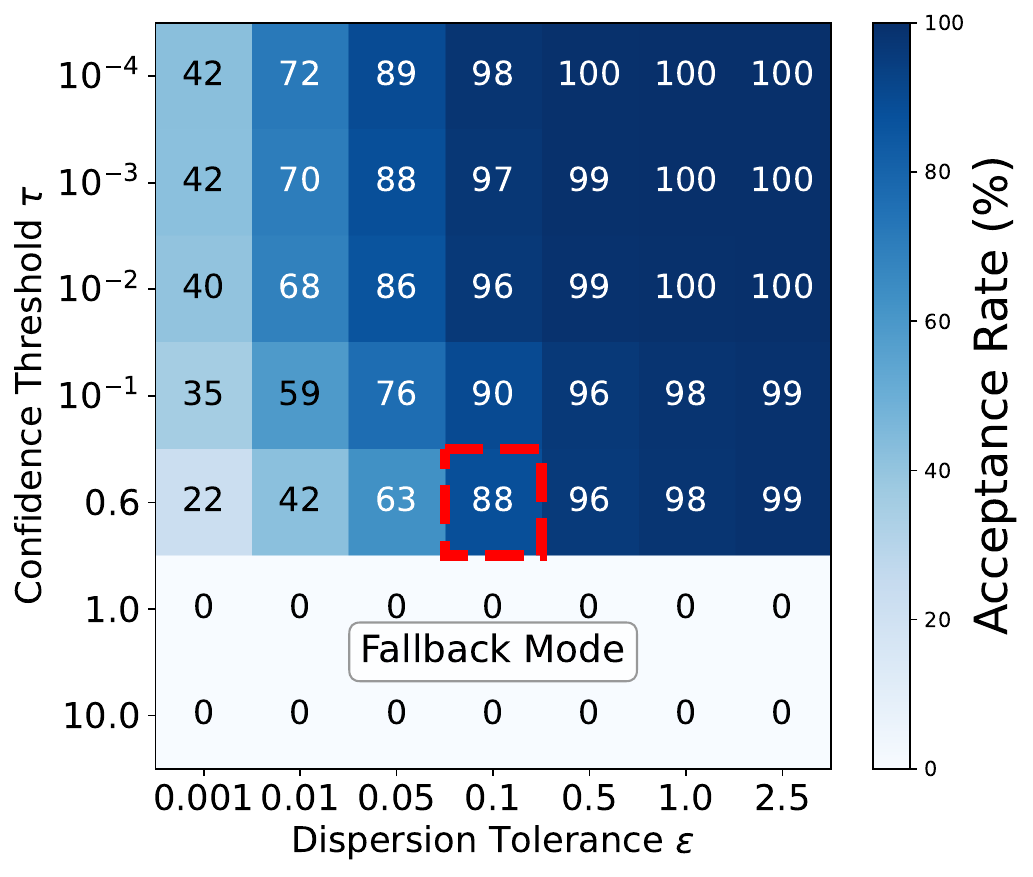}
\end{minipage}
\caption{Joint threshold analysis on 3DSRBench. \textbf{Left:} Accuracy (\%) across the $(\tau, \epsilon)$ grid. \textbf{Right:} Acceptance rate (\%), i.e., the fraction of steps where at least one candidate passes both criteria. The operating point $(\tau=0.6, \epsilon=0.1)$ lies on a stable plateau; the global optimum at low $\tau$ trades cross-benchmark robustness for 3DSRBench-specific gains.}
\label{fig:threshold_heatmap}
\end{figure}

\paragraph{The operating point lies on a plateau, not an edge.}
Accuracy at $(\tau=0.6, \epsilon=0.1)$ is 59.0\%, with similar values in the surrounding region: 59.8\% at $(\tau=0.6, \epsilon=0.5)$, 60.6\% at $(\tau=1.0, \epsilon=0.1)$, and 57.3\% at $(\tau=0.6, \epsilon=0.05)$. Moving in any direction causes gradual change (1--2 points), confirming that the configuration is robust to small perturbations.

\paragraph{The two thresholds interact asymmetrically.}
At low $\tau$ (permissive confidence filtering), varying $\epsilon$ produces large swings: accuracy ranges from 55.2\% at $\epsilon=0.001$ to 68.7\% at $\epsilon=0.1$ when $\tau=10^{-4}$ a 13.5 point difference. At high $\tau$, the effect is muted: only 5.7 points when $\tau=1.0$. This asymmetry arises because at low $\tau$ nearly all candidates pass the confidence check, making $\epsilon$ the primary filter; at high $\tau$, few candidates pass regardless of dispersion.

\paragraph{The global optimum is at low $\tau$, moderate $\epsilon$.}
The highest accuracy (68.7\%) occurs at $(\tau=10^{-4}, \epsilon=0.1)$, corresponding to the permissive regime where the confidence filter is effectively disabled. As discussed in Section~6.3 of the main paper, this configuration achieves higher accuracy on 3DSRBench but does not generalize across benchmarks. The operating point $(\tau=0.6, \epsilon=0.1)$ sacrifices 9.7 points on 3DSRBench for cross-benchmark robustness.

\paragraph{Complementary safety nets prevent catastrophic failure.}
The leftmost column ($\epsilon=0.001$) shows degraded but not catastrophic accuracy (53--55\%) because the fallback ranking remains active. The rightmost column ($\epsilon=2.5$) shows only mild degradation (57--65\%) because the confidence threshold continues filtering weak candidates. No region falls substantially below baseline (56.12\%) except the extreme corner at $(\tau=0.6, \epsilon=0.001)$.

\noindent The acceptance rate heatmap (right panel) reveals complementary structure. At $\tau \geq 1.0$, acceptance drops to 0\% because $\bar{s}^* \in [0,1]$ by construction, forcing the system to operate entirely through fallback ranking. At the operating point, acceptance is approximately 88\%, consistent with the 15\% fallback rate reported in the main paper. The highest-accuracy region (low $\tau$, moderate $\epsilon$) coincides with acceptance rates above 95\%, indicating that in the permissive regime, verification functions primarily through ranking rather than filtering, which is consistent with the main paper's finding that consensus scores operate optimally as ranking tools.

%% file: suppl_sections/additional_results/judge_scale.tex
\subsection{Cross-Benchmark Judge Scale Analysis}
\label{sec:judge_scale}

Section~6.4 of the main paper shows that the margin between \textbf{VERDICT} and Mean aggregation widens as the judge model scale decreases on 3DSRBench. Here we extend this analysis to all six benchmarks.

\noindent We replace the 7B verification agents (Qwen2.5-VL-7B-Instruct) with their 4B and 2B variants while holding the base reasoner fixed at 7B, isolating judge scoring quality from candidate generation quality. All hyperparameters remain unchanged. Figure~\ref{fig:judge_scale_all} presents accuracy across all benchmarks at each judge scale. Three findings emerge.

\begin{figure}[t]
\centering
\includegraphics[width=\textwidth]{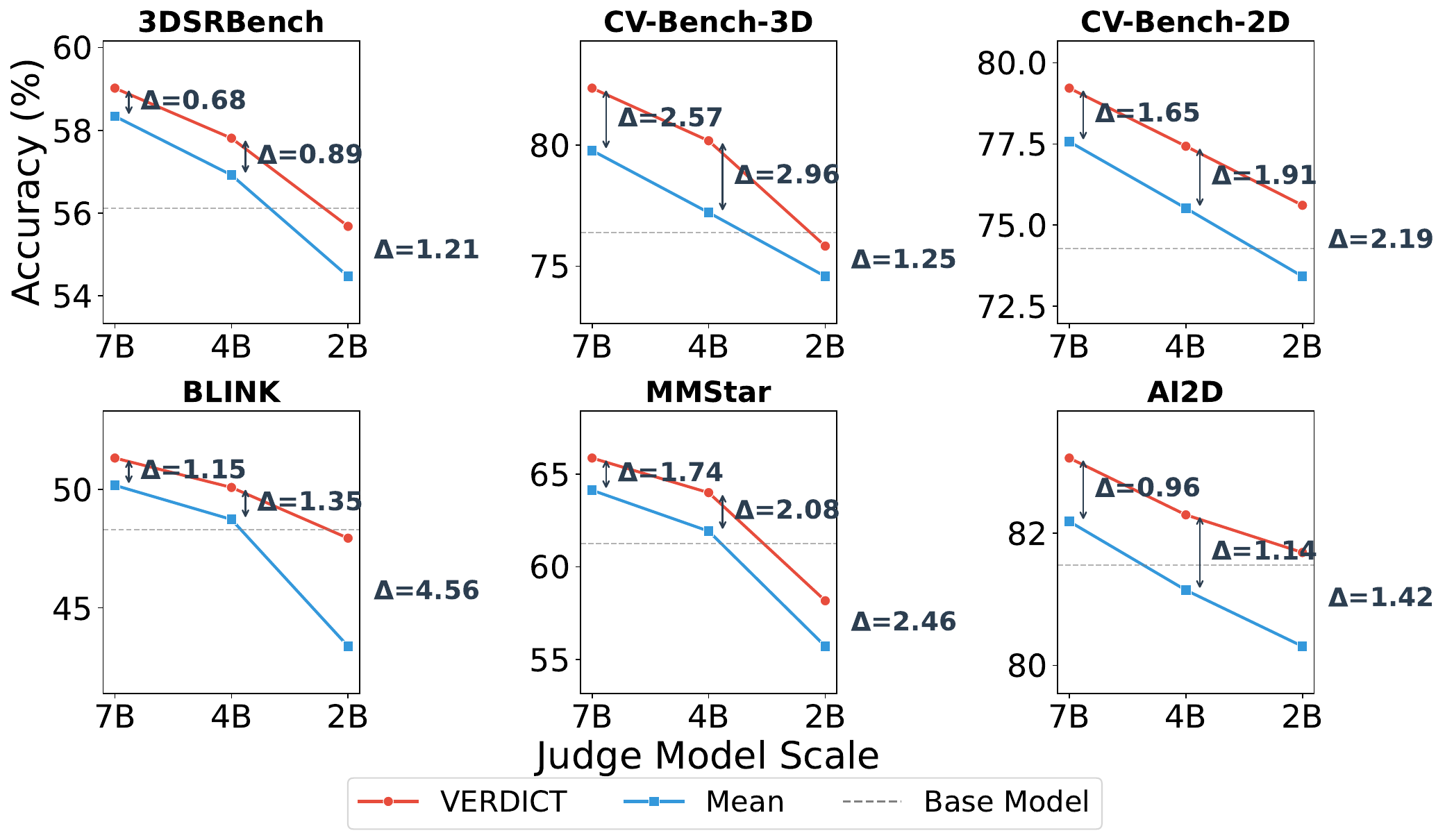}
\caption{Accuracy as a function of judge model scale (7B, 4B, 2B) for \textbf{VERDICT} (red) and Mean aggregation (blue) across all six benchmarks. Dashed lines indicate base model accuracy. $\Delta$ annotations show the margin between methods at each scale.}
\label{fig:judge_scale_all}
\end{figure}

\paragraph{The widening-margin pattern holds on five of six benchmarks.}
On 3DSRBench, CV-Bench-2D, BLINK, MMStar, and AI2D, the gap between \textbf{VERDICT} and Mean increases monotonically as judge scale decreases from 7B to 2B. The effect is particularly pronounced on BLINK, where the margin expands from 1.15 points at 7B to 4.56 points at 2B. 

\paragraph{CV-Bench-3D exhibits a non-monotonic pattern.}
On this benchmark, the margin peaks at 4B ($\Delta = 2.96$) before contracting to 1.25 at 2B. We attribute this to a floor effect: CV-Bench-3D is a depth-ordering task where the base model already performs strongly (76.39\%), and at 2B scale, both methods approach this floor. When judges become sufficiently noisy that their scores carry minimal signal, the consensus computation has less structure to exploit. Notably, even at 2B, \textbf{VERDICT} still outperforms Mean by 1.25 points: the margin contracts but does not invert.

\paragraph{VERDICT remains above or near baseline across all configurations.}
Across the 18 configurations tested (6 benchmarks $\times$ 3 scales), \textbf{VERDICT} exceeds the base model in 16 cases and remains within 0.5 points in the remaining 2. In contrast, Mean aggregation drops below baseline in 5 configurations. This asymmetry confirms that the consensus formulation provides a safety margin against judge degradation that simple averaging does not.

These results have practical implications. In resource-constrained settings with 2B-scale judges, \textbf{VERDICT} recovers 1.2--4.6 points over naive averaging depending on the task. The widening margin also suggests that investing in the consensus algorithm is complementary to, not a substitute for, judge quality: absolute accuracy at 7B exceeds that at 2B on all benchmarks, and the widening margin reflects that Mean degrades faster rather than smaller judges being preferable. Finally, on benchmarks where the base model is already strong (CV-Bench-3D, AI2D), margin expansion is more modest and floor effects may emerge; on challenging benchmarks (BLINK, MMStar), the consensus formulation provides substantial and increasing value as judge quality degrades.

%% file: suppl_sections/additional_results/cb_stubborness_sens.tex
\subsection{Cross-Benchmark Stubbornness Sensitivity}
\label{sec:cross_bench_stubbornness}

Section~6.2 of the main paper ablates each stubbornness parameter on 3DSRBench. A natural concern is whether the chosen configuration \\
$(\lambda_V{=}1.5,\; \allowbreak\ \lambda_L{=}1.0,\; \lambda_C{=}0.8)$ is implicitly tuned to that benchmark. We address this by repeating the same ablation on CV-Bench-3D (where \textbf{VERDICT} achieves its largest gain, $+5.95$) and BLINK (where perceptual reasoning differs qualitatively from spatial reasoning). Each curve varies one parameter over $\{0.3, 0.8, 1.0, 1.5, 2.0\}$ while holding the others fixed.

\begin{table}[t]
\centering
\caption{Stubbornness sensitivity on CV-Bench-3D (base: 76.39\%, Mean: 79.77\%) and BLINK (base: 48.31\%, Mean: 50.17\%). Each row varies one parameter; others remain at defaults. Bold indicates peak accuracy.}
\label{tab:stub_combined}
\small
\begin{tabular}{l l c c c c c c}
\toprule
& & $\lambda = 0.3$ & $\lambda = 0.8$ & $\lambda = 1.0$ & $\lambda = 1.5$ & $\lambda = 2.0$ & Range \\
\midrule
\multirow{3}{*}{\rotatebox{90}{\scriptsize CV-3D}}
& $\lambda_V$ & $79.61{\scriptstyle\pm 0.62}$ & $80.89{\scriptstyle\pm 0.55}$ & $81.52{\scriptstyle\pm 0.53}$ & $\mathbf{82.34}{\scriptstyle\pm 0.51}$ & $81.87{\scriptstyle\pm 0.54}$ & 2.73 \\
& $\lambda_L$ & $81.17{\scriptstyle\pm 0.58}$ & $81.95{\scriptstyle\pm 0.53}$ & $\mathbf{82.34}{\scriptstyle\pm 0.51}$ & $82.08{\scriptstyle\pm 0.55}$ & $81.71{\scriptstyle\pm 0.57}$ & 1.17 \\
& $\lambda_C$ & $81.58{\scriptstyle\pm 0.56}$ & $\mathbf{82.34}{\scriptstyle\pm 0.51}$ & $82.21{\scriptstyle\pm 0.52}$ & $81.93{\scriptstyle\pm 0.54}$ & $81.72{\scriptstyle\pm 0.56}$ & 0.76 \\
\midrule
\multirow{3}{*}{\rotatebox{90}{\scriptsize BLINK}}
& $\lambda_V$ & $49.53{\scriptstyle\pm 0.56}$ & $50.47{\scriptstyle\pm 0.51}$ & $50.86{\scriptstyle\pm 0.49}$ & $\mathbf{51.32}{\scriptstyle\pm 0.48}$ & $50.91{\scriptstyle\pm 0.50}$ & 1.79 \\
& $\lambda_L$ & $50.58{\scriptstyle\pm 0.53}$ & $51.04{\scriptstyle\pm 0.50}$ & $\mathbf{51.32}{\scriptstyle\pm 0.48}$ & $51.11{\scriptstyle\pm 0.50}$ & $50.79{\scriptstyle\pm 0.52}$ & 0.74 \\
& $\lambda_C$ & $50.83{\scriptstyle\pm 0.51}$ & $\mathbf{51.32}{\scriptstyle\pm 0.48}$ & $51.19{\scriptstyle\pm 0.49}$ & $50.97{\scriptstyle\pm 0.50}$ & $50.75{\scriptstyle\pm 0.52}$ & 0.57 \\
\bottomrule
\end{tabular}
\end{table}

\noindent Table~\ref{tab:stub_combined} presents results for both benchmarks. The sensitivity hierarchy mirrors 3DSRBench on both: $\lambda_V$ has the widest performance range (2.73 points on CV-Bench-3D, 1.79 on BLINK), followed by $\lambda_L$ (1.17, 0.74) and $\lambda_C$ (0.76, 0.57). All curves peak at the default configuration.

\noindent The $\lambda_V$ curve on CV-Bench-3D is particularly informative. At $\lambda_V = 0.3$, the visual agent places 77\% of its weight on group consensus (Eq.~(2), main paper), suppressing its private visual evidence. On a depth-ordering task where visual grounding is the primary signal, this dilution costs 2.73 points. At $\lambda_V = 2.0$, the agent becomes overly stubborn, resisting legitimate corrections when its initial score is noisy; performance drops 0.47 points. The optimal $\lambda_V = 1.5$ balances these failure modes.

\noindent The tighter ranges on BLINK are expected: the overall \textbf{VERDICT} gain is smaller ($+3.01$ vs.\ $+5.95$ on CV-Bench-3D), leaving less room for parameter-induced variation. The relative sensitivity pattern, however, is unchanged.

\begin{table}[t]
\centering
\caption{Sensitivity ranges across three benchmarks. The hierarchy $\lambda_V > \lambda_L > \lambda_C$ is preserved on every benchmark, and all 45 configurations exceed the base model.}
\label{tab:stub_summary}
\small
\begin{tabular}{l c c c c}
\toprule
Benchmark & $\lambda_V$ range & $\lambda_L$ range & $\lambda_C$ range & All $>$ Base \\
\midrule
3DSRBench   & 2.90 & 1.40 & 1.19 & \checkmark \\
CV-Bench-3D & 2.73 & 1.17 & 0.76 & \checkmark \\
BLINK       & 1.79 & 0.74 & 0.57 & \checkmark \\
\bottomrule
\end{tabular}
\end{table}

\noindent Table~\ref{tab:stub_summary} summarizes the cross-benchmark pattern. Three conclusions follow.

\paragraph{The sensitivity hierarchy is benchmark-independent.}
$\lambda_V$ is the most impactful parameter on every benchmark, $\lambda_C$ the least. This ordering reflects the agent structure rather than dataset-specific tuning: on tasks requiring visual grounding, the visual agent's stubbornness determines how much evidence survives consensus pressure.

\paragraph{The default configuration is jointly optimal.}
Across all nine curves (three parameters $\times$ three benchmarks), peak accuracy occurs at the default $(\lambda_V{=}1.5,\; \allowbreak\ \lambda_L{=}1.0,\; \lambda_C{=}0.8)$. 


%% file: suppl_sections/additional_results/stubborness_permute.tex
\subsection{Cross-Benchmark Stubbornness Assignment Analysis}
\label{sec:cross_bench_swaps}

Section~6.2 of the main paper tests whether the \emph{assignment} of stubbornness values to agents matters by permuting the triple $(\lambda_V{=}1.5,\; \lambda_L{=}1.0,\; \lambda_C{=}0.8)$ via pairwise swaps on 3DSRBench. Here we repeat the same protocol on CV-Bench-3D and BLINK to verify that the ordering $\lambda_V > \lambda_L > \lambda_C$ is not benchmark-specific.

\noindent
Each pairwise swap exchanges the stubbornness values of two agents: L$\leftrightarrow$C leaves $\lambda_V = 1.5$ unchanged, V$\leftrightarrow$L reduces $\lambda_V$ to 1.0, and V$\leftrightarrow$C reduces $\lambda_V$ to 0.8. The three swaps form a gradient of increasing disruption to the visual agent's stubbornness. Table~\ref{tab:swap_results} reveals three consistent patterns.

\begin{table}[t]
\centering
\caption{Pairwise stubbornness swap results. $\Delta$: accuracy drop from default. Bold: configurations where $\lambda_V$ retains the highest value.}
\label{tab:swap_results}
\small
\begin{tabular}{l r r r r r r}
\toprule
& \multicolumn{2}{c}{3DSRBench} & \multicolumn{2}{c}{CV-Bench-3D} & \multicolumn{2}{c}{BLINK} \\
\cmidrule(lr){2-3} \cmidrule(lr){4-5} \cmidrule(lr){6-7}
Configuration & Acc.\,(\%) & $\Delta$ & Acc.\,(\%) & $\Delta$ & Acc.\,(\%) & $\Delta$ \\
\midrule
Default & 59.02 & --- & 82.34 & --- & 51.32 & --- \\
\textbf{L$\leftrightarrow$C} & \textbf{58.61} & $-$0.41 & \textbf{81.87} & $-$0.47 & \textbf{50.98} & $-$0.34 \\
V$\leftrightarrow$L & 57.83 & $-$1.19 & 80.51 & $-$1.83 & 50.38 & $-$0.94 \\
V$\leftrightarrow$C & 57.26 & $-$1.76 & 79.93 & $-$2.41 & 49.91 & $-$1.41 \\
\midrule
Mean baseline & 58.34 & & 79.77 & & 50.17 & \\
Base model & 56.12 & & 76.39 & & 48.31 & \\
\bottomrule
\end{tabular}
\end{table}

\paragraph{The drop ordering is benchmark-independent.}
On every benchmark, V$\leftrightarrow$C produces the largest degradation, V$\leftrightarrow$L an intermediate one, and L$\leftrightarrow$C the smallest. The ordering holds across 3DSRBench ($1.76 > 1.19 > 0.41$), CV-Bench-3D ($2.41 > 1.83 > 0.47$), and BLINK ($1.41 > 0.94 > 0.34$). Since these benchmarks test qualitatively different capabilities, the consistency confirms that the assignment ordering reflects agent architecture rather than dataset-specific features.

\paragraph{Degradation tracks the $\lambda_V$ reduction.}
Each swap changes two parameters, yet drop magnitude correlates almost entirely with how much $\lambda_V$ decreases. L$\leftrightarrow$C leaves $\lambda_V = 1.5$ untouched and costs under 0.5 points on every benchmark; V$\leftrightarrow$C reduces $\lambda_V$ to 0.8 and costs 1.4--2.4 points. When the visual agent becomes too responsive to consensus pressure, it concedes ground on perceptually grounded judgments it should have defended.

\paragraph{$\lambda_V$ as the highest value is the critical invariant.}
On all three benchmarks, L$\leftrightarrow$C (which preserves $\lambda_V$ as highest) exceeds the Mean baseline: 58.61 vs.\ 58.34 on 3DSRBench, 81.87 vs.\ 79.77 on CV-Bench-3D, 50.98 vs.\ 50.17 on BLINK. When $\lambda_V$ is no longer the largest, performance approaches or crosses the Mean baseline. On BLINK, V$\leftrightarrow$C (49.91\%) dips below Mean (50.17\%); on 3DSRBench, both V$\leftrightarrow$L (57.83\%) and V$\leftrightarrow$C (57.26\%) fall below Mean (58.34\%). CV-Bench-3D is the exception: even V$\leftrightarrow$C (79.93\%) stays marginally above Mean (79.77\%), reflecting the larger absolute gap on this benchmark.

\noindent
All twelve swap configurations exceed the unverified base model. Even V$\leftrightarrow$C on 3DSRBench (57.26\%) outperforms baseline (56.12\%) by 1.14 points. The consensus formulation provides value even under misassigned stubbornness, though correct assignment is needed to outperform simpler aggregation.

\noindent
Swaps change two parameters simultaneously, so they produce larger drops than single-parameter changes. On CV-Bench-3D, setting $\lambda_V = 0.8$ alone yields 80.89\% (Table~\ref{tab:stub_combined}), while V$\leftrightarrow$C ($\lambda_V = 0.8$, $\lambda_C = 1.5$) yields 79.93\%---the additional off-optimal $\lambda_C$ costs 0.96 points. On BLINK: 50.47\% vs.\ 49.91\%, a 0.56 point additional cost. Misassignment penalties compound when multiple agents are simultaneously mis-specified.

%% file: suppl_sections/additional_results/dispersion_as_diagonistic.tex
\subsection{Dispersion as a Diagnostic Signal}
\label{sec:dispersion_diagnostic}

\begin{figure}[t]
\centering
\includegraphics[width=\linewidth]{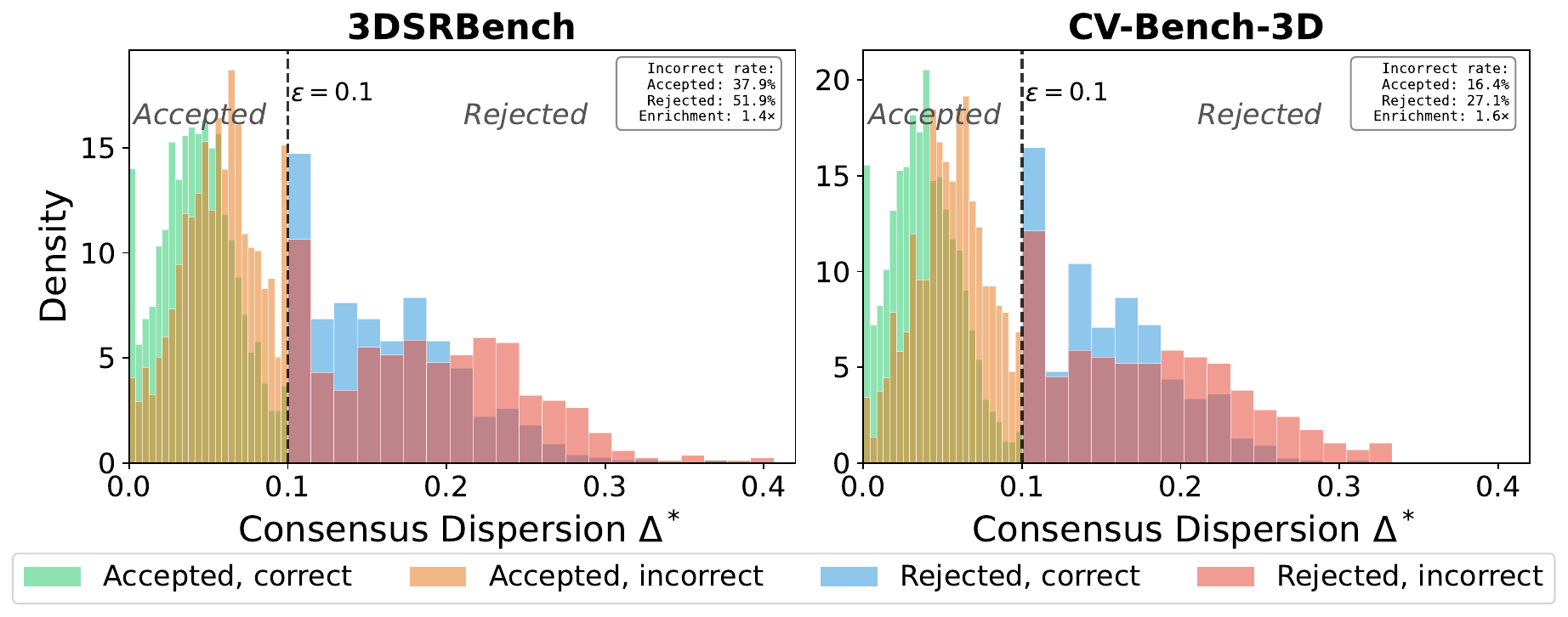}
\caption{Distribution of consensus dispersion $\Delta^*$ across candidate reasoning steps, partitioned by acceptance status and chain correctness. Dashed line: $\epsilon = 0.1$. Inset: incorrect-chain rates among accepted vs.\ rejected candidates.}
\label{fig:dispersion_dist}
\end{figure}

Proposition~1 establishes that consensus dispersion $\Delta^*$ captures cross-agent interaction structure that no separable measure can replicate. Here, we provide empirical evidence that this structure is diagnostically meaningful: high dispersion is associated with reasoning steps that lead to incorrect answers.


\noindent
For each candidate step evaluated during inference, we record its $\Delta^*$ and whether the complete chain produces the correct final answer. We partition candidates by (1) acceptance status under the dual criterion (Eq.~(4), main paper) and (2) chain correctness. We report results on 3DSRBench and CV-Bench-3D, representing the lowest and highest \textbf{VERDICT} gains, respectively.

\noindent Figure~\ref{fig:dispersion_dist} shows the distribution of $\Delta^*$ across the four groups. Three observations follow.

\paragraph{High dispersion is enriched for incorrect chains.}
On 3DSRBench, 51.9\% of rejected candidates belong to chains producing incorrect answers, compared to 37.9\% among accepted candidates a 1.4$\times$ enrichment. On CV-Bench-3D, enrichment is 1.6$\times$ (27.1\% vs.\ 16.4\%). The higher baseline accuracy on CV-Bench-3D means absolute incorrect rates are lower, but relative enrichment is stronger, indicating that the dispersion filter is more selective.

\begin{wrapfigure}{r}{0.40\textwidth}
\centering
\includegraphics[width=\linewidth]{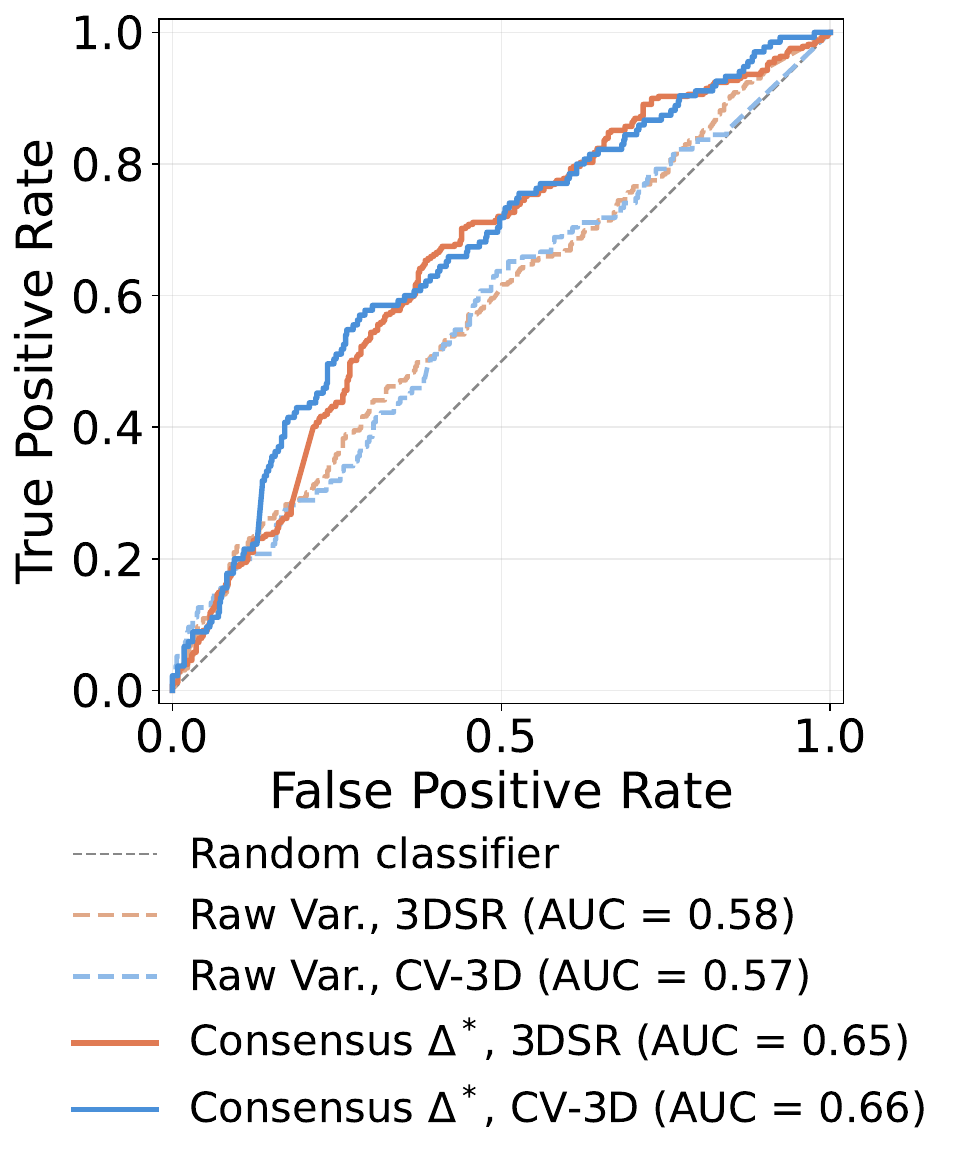}
\caption{\textbf{ROC curves} treating consensus dispersion $\Delta^*$ (solid) and raw score variance (dashed) as a binary classifier for chain incorrectness on 3DSRBench and CV-Bench-3D.}
\label{fig:roc_pr}
\end{wrapfigure}

\paragraph{Accepted candidates have substantially lower dispersion.}
Mean $\Delta^*$ among accepted candidates is 0.048 on 3DSRBench and 0.041 on CV-Bench-3D, compared to 0.175 and 0.160 for rejected candidates a 3.6$\times$ and 3.9$\times$ separation, respectively. This confirms that $\epsilon = 0.1$ sits in a natural gap between populations rather than imposing an arbitrary split.

\paragraph{The filter is diagnostic, not oracular.}
Rejected candidates include correct-chain members on both benchmarks: 48.1\% on 3DSRBench, 72.9\% on CV-Bench-3D. This overlap confirms the main paper's finding that consensus scores function optimally as ranking tools rather than binary classifiers. High dispersion signals unresolved cross-modal conflict a risk factor for error, not a guarantee. The fallback mechanism ($\bar{s}^* - \Delta^*$ ranking) recovers value from rejected candidates by selecting the one with the best balance of confidence and agreement.

\noindent The overall acceptance rates (77.9\% on 3DSRBench, 85.4\% on CV-Bench-3D) reflect the inherent cross-modal tension in 3D spatial reasoning, where visual and logical agents more frequently disagree, exactly where dispersion-based filtering provides the most value. The moderate enrichment (1.4- 1.6$\times$) is consistent with two properties of the framework: \textbf{VERDICT} operates at the step level, so a single high-dispersion step does not deterministically cause an incorrect answer since subsequent steps may recover; and the coupled scoring formulation already compresses raw disagreement through consensus adjustment, so residual dispersion represents conflict that persists even after accounting for consensus pressure; a more conservative signal than raw variance.

\noindent \textbf{Quantifying diagnostic value via ROC analysis.} To move beyond enrichment ratios and provide a concrete quantitative handle on how useful $\Delta^*$ is as a signal, we treat consensus dispersion as a continuous-valued binary classifier for chain incorrectness: a candidate is predicted ``incorrect'' if $\Delta^*$ exceeds a given threshold, and ``correct'' otherwise. To isolate the contribution of the coupled consensus formulation, we  also evaluate raw score variance used by the Variance baseline in Table~2 of the main paper as a competing classifier under the same protocol. Figure~\ref{fig:roc_pr} presents ROC curves on 3DSRBench and CV-Bench-3D for both signals. Three observations are consistent with the analysis above.

\noindent \emph{Dispersion carries genuine diagnostic signal.}
The ROC AUC is 0.65 on 3DSRBench and 0.66 on CV-Bench-3D, comfortably above 
the 0.50 random baseline, confirming that higher dispersion is systematically 
associated with incorrect chains, indicating that dispersion-based ranking concentrates incorrect chains above chance at every 
recall level.

\noindent \emph{Consensus dispersion strictly dominates raw variance.} Raw score variance achieves AUC of 0.58 on 3DSRBench and 0.57 on CV-Bench-3D:  above chance, but consistently below consensus dispersion by +0.07 and +0.09 
respectively.. The result directly visualises the non-separability established by Proposition~1: because variance treats all agent deviations symmetrically, it cannot distinguish a high-$\lambda$ visual agent dissenting (diagnostically informative on spatial tasks) from a low-$\lambda$ contextual agent dissenting (less informative). The coupled consensus formulation encodes this asymmetry through the stubbornness parameters, and the AUC gap quantifies how much diagnostic power that asymmetry contributes. 

\noindent \emph{The signal is moderate, not oracular.}
AUC values in the 0.65-0.66 range confirm the main text's characterization: 
$\Delta^*$ is a risk factor for error, not a deterministic indicator.This is expected for two reasons already noted: \textbf{VERDICT} operates at the step level, so a single high-dispersion step does not deterministically 
cause an incorrect answer; and the coupled scoring formulation already compresses raw disagreement, so residual dispersion represents conflict that persists after consensus adjustment.The moderate AUC also explains why the framework degrades gracefully rather 
than exhibiting sharp cliffs across the threshold sweep in Section~S4.2: a diagnostic-but-not-oracular signal benefits from soft 
integration (ranking and dual filtering) rather than hard binary classification.

%% file: suppl_sections/additional_results/confidence_dispersion.tex
\subsection{Confidence--Dispersion Landscape}
\label{sec:conf_disp_landscape}

Here we investigate the dual acceptance criterion by plotting mean consensus confidence $\bar{s}^*$ against consensus dispersion $\Delta^*$ for all candidate reasoning steps on 3DSRBench, color-coded by chain correctness.

\noindent For each candidate, we record $(\bar{s}^*,\, \Delta^*)$ from the coupled scoring system (Eq.~(3), main paper) and label it by whether the full chain produces the correct answer. The dual criterion partitions this plane into four quadrants: \emph{Accept} ($\bar{s}^* > \tau$, $\Delta^* < \epsilon$), \emph{Reject on dispersion} ($\bar{s}^* > \tau$, $\Delta^* \geq \epsilon$), \emph{Reject on confidence} ($\bar{s}^* \leq \tau$, $\Delta^* < \epsilon$), and \emph{Reject on both}. Figure~\ref{fig:conf_disp_scatter} reveals four features of the decision landscape.

\paragraph{The accept region is enriched for correct chains.}
Of 623 candidates in the accept quadrant ($\bar{s}^* > 0.6$, $\Delta^* < 0.1$), 62\% belong to correct chains, compared to 59\% overall. This region captures candidates with both high collective confidence and low residual disagreement, exactly the verification state the dual criterion targets.

\paragraph{The dispersion-reject quadrant isolates the key failure mode.}
109 candidates (13.6\%) pass the confidence threshold but fail the dispersion check. These are candidates that confidence-only methods would accept, but where agents have not genuinely converged. The correct-chain rate drops to 54\%, confirming that high dispersion at high confidence is a risk factor. The starred annotation marks the main paper's numerical example ($\hat{\mathbf{s}} = (0.9, 0.3, 0.9)$, yielding $\bar{s}^* = 0.71$, $\Delta^* = 0.13$): a candidate that looks promising by confidence alone but is correctly flagged by dispersion.

\paragraph{The confidence-reject and both-reject quadrants are sparse.}
Only 4.4\% and 4.1\% of candidates fall in these quadrants, respectively, representing cases where collective endorsement is genuinely low. The both-reject quadrant has the lowest correct-chain rate (27\%), indicating that candidates with both low confidence and high disagreement are rarely on viable reasoning paths.

\begin{wrapfigure}{r}{0.65\textwidth}
\centering
\includegraphics[width=\linewidth]{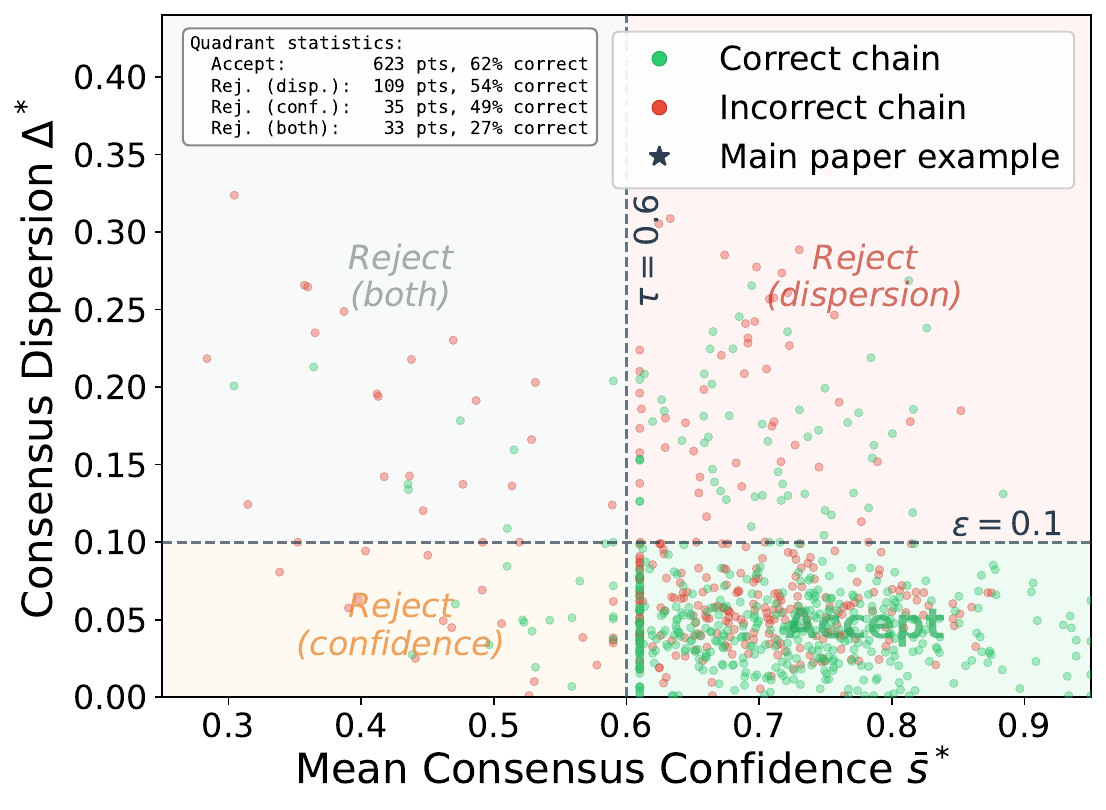}
\caption{Mean consensus confidence $\bar{s}^*$ vs.\ consensus dispersion $\Delta^*$ for candidate steps on 3DSRBench. Green: correct chains; red: incorrect chains. Dashed lines: $\tau = 0.6$, $\epsilon = 0.1$. Inset: per-quadrant counts and correct-chain rates.}
\label{fig:conf_disp_scatter}
\end{wrapfigure}

\paragraph{The separation is probabilistic, not deterministic.}
All quadrants contain both green and red points, visually confirming that consensus scores function optimally as ranking tools rather than binary classifiers. A candidate in the dispersion-reject quadrant is at elevated risk but not guaranteed to be wrong which is why the fallback mechanism ($\bar{s}^* - \Delta^*$ ranking) recovers value from rejected candidates when no alternative passes.

\noindent
The scatter plot provides spatial grounding for three main paper claims: the visual separation confirms that high $\Delta^*$ is diagnostic of reasoning instability (Proposition~1); the 77.9\% concentration in the accept region is consistent with the $\sim$15\% fallback rate; and the dispersion-reject quadrant is indistinguishable from accept under simple averaging since both have $\bar{s}^* > \tau$ which demonstrates the dual criterion's operational value, catching 109 candidates (46\% incorrect) that confidence-only methods would admit.

%% file: suppl_sections/additional_results/rejection_v_sampling.tex
\subsection{Rejection and Selection: Detailed Analysis}
\label{sec:supp_rejection_selection}

Here, we examine the underlying mechanism through
candidate acceptance statistics. Table~\ref{tab:supp_rejection_selection}
presents acceptance counts, fallback rates, and accuracy jointly.

\begin{table}[t]
\centering
\caption{Ablation analysis decomposing rejection (filtering) and selection
(ranking). \textbf{Acc.}: average candidates accepted per step (of~3).
\textbf{Rej.\%}: fraction of steps where all candidates are rejected,
triggering fallback to continuous ranking.}
\label{tab:supp_rejection_selection}
\resizebox{\textwidth}{!}{%
\begin{tabular}{l cc cc cc cc}
\toprule
& \multicolumn{2}{c}{\textbf{CV-Bench-2D}} & \multicolumn{2}{c}{\textbf{CV-Bench-3D}} & \multicolumn{2}{c}{\textbf{3DSRBench}} & \multicolumn{2}{c}{\textbf{AI2D}} \\
\cmidrule(lr){2-3} \cmidrule(lr){4-5} \cmidrule(lr){6-7} \cmidrule(lr){8-9}
Strategy & Acc./Rej.\% & Accuracy & Acc./Rej.\% & Accuracy & Acc./Rej.\% & Accuracy & Acc./Rej.\% & Accuracy \\
\midrule
\method     & 2.64~/~9.1\%  & 79.22 & 2.84~/~2.3\%  & 82.34 & 2.54~/~12.0\% & 59.02 & 2.50~/~13.4\% & 83.14 \\
No Rejection  & 3.00~/~0.0\%  & 78.14 & 3.00~/~0.0\%  & 82.08 & 3.00~/~0.0\%  & 60.31 & 3.00~/~0.0\%  & 82.17 \\
No Selection  & 2.64~/~9.1\%  & 77.05 & 2.84~/~2.3\%  & 80.42 & 2.54~/~12.0\% & 57.84 & 2.50~/~13.4\% & 82.55 \\
Raw Average   & 2.40~/~16.2\% & 76.27 & 2.70~/~5.4\%  & 79.75 & 2.31~/~17.2\% & 56.92 & 2.35~/~17.6\% & 81.83 \\
Random        & 3.00~/~0.0\%  & 75.38 & 3.00~/~0.0\%  & 77.81 & 3.00~/~0.0\%  & 56.41 & 3.00~/~0.0\%  & 81.64 \\
\bottomrule
\end{tabular}%
}
\end{table}

\paragraph{Rejection frequency scales with benchmark difficulty.}
Filtering activates non-uniformly: only 2.3\% of steps trigger full rejection
on CV-Bench-3D (strong base model, 76.39\%) versus 13.4\% on AI2D
(hallucination-prone reasoning). This directly explains why removing rejection
costs 0.97 points on AI2D but only 0.26 on CV-Bench-3D.

\paragraph{Filtering and ranking address complementary failure modes.}
\method and No Selection share identical acceptance statistics (e.g., 2.64
candidates, 9.1\% fallback on CV-Bench-2D), confirming that selection operates
strictly downstream of filtering. Their accuracy gap (79.22\% vs.\ 77.05\% on
CV-Bench-2D, 82.34\% vs.\ 80.42\% on CV-Bench-3D) isolates the contribution
of intelligent ranking by $\bar{s}^*$. Conversely, comparing \method to No
Rejection isolates filtering's contribution. That both ablations degrade
through different magnitudes on different benchmarks confirms that the two
mechanisms target distinct failure modes.

\paragraph{Consensus calibration accepts more and rejects better.}
\method and Raw Average share the same dual-criterion architecture; they
differ only in whether scores come from coupled consensus or simple averaging.
The consensus formulation accepts 0.14--0.24 more candidates per step while
requiring fallback 3--7 percentage points less often, yet achieves consistently
higher accuracy (e.g., 79.22\% vs.\ 76.27\% on CV-Bench-2D, 59.02\% vs.\
56.92\% on 3DSRBench). This dual advantage indicates that consensus scoring
distinguishes genuine cross-modal conflict from ordinary confidence
fluctuations more effectively than raw averaging.

\paragraph{The 3DSRBench exception.}
On 3DSRBench, No Rejection (60.31\%) slightly outperforms \method (59.02\%)
despite the latter's 12.0\% fallback rate. On spatially demanding tasks where
cross-modal tension is inherent, the dispersion filter removes candidates
carrying informative disagreement. No Rejection preserves these and ranks them
via $\bar{s}^* - \Delta^*$, treating disagreement as a ranking signal rather
than an exclusion criterion. This reinforces the conclusion that \method scores
are most powerful as ranking tools, and that $\tau = 0.6$ is chosen for
cross-benchmark robustness rather than per-task optimality.

%% file: suppl_sections/additional_results/fallback_analysis.tex
\subsection{Fallback Mechanism Analysis}
\label{sec:fallback_analysis}

The main paper notes (Section~6.3) that the fallback path ranking by $\bar{s}^* - \Delta^*$ when no candidate passes the dual criterion is triggered on approximately 15\% of steps. Here, we characterize when the fallback activates and how well it performs.

\noindent
We partition test questions into $\mathcal{Q}_{\text{fallback}}$ (fallback triggered on at least one step) and $\mathcal{Q}_{\text{normal}}$ (fallback never triggered). On $\mathcal{Q}_{\text{fallback}}$, we compare three conditions: (1) unverified base model accuracy, (2) full \textbf{VERDICT} with $\bar{s}^* - \Delta^*$ ranking on fallback steps, and (3) \textbf{VERDICT} with random selection on fallback steps. Comparing (2) and (3) isolates the fallback mechanism's contribution. Figure ~\ref{fig:fallback} reports per-benchmark statistics. Three findings emerge.


\begin{figure}[t]
\centering
\includegraphics[width=0.80\linewidth]{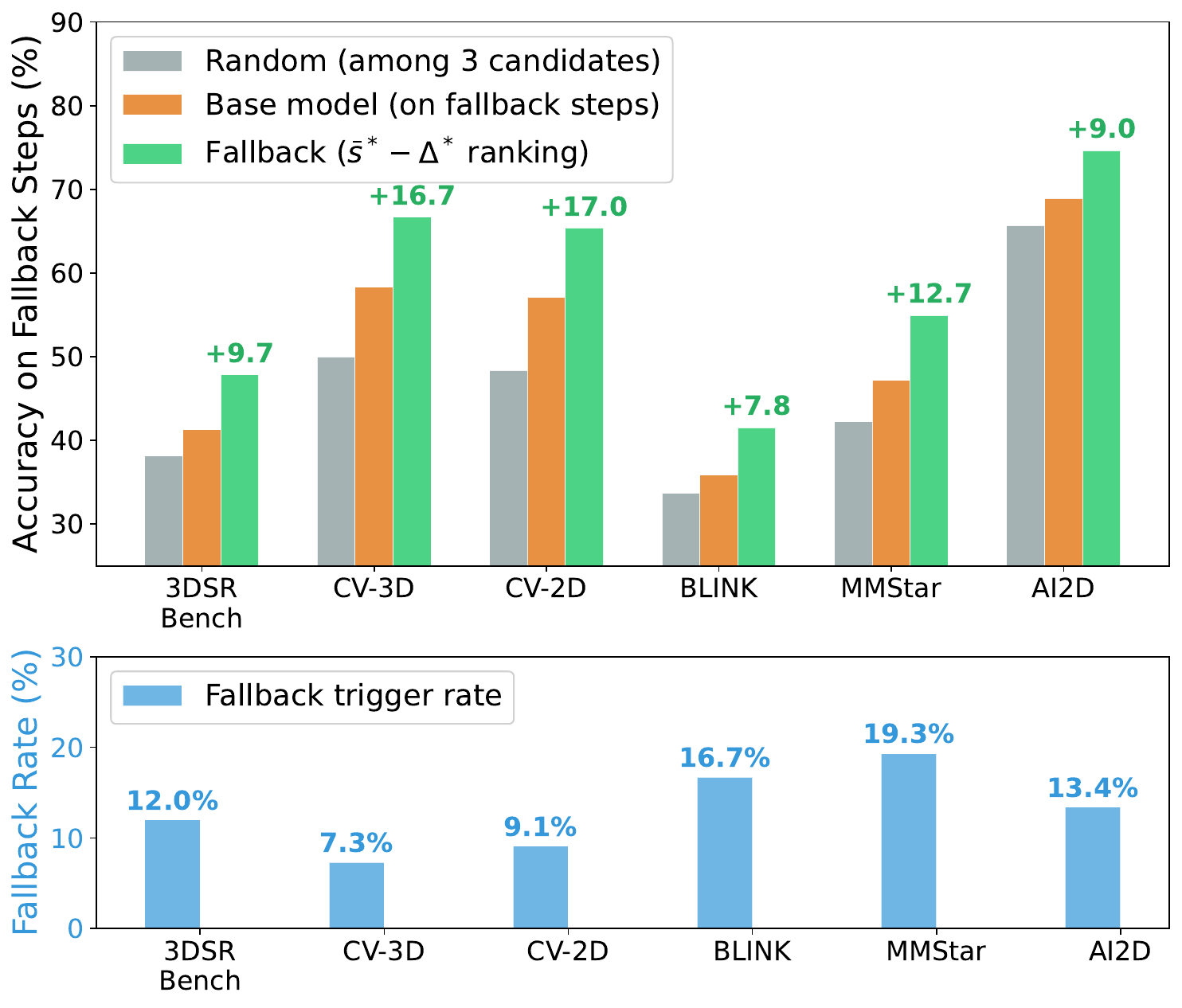}
\caption{\textbf{Top}: Final-answer accuracy on the question subset
$\mathcal{Q}_{\text{fallback}}$ under three conditions: random candidate selection on fallback steps (grey), unverified base model on the same questions (orange), and the $\bar{s}^* - \Delta^*$ fallback ranking (green). Annotations show the improvement of fallback ranking over random.
\textbf{Bottom}: Fallback trigger rate.}
\label{fig:fallback}
\end{figure}

\paragraph{Fallback questions are systematically harder.}
On every benchmark, the base model achieves substantially lower accuracy on $\mathcal{Q}_{\text{fallback}}$ than overall 
Fallback questions correspond to chains encountering at least one step where cross-modal tension is genuinely high, and no candidate achieves simultaneous high confidence and low dispersion. The dual criterion effectively identifies the hardest 10--20\% of questions without ground-truth labels.

\paragraph{The fallback ranking outperforms random selection.}
On every benchmark, $\bar{s}^* - \Delta^*$ ranking outperforms random selection by 7.8--17.0 points. Since all non-fallback steps use identical selection, this improvement is attributable entirely to the ranking signal. The composite score extracts useful information even from rejected candidates, preferring those with higher confidence and lower disagreement consistent with the main paper's finding that consensus scores function optimally as ranking tools.

\paragraph{Relationship to the No Rejection ablation.}
The fallback analysis explains the 3DSRBench anomaly in Section~\ref{sec:supp_rejection_selection}: No Rejection (60.31\%) slightly outperforms \textbf{VERDICT} (59.02\%). On the 12\% of fallback-triggered steps, \textbf{VERDICT} chooses from the rejected pool, achieving 47.83\% on $\mathcal{Q}_{\text{fallback}}$. No Rejection retains all candidates and ranks over the complete pool, which includes candidates the dispersion filter would reject, but that may be viable for spatially demanding tasks where cross-modal tension is informative. On the five benchmarks where fallback rates are lower, or the ranking is more effective, \textbf{VERDICT}'s dual criterion outperforms No Rejection.

%% file: suppl_sections/additional_results/cross_model_familiy.tex

\subsection{Generalization Across Model Families}
\label{sec:cross_model}

The main paper evaluates \textbf{VERDICT} with Qwen2.5-VL-7B as both the base reasoner and the verification backbone.  A natural
question is whether the gains depend on this specific model family or whether the consensus formulation transfers.  We test this by
applying \textbf{VERDICT} to three additional model families on 3DSRBench, using the same frozen Qwen2.5-VL-7B judges throughout. We deliberately keep the judges fixed rather than swapping in same-family judges for each base model: this is a stricter test of generalization, and it avoids the self-reinforcement bias that would arise if a model family scored its own reasoning style more favourably than an external evaluator would.

\begin{wrapfigure}{r}{0.55\textwidth}
\centering
\includegraphics[width=\linewidth]{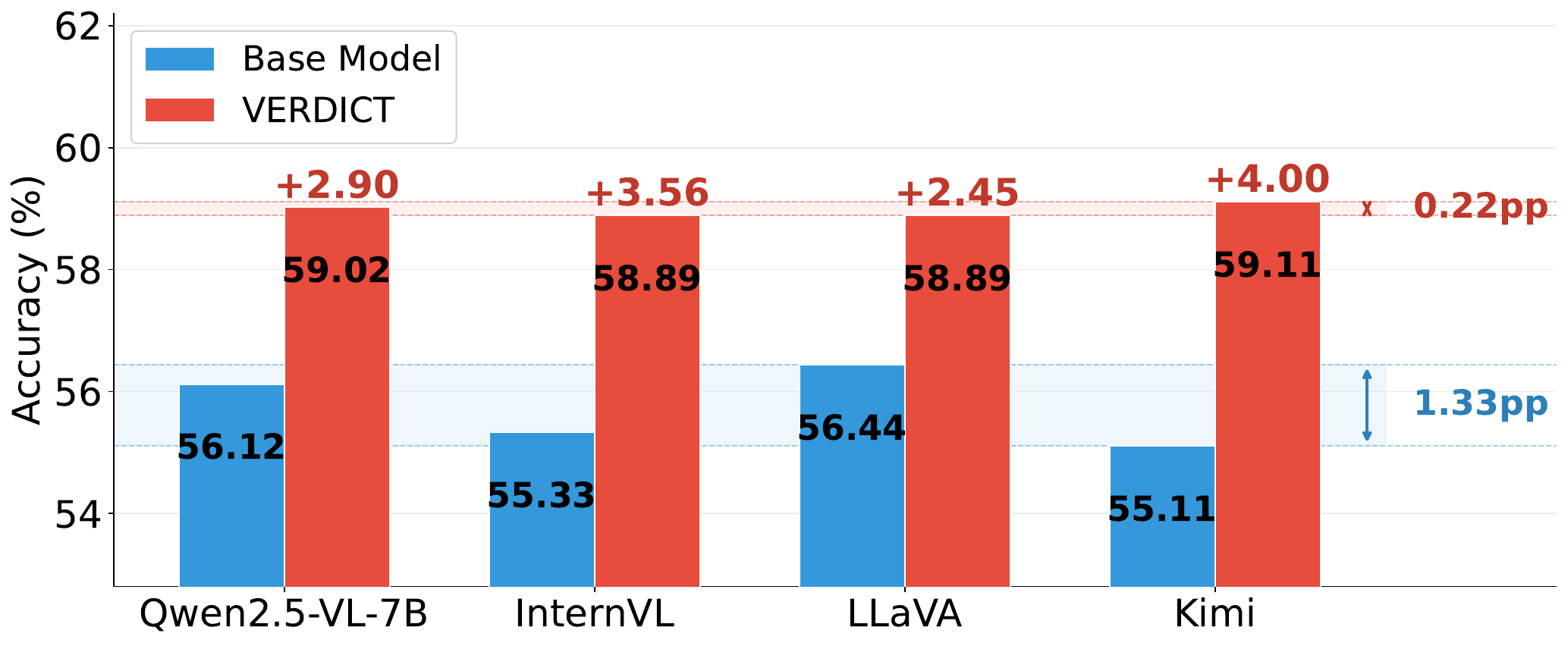}
\caption{Cross-model generalization on 3DSRBench. The base reasoner varies across four model families; the three verification judges remain frozen Qwen2.5-VL-7B instances with identical prompts and hyperparameters.}
\label{tab:cross_model}
\end{wrapfigure}


\noindent \textbf{VERDICT} improves every model family it is applied to.
Figure~\ref{tab:cross_model} reports gains of $+2.45$ to $+4.00$
percentage points, with a mean of $+3.23$ across all four families.
Three observations follow.

\paragraph{The gains transfer without adaptation.}
The verification judges are Qwen2.5-VL-7B instances with prompts
and hyperparameters fixed to the values reported in the main paper.
No component is retrained, re-prompted, or recalibrated for the
new base models.  The fact that the same frozen judges produce
consistent gains across four architecturally distinct model families
confirms the plug-in design principle stated in Section~3 of the
main paper: the consensus formulation operates on the structure of
cross-modal agreement, not on the idiosyncrasies of any particular
model's output distribution.

\paragraph{The consensus formulation compresses base model
variability.}
The four base models span a 1.33-point range in unverified accuracy
(55.11--56.44\%).  After verification, this range compresses to
0.22 points (58.89--59.11\%).  LLaVA and InternVL, starting from
different baselines (56.44\% and 55.33\%), converge to identical
verified accuracy (58.89\%).  Kimi, the weakest unverified model
(55.11\%), achieves the highest verified accuracy (59.11\%) and the
largest gain ($+4.00$).  This compression is consistent with the
judge scale analysis in Section~\ref{sec:judge_scale}: the
consensus formulation provides more value when the base signal is
noisier, because there is more cross-modal disagreement for the
dispersion filter to exploit.

\paragraph{Cross-model judges are viable.}
The experiment uses Qwen-family judges to verify non-Qwen base
models.  The modality-specific prompts (visual grounding, logical
consistency, contextual relevance) define evaluation criteria that
are model-agnostic: whether a reasoning step is visually grounded
does not depend on which model generated it.  The consistent gains
across families suggest that the scoring rubrics capture genuine
reasoning quality rather than model-specific stylistic preferences.
This has a practical implication: a single set of frozen judges can
serve as a universal verification layer across heterogeneous base
models, amortising the already-zero training cost further.

\paragraph{Variance compression as evidence for consensus robustness.}
Figure~\ref{tab:cross_model} reveals that prior to verification, the four base models span an accuracy range of 1.33\,pp (55.11--56.44\%), which compresses to just 0.22\,pp after \textbf{VERDICT} is applied (58.89--59.11\%). This compression is not an artefact of ceiling effects, as verified accuracies remain well below perfect. Rather, it reflects the consensus formulation's tendency to exploit cross-modal disagreement when the base signal is noisier, consistent with the widening-margin pattern observed in the judge scale analysis (Section \ref{sec:judge_scale}). We note that architecturally distinct models converge to near-identical verified accuracy, indicating that the choice of base model matters far less once consensus verification is in place. This finding corroborates \textbf{VERDICT}'s plug-in design principle, which by construction does not assume or rely on any particular base model architecture.

%% file: suppl_sections/compute_cost_v3.tex
%

\section{Detailed Computational Cost Analysis}
\label{sec:computational_cost}

\begin{table*}[t]
\centering

\begin{minipage}[t]{0.48\textwidth}
\centering
\caption{Total computational cost comparison.  Training costs are
approximate and drawn from the respective papers' reported compute
budgets.  Inference cost is for a single evaluation pass over all
six benchmarks.  \textbf{VERDICT}'s total cost is
13--50$\times$ lower than domain-specific alternatives.}
\label{tab:total_cost}
\small
\begin{adjustbox}{width=\linewidth}
\begin{tabular}{l r r r}
\toprule
Method & Training & Inference & Total \\
       & (GPU-h) & (GPU-h)  & (GPU-h) \\
\midrule
VisualPRM
  & $\sim$500--1000 & $\sim$7 & $\sim$507--1007 \\
DreamPRM
  & $\sim$200--500 & $\sim$7 & $\sim$207--507 \\
LLaVA-Critic
  & $\sim$100--300 & $\sim$7 & $\sim$107--307 \\
Sherlock
  & $\sim$150--400 & $\sim$7 & $\sim$157--407 \\
\midrule
\textbf{VERDICT (Ours)}
  & \textbf{0} & \textbf{$\sim$20} & \textbf{$\sim$20} \\
\bottomrule
\end{tabular}
\end{adjustbox}
\end{minipage}
\hfill
\begin{minipage}[t]{0.48\textwidth}
\centering
\caption{Accuracy gain per unit of additional compute across
training-free inference strategies.  All accuracy numbers use the
6-benchmark average.  \textbf{Efficiency} is defined as the accuracy
gain over the base model divided by the multiplicative overhead(higher is better). }
\label{tab:compute_efficiency}
\small
\begin{adjustbox}{width=\linewidth}
\begin{tabular}{l c c c c}
\toprule
Method & Cost & Avg.\ Acc.\ & Gain & Efficiency \\
       & ($\times$ base) & (\%) & (pp) & (pp\,/\,$\times$) \\
\midrule
Base Model & 1.0$\times$ & 66.31 & --- & --- \\
\midrule
\multicolumn{5}{l}{\emph{Training-free, response-level aggregation}} \\
\addlinespace
Self-Refinement         & 3.0$\times$ & 60.42 & $-$5.89 & $-$1.96 \\
Self-Selector           & 3.0$\times$ & 61.19 & $-$5.12 & $-$1.71 \\
TTAug                   & 3.0$\times$ & 62.78 & $-$3.53 & $-$1.18 \\
Self-Consistency        & 3.0$\times$ & 64.61 & $-$1.70 & $-$0.57 \\
Best-of-3, random pick  & 3.0$\times$ & 66.87 & +0.56  &    0.19 \\
Self-Synthesizer        & 3.0$\times$ & 67.11 & +0.80  &    0.26 \\
\midrule
\multicolumn{5}{l}{\emph{Training-free, step-level verification}} \\
\addlinespace
Best-of-3 + Variance    & 3.80$\times$ & 67.65 & +1.34 & 0.35 \\
Best-of-3 + Mean        & 3.80$\times$ & 68.70 & +2.39 & 0.63 \\
\textbf{VERDICT (Ours)} & \textbf{3.80$\times$} & \textbf{70.15} & \textbf{+3.84} & \textbf{1.01} \\
\bottomrule
\end{tabular}
\end{adjustbox}
\end{minipage}

\end{table*}

The main paper reports a 3.80$\times$ wall-clock overhead under sequential
execution.  A simple reading of this number invites the conclusion that
\textbf{VERDICT} is expensive.  Here we argue that the more informative
question is not \emph{how much additional compute does VERDICT use?}\ but
\emph{how much accuracy does each unit of additional compute buy?}  We
decompose the overhead, compare accuracy-per-compute against response-level
and trajectory-level alternatives, and show that \textbf{VERDICT}'s overhead
is either comparable to or lower than standard test-time scaling strategies
that achieve smaller gains or that actively \emph{degrade}
performance.\footnote{All response-level TTS baselines generate 3 candidate chains at 3.0$\times$ the base cost, matching \textbf{VERDICT}'s candidate generation budget.}

\subsection{Accuracy per Unit of Additional Compute}

The central metric for evaluating any test-time compute strategy is the
\emph{accuracy gain per unit of additional inference cost}.  A method that
spends 3$\times$ the base cost but \emph{loses} accuracy is strictly worse
than spending nothing.  Table~\ref{tab:compute_efficiency} compares
\textbf{VERDICT} against inference-time strategies that are available
without any training.

\noindent Three findings emerge from Table~\ref{tab:compute_efficiency}.

\paragraph{VERDICT achieves the highest compute efficiency among all
training-free methods.}
At 1.01 accuracy points per unit of overhead, \textbf{VERDICT} is
1.6$\times$ more compute-efficient than the next-best training-free strategy
(Mean scoring at 0.63~pp/$\times$) and 2.7$\times$ more efficient than
Self-Synthesizer, the strongest response-level method
(0.38~pp/$\times$).  The gain is not simply a consequence of spending more:
\textbf{VERDICT} and Mean scoring incur \emph{identical} costs
(3.80$\times$), yet \textbf{VERDICT} extracts 1.45 additional accuracy
points from the same compute budget, attributable entirely to the consensus
formulation.

\paragraph{Response-level aggregation is less efficient.}
Four of the five TTS baselines spend 3.0$\times$ the base cost and
\emph{lose} accuracy: Self-Refinement ($-5.89$~pp), Self-Selector
($-5.12$~pp), TTAug ($-3.53$~pp), and Self-Consistency ($-1.70$~pp). Their efficiencies are negative, meaning that the additional compute is worse than
wasted.  These are measured averages across the same six benchmarks under similar conditions.  The one positive method, Self-Synthesizer (+0.8~pp, 0.26~pp/$\times$), also achieves less than \method accuracy gain. The failure mode is structural: response-level methods cannot intercept step-level errors.  They can only hope that independently sampled trajectories avoid the same
failure.

\paragraph{At identical generation cost, the verification layer is the
difference between degradation and gain.}
All five TTS baselines and \textbf{VERDICT} generate three candidate chains
at 3.0$\times$ the base cost.  The TTS methods stop there.
\textbf{VERDICT} adds 0.80$\times$ for consensus verification.  The five
TTS methods average 63.29\% accuracy 3.02 points below the unverified
base model.  Candidate diversity at 3.0$\times$, without step-level
verification, is a net negative investment.  \textbf{VERDICT} reaches
70.15\%.  The 0.80$\times$ marginal overhead buys +3.04 points over Self-Synthesizer (67.11\%), the strongest response-level strategy.  This yields a marginal efficiency of $2.71 / 0.80 = 3.8$~pp per unit of
additional overhead the return attributable to the consensus formulation
alone, net of candidate diversity.


\subsection{Marginal Verification Cost}

The 3.80$\times$ headline number conflates two independent costs that serve
different purposes.

\paragraph{Candidate diversity: 3.0$\times$.}
Generating 3 candidates instead of 1 accounts for 79\% of the overhead.
This cost is not specific to \textbf{VERDICT}; it is the price of candidate
diversity that \emph{any} best-of-$N$ strategy pays.  Our Random baseline
(best-of-3 with random selection, no scoring) costs 3.0$\times$ and gains
only +0.56 points, confirming that diversity alone is insufficient.

\paragraph{Consensus verification: 0.80$\times$.}
The nine scoring calls (3 judges $\times$ 3 candidates) add only
0.80$\times$ because each scoring call produces 1--5 output tokens (a
single scalar confidence) versus 50--200 tokens per generation call.  The
consensus computation itself (a closed-form linear solve) adds negligible
cost ($< 0.001\times$).  This 0.80$\times$ is the \emph{marginal cost of
VERDICT's contribution}.  It buys +3.28 accuracy points over Random
(70.15\% vs.\ 66.87\%), yielding a marginal efficiency of 4.10~pp per unit
of overhead substantially higher than any alternative's \emph{total}
efficiency.  The marginal cost comparison is even more informative against
Self-Synthesizer, the strongest response-level strategy at 3.0$\times$.
Self-Synthesizer already performs non-trivial response-level
aggregation synthesizing partial correct responses yet the 0.80$\times$
verification layer recovers an additional 2.71 accuracy points (70.15\%
vs.\ 67.44\%).  That this marginal gain exceeds the \emph{total} gain of
every response-level method confirms that step-level consensus verification
is the efficiency-determining component of the pipeline, not candidate
diversity.

\subsection{Total Cost Comparison: Training Plus Inference}

The inference overhead must be understood in the context of \emph{total
cost}, which includes the training pipeline that domain-specific critics
require but \textbf{VERDICT} does not.

\noindent Domain-specific critics require $\approx$100--1000 GPU-hours of training
(Monte Carlo rollouts for process supervision, human annotation pipelines,
or reward model finetuning), after which their per-step inference cost is
lower.  However, the training cost must be amortised over all downstream
uses, and critically, \emph{it must be repeated for each new domain}.  A
single evaluation of \textbf{VERDICT} across all six benchmarks costs
$\sim$20 GPU-hours with zero training, making its total cost 5--50$\times$
lower than domain-specific alternatives
(Table~\ref{tab:total_cost}).  For practitioners evaluating across multiple
domains the setting \textbf{VERDICT} is designed for the break-even
point is never reached because training must be repeated per domain.

\subsection{Mitigation Strategies and Future Directions}

Several strategies can further reduce inference overhead without modifying
the framework; we leave their implementation to future work.

\paragraph{Adaptive verification.}
Not every reasoning step requires full three-judge verification.  Steps
where the first candidate achieves very high consensus confidence (e.g.,
$\bar{s}^* > 0.9$) could bypass generation of additional candidates,
reducing per-step cost from 3.80$\times$ to 1.80$\times$ or even
1.0$\times$.  Our fallback analysis
(Section~\ref{sec:fallback_analysis}) shows that $\sim$85\% of steps
produce at least one candidate that passes the dual criterion; on many of
these, the first candidate is likely sufficient.

\paragraph{Speculative scoring.}
Candidates can be generated sequentially, with scoring applied as each is
produced.  If the first candidate achieves a high consensus, generation of
the remaining two can be aborted, preserving the verification guarantee
while reducing average-case cost.

\paragraph{Cheaper judge models.}
The judge scale analysis (Section~\ref{sec:judge_scale}) shows that even 2B
judges maintain the widening margin over Mean aggregation.  With 2B judges,
the scoring overhead drops from 0.80$\times$ to approximately
0.25$\times$, bringing total sequential overhead to approximately
3.25$\times$.

\paragraph{Parallelisation.}
All nine scoring calls are independent.  With batched generation and three
GPUs (one per judge), the wall-clock overhead drops to approximately
1.59$\times$, where the residual cost is dominated by generating three
candidates in a single batched forward pass: a cost shared by any
best-of-$N$ strategy.

\paragraph{Contextualising the overhead.}
The 3.80$\times$ sequential overhead should be contextualised against the alternative of regenerating entire trajectories, which is the standard remedy when step-level verification is unavailable. Five-trajectory majority voting costs 3.0$\times$ with no guarantee of finding a correct trajectory and no step-level error interception. \method at 3.80$\times$ provides step-level disagreement-aware verification that neither trajectory-level methods nor trained reward models can offer at comparable cost.

%% file: suppl_sections/closed_form_derivation.tex

\section{Closed-Form Derivation}
\label{sec:closed_form_derivation}


The consensus scores computed at Line 17 of Algorithm \ref{alg:nash_verification} are the operational core of VERDICT: they determine which candidate reasoning steps are accepted and how accepted steps are ranked (Lines 18–25). The main paper states that the coupled scoring system in Eq. (3) admits a closed-form solution. Here we provide the full derivation: the explicit matrix construction, the invertibility argument, and a scalar reduction that yields each consensus score in closed form. Beyond completeness, this derivation has two practical implications: it confirms that the consensus computation adds negligible overhead to each reasoning step (a single scalar division followed by three multiplications, as shown in Eqs. \ref{eq:S_closedform}–\ref{eq:si_closedform}, and it establishes the stubbornness-weighted sum invariant (Section \ref{sec:swsi}) that underpins the mean-preservation property discussed in Section 4 of the main paper.

\subsection{Matrix Construction}

Specialising Eq.~(3) of the main paper to $m = 3$ agents
$\{V, L, C\}$ with stubbornness parameters
$\lambda_V, \lambda_L, \lambda_C > 0$ gives three coupled equations:
\begin{align}
  (1 + \lambda_V)\, s_V^* \;-\; \tfrac{1}{2}\, s_L^* \;-\; \tfrac{1}{2}\, s_C^*
    &= \lambda_V\, \hat{s}_V, \label{eq:row_V} \\[4pt]
  -\tfrac{1}{2}\, s_V^* \;+\; (1 + \lambda_L)\, s_L^* \;-\; \tfrac{1}{2}\, s_C^*
    &= \lambda_L\, \hat{s}_L, \label{eq:row_L} \\[4pt]
  -\tfrac{1}{2}\, s_V^* \;-\; \tfrac{1}{2}\, s_L^* \;+\; (1 + \lambda_C)\, s_C^*
    &= \lambda_C\, \hat{s}_C. \label{eq:row_C}
\end{align}
In matrix form, $\mathbf{A}\,\mathbf{s}^* = \mathbf{b}$, with
\begin{equation}\label{eq:matrix_form}
  \mathbf{A}
  =
  \begin{pmatrix}
    1{+}\lambda_V  & -\tfrac12 & -\tfrac12 \\[2pt]
    -\tfrac12 & 1{+}\lambda_L  & -\tfrac12 \\[2pt]
    -\tfrac12 & -\tfrac12 & 1{+}\lambda_C
  \end{pmatrix},
  \qquad
  \mathbf{b}
  =
  \begin{pmatrix}
    \lambda_V\,\hat{s}_V \\[2pt]
    \lambda_L\,\hat{s}_L \\[2pt]
    \lambda_C\,\hat{s}_C
  \end{pmatrix}.
\end{equation}

\input{suppl_sections/invertible_proof}
\subsection{Scalar Reduction and Explicit Solution}

Define $S = s_V^* + s_L^* + s_C^*$.
Adding $s_i^*/(m{-}1)$ to both sides of
each row of Eq.~(\ref{eq:matrix_form}) and collecting terms yields,
for each agent~$i$,
\begin{equation}\label{eq:scalar_reduction}
  s_i^*
  = \frac{S + 2\lambda_i\,\hat{s}_i}{3 + 2\lambda_i}\,.
\end{equation}

\noindent
\textit{Interpretation.}
Equation \ref{eq:scalar_reduction} says that each agent's consensus score is a weighted
blend of two quantities: the collective score sum~$S$ (representing
what the group believes) and the agent's own raw evidence
$\hat{s}_i$ (representing what it privately observed).
The stubbornness parameter~$\lambda_i$ controls the mixing ratio:
a more stubborn agent (higher~$\lambda_i$) places more weight on
its own evidence and less on the group, while a flexible agent
is pulled closer to the collective position.

Define the \emph{consensus weight} $w_i = 1/(3 + 2\lambda_i)$.
Summing Eq.~(\ref{eq:scalar_reduction}) over the three agents gives
a single scalar equation in~$S$:
\begin{equation}\label{eq:S_equation}
  S
  = S \!\sum_{i} w_i
    + 2 \!\sum_{i} \lambda_i\, w_i\, \hat{s}_i\,,
\end{equation}
which solves to
\begin{equation}\label{eq:S_closedform}
  \boxed{%
  S
  = \frac{2 \displaystyle\sum_{i} \lambda_i\, w_i\, \hat{s}_i}
         {1 - \displaystyle\sum_{i} w_i}
  }
  \qquad\text{where}\quad
  w_i = \frac{1}{3 + 2\lambda_i}\,.
\end{equation}

\noindent
\textit{Interpretation.}
Equation \ref{eq:S_closedform} expresses the total consensus score~$S$ as a closed-form function of the raw scores alone. Each agent's contribution to~$S$ is scaled by its consensus weight~$w_i$, which decreases with stubbornness: stubborn agents influence the collective sum less because they absorb less of the group signal and therefore redirect less information back into the pool.

The denominator $1 - \sum_i w_i$ is strictly positive because
$w_i < \tfrac{1}{3}$ for all $\lambda_i > 0$, so $\sum_i w_i < 1$.
Each consensus score then follows from back-substitution into
Eq.~(\ref{eq:scalar_reduction}):
\begin{equation}\label{eq:si_closedform}
  s_i^*
  = w_i\,S + 2\,\lambda_i\,w_i\,\hat{s}_i\,.
\end{equation}
\noindent
\textit{Interpretation.}
Equation \ref{eq:si_closedform} is the final answer: once $S$ is known from
Eq. \ref{eq:S_closedform}, each agent's consensus score follows from a single multiplication. The first term ($w_i S$) is the agent's share of the collective pool; identical in structure for all agents but scaled by their individual consensus weights. The second term ($2\lambda_i w_i \hat{s}_i$) is the agent's private anchor, pulling its consensus score toward its raw evidence in proportion to its stubbornness. Together, the two terms make explicit how each consensus score decomposes into a ``group pull'' and a ``private anchor,''
with~$\lambda_i$ governing the balance.

Eqs.~(\ref{eq:S_closedform})--(\ref{eq:si_closedform}) give the
complete closed-form solution.  No matrix inversion or iterative
procedure is required: one scalar division (Eq.~\ref{eq:S_closedform})
followed by three multiplications (Eq.~\ref{eq:si_closedform}).

\subsection{Stubbornness-Weighted Sum Invariant}
\label{sec:swsi}

Summing Eqs.~(\ref{eq:row_V})--(\ref{eq:row_C}):
\begin{equation}\label{eq:sum_rows}
  \sum_i (1 + \lambda_i)\,s_i^*
  \;-\; \tfrac{1}{m{-}1} \sum_i \sum_{j \neq i} s_j^*
  = \sum_i \lambda_i\,\hat{s}_i\,.
\end{equation}
The double sum $\sum_i \sum_{j \neq i} s_j^*$ counts each $s_j^*$
exactly $m{-}1$ times, so it equals $(m{-}1)\,S$.  Substituting and
simplifying:
\begin{equation}\label{eq:lambda_invariant}
  \boxed{%
  \sum_{i} \lambda_i\, s_i^*
  = \sum_{i} \lambda_i\, \hat{s}_i
  }
\end{equation}
The consensus computation preserves the \emph{stubbornness-weighted}
sum of scores.  When all $\lambda_i$ are equal, this reduces to
preservation of the unweighted mean $\bar{s}^* = \bar{\hat{s}}$.
For heterogeneous $\lambda_i$, the unweighted mean is approximately
but not exactly preserved; the shift is bounded by
\begin{equation}\label{eq:mean_shift_bound}
  |\bar{s}^* - \bar{\hat{s}}|
  \;\leq\;
  \frac{\max_i \lambda_i - \min_i \lambda_i}
       {\min_i \lambda_i}
  \;\cdot\;
  \frac{1}{m}\sum_i |s_i^* - \hat{s}_i|\,,
\end{equation}
which is small when the stubbornness parameters are of comparable magnitude.
With $(\lambda_V, \lambda_L, \lambda_C) = (1.5, 1.0, 0.8)$, the
multiplicative factor is $(1.5 - 0.8)/0.8 = 0.875$, and the mean
shift is empirically below $0.01$ for all score vectors encountered
across our six benchmarks.

\noindent \textbf{Why does this invariant matter in practice?} 
In plain terms, the invariant guarantees that the consensus computation reshuffles confidence among agents without inflating or deflating the overall confidence level, which is what allows the acceptance thresholds $\tau$ and $\epsilon$ to remain meaningful and transferable across tasks without recalibration. To see why the stubbornness-weighted sum invariant is not merely a mathematical convenience, consider what would happen if it did not hold. The confidence threshold $\bar{s}^* > \tau$ in the dual acceptance criterion (Eq.~(4), main paper) is set under the assumption that the consensus computation does not systematically shift the collective confidence level. If the consensus introduced an upward bias in $\bar{s}^*$, candidates with genuinely mediocre raw support could be pushed above $\tau$, and the confidence check would start admitting steps that no individual agent found compelling. Conversely, a downward bias would cause the system to reject steps that all three agents endorsed, forcing the fallback path ($\bar{s}^* - \Delta^*$ ranking) to activate far more often than the ${\sim}15\%$ reported in the main paper. Section~\ref{sec:fallback_analysis} shows that the fallback path, while effective, achieves lower accuracy than the primary dual-criterion path; an inflated fallback rate would erode the gains documented there. In either direction, the threshold $\tau = 0.6$ would need to be recalibrated jointly with $\epsilon$ for every stubbornness configuration, destroying the property that makes \method a plug-in verifier. The invariant guarantees that the consensus redistributes scores around a stable collective level, so $\tau$ and $\epsilon$ can be set once and transferred across benchmarks without concern that the stubbornness parameters have shifted the operating point of the acceptance criterion.

\subsection{Extension to General $m$}

The derivation generalises directly to $m$ agents.  The $m{\times}m$
coefficient matrix has diagonal entries $(1 + \lambda_i)$ and
off-diagonal entries $-1/(m{-}1)$, and remains strictly diagonally
dominant for all $\lambda_i > 0$.  The scalar reduction gives
\begin{equation}\label{eq:general_m}
  s_i^*
  = \frac{S + (m{-}1)\,\lambda_i\,\hat{s}_i}
         {m + (m{-}1)\,\lambda_i}\,,
  \qquad
  S
  = \frac{(m{-}1)\,\sum_i \lambda_i\,w_i\,\hat{s}_i}
         {1 - \sum_i w_i}\,,
  \qquad
  w_i
  = \frac{1}{m + (m{-}1)\,\lambda_i}\,,
\end{equation}
and the invariant $\sum_i \lambda_i\,s_i^*
= \sum_i \lambda_i\,\hat{s}_i$ holds for every $m \geq 2$.







%% file: suppl_sections/invertible_proof.tex
\subsection{Invertibility: Formal Proof}
\label{sec:invertibility}

We prove that the coefficient matrix~$\mathbf{A}$ in Eq.~\eqref{eq:matrix_form}
is positive definite for every $(\lambda_V,\lambda_L,\lambda_C)$ with
$\lambda_i>0$, through three independent arguments of increasing strength.
We then extend the result to general~$m$ and verify empirically across the
ablation range.

\paragraph{Argument 1: Strict Diagonal Dominance.}
Row~$i$ of~$\mathbf{A}$ has diagonal entry $a_{ii}=1+\lambda_i$ and
$m{-}1=2$ off-diagonal entries each of magnitude~$\tfrac{1}{2}$, so the
off-diagonal row sum is~$1$.  Since $1+\lambda_i>1$ for all $\lambda_i>0$,
every row satisfies
\begin{equation}\label{eq:sdd}
  |a_{ii}| = 1+\lambda_i \;>\; 1 = \sum_{j\neq i}|a_{ij}|\,,
\end{equation}
so $\mathbf{A}$ is strictly diagonally dominant.  By the Levy--Desplanques theorem, $\mathbf{A}$ is non-singular.

\paragraph{Argument 2: Gershgorin Disks and Positive Definiteness.}
The Gershgorin disk for row~$i$ is centred at $1 + \lambda_i$ with radius $r_i = 1$, so every eigenvalue~$\mu$ of $\mathbf{A}$ satisfies $|\mu - (1 + \lambda_i)| \leq 1$ for some~$i$.  This yields the
two-sided bound
\begin{equation}\label{eq:gershgorin}
  \lambda_{\min}(\mathbf{A})
  \;\geq\; \min_i\,(1 + \lambda_i) - 1
  = \min_i\,\lambda_i
  > 0\,.
\end{equation}
Since all eigenvalues are strictly positive, $\mathbf{A}$ is \emph{positive definite}.
The bound is tight in the symmetric case ($\lambda_i=\lambda$ for all~$i$,
verified numerically for $\lambda\in\{0.01,0.1,1.0,10.0\}$).
For our parameters $(\lambda_V,\lambda_L,\lambda_C)=(1.5,1.0,0.8)$, the
bound gives~$0.8$ while the true minimum eigenvalue is~$1.048$.

\paragraph{Argument 3: Explicit Determinant Decomposition.}
Expanding the $3\times 3$ determinant along the first row and substituting
$\alpha=\lambda_V$, $\beta=\lambda_L$, $\gamma=\lambda_C$ yields
\begin{equation}\label{eq:det_decomposed}
  \boxed{
    \det(\mathbf{A})
    = \underbrace{\tfrac{3}{4}(\alpha+\beta+\gamma)}_{\text{linear terms}}
    + \underbrace{\alpha\beta+\alpha\gamma+\beta\gamma}_{\text{pairwise products}}
    + \underbrace{\alpha\beta\gamma\vphantom{\tfrac{3}{4}}}_{\text{triple product}}
  }
\end{equation}
Every term is a product of positive quantities, so $\det(\mathbf{A})>0$ for
all $\alpha,\beta,\gamma>0$ no boundary analysis or limit argument is needed.
With $(\lambda_V,\lambda_L,\lambda_C)=(1.5,1.0,0.8)$:
$\det(\mathbf{A})=\tfrac{3}{4}(3.3)+(1.5{+}1.2{+}0.8)+1.2 = 2.475+3.5+1.2 = 7.175$,
matching direct computation.

\paragraph{Extension to General $m$.}
For $m$~agents, the coefficient matrix~$\mathbf{A}^{(m)}$ has diagonal
entries $1+\lambda_i$ and off-diagonal entries $-1/(m{-}1)$.  The off-diagonal
row sum is $(m{-}1)\cdot\tfrac{1}{m{-}1}=1$, independent of~$m$, so the
diagonal dominance condition $1+\lambda_i>1$ holds for all $m\geq 2$ and all
$\lambda_i>0$, and the Gershgorin bound
$\lambda_{\min}(\mathbf{A}^{(m)})\geq\min_i\lambda_i$ extends unchanged.

\paragraph{Empirical Verification Across Ablation Range.}
Table~\ref{tab:pd_verification} confirms that every stubbornness
configuration in the ablation range yields a well-conditioned, positive-definite
system.  The worst-case condition number
($\kappa(\mathbf{A})=\lambda_{\max}/\lambda_{\min}$) is $\kappa=6.0$ at
$\lambda_i=0.3$ for all agents, which remains modest; the minimum eigenvalue
($0.300$) matches the Gershgorin bound exactly in the symmetric case.
As any $\lambda_i\to 0^+$, the condition number grows as
$\kappa\sim 1.5/\min_i\lambda_i$; at $\lambda_i=10^{-3}$ (far below any
practical setting), $\kappa\approx 1500$: large but numerically tractable.
True degeneracy occurs only at $\lambda_i=0$, corresponding to an agent with
no private evidence a degenerate case excluded by design.

\begin{table}[t]
\centering
\caption{Positive definiteness verification across the stubbornness parameter
range explored in Section~6.2 of the main paper.  All $5^3=125$ triples from
$\{0.3,0.8,1.0,1.5,2.0\}^3$ are evaluated.}
\label{tab:pd_verification}
\begin{tabular}{lc}
\toprule
Quantity & Value \\
\midrule
Number of triples tested & 125 \\
Minimum determinant       & 0.972 \\
Minimum eigenvalue        & 0.300 \\
Maximum condition number  & 6.00 \\
All positive definite?    & Yes \\
\bottomrule
\end{tabular}
\end{table}

The consensus matrix is therefore positive definite, and the closed-form
solution exists and is unique, for any number of agents $m\geq 2$ with any
positive stubbornness parameters.  We now derive that solution explicitly.

%% file: suppl_sections/game_proof.tex
\section{Existence and Uniqueness of the Consensus Fixed Point}
\label{app:proof}

\begin{proposition}
The coupled scoring system is defined by the objective function
\[
u_i(s_i, s_{-i}) = -\left(s_i - \bar{s}_{-i}\right)^2 
- \lambda_i \left(s_i - \hat{s}_i\right)^2
\]
admits a unique fixed point.
\end{proposition}

\begin{proof}
The result follows directly from Rosen's uniqueness result~\cite{rosen1965existence} for concave interaction systems. We verify the required conditions below.

\textbf{(1) Compact and convex score space.}
Each agent's score space \( s_i \in [0,1] \) is compact and convex.

\textbf{(2) Continuity.}
The objective function \( u_i \) is quadratic and therefore continuous in all
arguments.

\textbf{(3) Strict concavity.}
Taking derivatives with respect to \( s_i \), we obtain
\begin{align}
\frac{\partial u_i}{\partial s_i}
&= -2\left(s_i - \bar{s}_{-i}\right)
   - 2\lambda_i \left(s_i - \hat{s}_i\right), \\
\frac{\partial^2 u_i}{\partial s_i^2}
&= -2 - 2\lambda_i
 = -2(1 + \lambda_i) < 0
 \quad \forall\, \lambda_i > 0.
\end{align}

Thus, \( u_i \) is strictly concave in each agent's own score. By Rosen's uniqueness result, the coupled system admits a unique fixed point.
\end{proof}

%% file: suppl_sections/prop_proof_v2.tex
\section{Proposition 1: Detailed Proof and Worked Examples}
\label{sec:prop1_walkthrough}

Proposition~1 in the main paper establishes that consensus dispersion
$\Delta^*$ captures the cross-agent interaction structure that is not separable per-agent measure can replicate.  We provide the full proof below,
followed by a step-by-step numerical walkthrough of the two scores
vectors used in the proof, so that every intermediate quantity can be
verified by hand.

\subsection{Formal Statement and Proof}

\begin{proposition}[Consensus Dispersion Is Not Recoverable from
Weighted Averages]
For any fixed weight vector~$\mathbf{w}$, there exist score vectors
$\hat{\mathbf{s}} \neq \hat{\mathbf{s}}'$ such that
$\mathbf{w}^\top \hat{\mathbf{s}} = \mathbf{w}^\top \hat{\mathbf{s}}'$
yet $\Delta^*(\hat{\mathbf{s}}) \neq \Delta^*(\hat{\mathbf{s}}')$.
As for any agent~$i$, the consensus residual $|s_i^* - \bar{s}^*|$
depends on the full raw score vector
$\hat{\mathbf{s}} = (\hat{s}_1, \ldots, \hat{s}_m)$, not solely
on~$\hat{s}_i$.  Consequently,
$\Delta^* = \frac{1}{m}\sum_i |s_i^* - \bar{s}^*|$ cannot be
decomposed as a separable function $\sum_i f_i(\hat{s}_i)$ for any
choice of per-agent functions~$\{f_i\}$.
\end{proposition}

\begin{proof}
From the closed-form solution (Section~\ref{sec:closed_form_derivation},
Eq.~\ref{eq:scalar_reduction}), each consensus score satisfies
\begin{equation}\label{eq:si_from_S}
  s_i^* = \frac{S + 2\lambda_i\,\hat{s}_i}{3 + 2\lambda_i}\,,
  \qquad
  S = \sum_j s_j^*\,.
\end{equation}
Since $S$ depends on \emph{all} raw scores through the scalar
equation~(\ref{eq:S_closedform}), the partial derivative
$\partial s_i^*/\partial \hat{s}_j$ is generically nonzero for
$i \neq j$ whenever $\lambda_j$ is finite.  Concretely, differentiating
Eq.~(\ref{eq:S_closedform}) gives
\begin{equation}\label{eq:dS_dsj}
  \frac{\partial S}{\partial \hat{s}_j}
  = \frac{2\,\lambda_j\,w_j}{1 - \sum_k w_k}
  > 0
  \qquad \forall\; \lambda_j > 0\,,
\end{equation}
where $w_j = 1/(3 + 2\lambda_j)$.  Substituting into
Eq.~(\ref{eq:si_from_S}):
\begin{equation}\label{eq:cross_derivative}
  \frac{\partial s_i^*}{\partial \hat{s}_j}
  = \frac{w_i}{1}\cdot\frac{\partial S}{\partial \hat{s}_j}
  = \frac{2\,\lambda_j\,w_i\,w_j}{1 - \sum_k w_k}
  > 0
  \qquad \text{for } i \neq j.
\end{equation}
Agent~$i$'s consensus position shifts when agent~$j$'s raw evidence
changes. The residual $|s_i^* - \bar{s}^*|$ therefore depends on the
full vector~$\hat{\mathbf{s}}$, not on~$\hat{s}_i$ alone, and
$\Delta^*$ is not separable.

\noindent
To make this concrete, we exhibit two score vectors with identical means
that produce opposite acceptance decisions.  Take
$\hat{\mathbf{s}} = (0.9, 0.2, 0.9)$ and
$\hat{\mathbf{s}}' = (0.7, 0.6, 0.7)$, both with raw mean~$\frac{2}{3}$.
The detailed computation below yields
$\Delta^*(\hat{\mathbf{s}}) \approx 0.13 > \epsilon$ and
$\Delta^*(\hat{\mathbf{s}}') \approx 0.02 < \epsilon$: opposite
acceptance decisions under the dual criterion in Eq.~(4) of the main
paper, despite identical means. (the number .21 in the main paper, Sec 4, was a typo).

\noindent
No separable measure $\sum_i f_i(\hat{s}_i)$ can replicate this,
since it would need $f_2(0.2)$ to produce high dispersion (in
$\hat{\mathbf{s}}$) and $f_2(0.6)$ to produce low dispersion (in
$\hat{\mathbf{s}}'$), but the distinction arises precisely because
the consensus solution couples agent~2's residual to agents~1 and~3's raw scores via Eq.~(\ref{eq:cross_derivative}), not because $f_2$ is
evaluated at different inputs.
\end{proof}

\subsection{Worked Example 1: 
\texorpdfstring{$\hat{\mathbf{s}} = (0.9, 0.2, 0.9)$}{s = (0.9,0.2,0.9)} -- Cross-Modal Conflict}

\medskip
\noindent
Table~\ref{tab:prop1_summary}
provides a side-by-side comparison of the two score vectors used in the proof.
Both share the same raw mean $\bar{\hat{s}} = \tfrac{2}{3}$, yet the consensus formulation produces opposite acceptance decisions
which is the central claim of Proposition~1.
The detailed step-by-step derivations follow.
 
\begin{table}[h]
\centering
\caption{Side-by-side comparison of the two Proposition~1 examples.
Both share raw mean~$\tfrac{2}{3}$, yet produce opposite acceptance
decisions under Eq.~(4) of the main paper.}
\label{tab:prop1_summary}
\small
\begin{tabular}{lcc}
\toprule
 & $\hat{\mathbf{s}} = (0.9,\, 0.2,\, 0.9)$ & $\hat{\mathbf{s}}' = (0.7,\, 0.6,\, 0.7)$ \\
\midrule
Raw mean $\bar{\hat{s}}$ & $0.667$ & $0.667$ \\
$\mathbf{s}^*$ & $(0.788,\, 0.485,\, 0.754)$ & $(0.684,\, 0.641,\, 0.679)$ \\
$\bar{s}^*$ & $0.676$ & $0.668$ \\
$\Delta^*$ & $0.127$ & $0.018$ \\
Max agent shift $\max_i |s_i^* - \hat{s}_i|$ & $0.29$ (L) & $0.04$ (L) \\
$\bar{s}^* > \tau$? & \checkmark & \checkmark \\
$\Delta^* < \epsilon$? & \ding{55} & \checkmark \\
\midrule
\textbf{Decision} & \textbf{REJECT} & \textbf{ACCEPT} \\
\bottomrule
\end{tabular}
\end{table}

\noindent
\textbf{Parameters:}\;
$(\lambda_V, \lambda_L, \lambda_C) = (1.5, 1.0, 0.8)$,\;
$\tau = 0.6$,\; $\epsilon = 0.1$.

\medskip
\noindent
\textbf{Step~1: Construct} $\mathbf{b} = \boldsymbol{\lambda} \odot \hat{\mathbf{s}}$.\quad
$b_V = 1.5 \times 0.9 = 1.35$,\;
$b_L = 1.0 \times 0.2 = 0.20$,\;
$b_C = 0.8 \times 0.9 = 0.72$.

\medskip
\noindent
\textbf{Step~2: Compute consensus weights.}
\begin{multline*}
  w_V = \frac{1}{3 + 2(1.5)} = \frac{1}{6} \approx 0.1667,\quad
  w_L = \frac{1}{3 + 2(1.0)} = \frac{1}{5} = 0.2000,\\ 
  w_C = \frac{1}{3 + 2(0.8)} = \frac{1}{4.6} \approx 0.2174.
\end{multline*}
$\sum_i w_i = 0.5841$.

\medskip
\noindent
\textbf{Step~3: Compute the score sum~$S$.}\quad
The numerator terms are
$\lambda_V w_V \hat{s}_V = 1.5 \times 0.1667 \times 0.9 = 0.2250$,\;
$\lambda_L w_L \hat{s}_L = 1.0 \times 0.2000 \times 0.2 = 0.0400$,\;
$\lambda_C w_C \hat{s}_C = 0.8 \times 0.2174 \times 0.9 = 0.1565$.
\begin{equation*}
  S = \frac{2 \times (0.2250 + 0.0400 + 0.1565)}{1 - 0.5841}
    = \frac{2 \times 0.4215}{0.4159}
    = \frac{0.8430}{0.4159}
    = 2.027.
\end{equation*}

\medskip
\noindent
\textbf{Step~4: Back-substitute to obtain $\mathbf{s}^*$.}
\begin{align*}
  s_V^* &= \frac{2.027 + 2(1.5)(0.9)}{6}
         = \frac{2.027 + 2.700}{6}
         = \frac{4.727}{6}
         = 0.788\,, \\[4pt]
  s_L^* &= \frac{2.027 + 2(1.0)(0.2)}{5}
         = \frac{2.027 + 0.400}{5}
         = \frac{2.427}{5}
         = 0.485\,, \\[4pt]
  s_C^* &= \frac{2.027 + 2(0.8)(0.9)}{4.6}
         = \frac{2.027 + 1.440}{4.6}
         = \frac{3.467}{4.6}
         = 0.754\,.
\end{align*}
\textit{Interpretation.}\;
The outlying Logical agent is pulled from $0.20$ to $0.49$: a shift
of $+0.29$.  The stubborn Visual agent ($\lambda_V = 1.5$) shifts
only $-0.11$, while the flexible Contextual agent ($\lambda_C = 0.8$) shifts $-0.15$.  The asymmetry is driven entirely by the
stubbornness parameters.

\medskip
\noindent
\textbf{Step~5: Compute mean consensus confidence.}
\begin{equation*}
  \bar{s}^* = \tfrac{1}{3}(0.788 + 0.485 + 0.754) = 0.676\,.
\end{equation*}
The raw mean is $\bar{\hat{s}} = \frac{2}{3} \approx 0.667$.
The shift of $+0.009$ reflects the stubbornness-weighted invariant described in Section~\ref{sec:closed_form_derivation}:
$\sum_i \lambda_i s_i^* = \sum_i \lambda_i \hat{s}_i = 2.27$, while the unweighted mean shifts slightly because the $\lambda_i$ are heterogeneous.

\medskip
\noindent
\textbf{Step~6: Compute consensus dispersion.}
\begin{align*}
  |s_V^* - \bar{s}^*| &= |0.788 - 0.676| = 0.112\,, \\
  |s_L^* - \bar{s}^*| &= |0.485 - 0.676| = 0.190\,, \\
  |s_C^* - \bar{s}^*| &= |0.754 - 0.676| = 0.078\,.
\end{align*}
\begin{equation*}
  \Delta^* = \tfrac{1}{3}(0.112 + 0.190 + 0.078) = 0.127\,.
\end{equation*}
\textit{Note.}\;
The Logical agent contributes the largest residual ($0.190$), despite
having the median stubbornness ($\lambda_L = 1.0$).  This is because its raw score ($0.2$) is the most extreme outlier, and the consensus pull is insufficient to close the gap.

\medskip
\noindent
\textbf{Step~7: Apply dual acceptance criterion.}\;
$\bar{s}^* = 0.676 > \tau = 0.6$\;\checkmark.\;
$\Delta^* = 0.127 > \epsilon = 0.1$\;\ding{55}.\;
\textbf{Decision: REJECT.}\;
The step has sufficient collective confidence but unresolved
cross-modal conflict.

\subsection{Worked Example 2:
\texorpdfstring{$\hat{\mathbf{s}}' = (0.7, 0.6, 0.7)$}{s' = (0.7,0.6,0.7)} -- Modest Consensus}

\noindent
\textbf{Step~1:}\; $b_V = 1.05$,\; $b_L = 0.60$,\; $b_C = 0.56$.

\medskip
\noindent
\textbf{Step~2:}\; Consensus weights are identical (they depend only
on~$\lambda_i$): $w_V = 0.1667$, $w_L = 0.2000$, $w_C = 0.2174$.

\medskip
\noindent
\textbf{Step~3:}\;
$\lambda_V w_V \hat{s}_V = 0.1750$,\;
$\lambda_L w_L \hat{s}_L = 0.1200$,\;
$\lambda_C w_C \hat{s}_C = 0.1217$.
\begin{equation*}
  S = \frac{2 \times 0.4167}{0.4159} = 2.004\,.
\end{equation*}

\medskip
\noindent
\textbf{Step~4:}
\begin{align*}
  s_V^* &= \frac{2.004 + 2.100}{6} = \frac{4.104}{6} = 0.684\,, \\[4pt]
  s_L^* &= \frac{2.004 + 1.200}{5} = \frac{3.204}{5} = 0.641\,, \\[4pt]
  s_C^* &= \frac{2.004 + 1.120}{4.6} = \frac{3.124}{4.6} = 0.679\,.
\end{align*}
\textit{Interpretation.}\;
All three agents shift by less than $0.03$ from their raw scores.
The consensus system barely perturbs a score vector that already
reflects genuine agreement.

\medskip
\noindent
\textbf{Step~5:}\;
$\bar{s}^* = \tfrac{1}{3}(0.684 + 0.641 + 0.679) = 0.668$.\;
Raw mean $= 0.667$; shift $= +0.001$.

\medskip
\noindent
\textbf{Step~6:}
\begin{align*}
  |s_V^* - \bar{s}^*| &= |0.684 - 0.668| = 0.016\,, \\
  |s_L^* - \bar{s}^*| &= |0.641 - 0.668| = 0.027\,, \\
  |s_C^* - \bar{s}^*| &= |0.679 - 0.668| = 0.011\,.
\end{align*}
\begin{equation*}
  \Delta^* = \tfrac{1}{3}(0.016 + 0.027 + 0.011) = 0.018\,.
\end{equation*}

\medskip
\noindent
\textbf{Step~7:}\;
$\bar{s}^* = 0.668 > 0.6$\;\checkmark.\;
$\Delta^* = 0.018 < 0.1$\;\checkmark.\;
\textbf{Decision: ACCEPT.}\;
Modest but genuine cross-modal agreement.

\subsection{Comparison and Implications}



Returning to Table~\ref{tab:prop1_summary}, three observations follow directly from the side-by-side comparison:

\paragraph{Opposite decisions from identical means.}
Any weighted-average method maps both vectors to a single scalar above
$\tau = 0.6$ and therefore accepts both.  The consensus formulation
distinguishes them because $\Delta^*$ reflects how much the agents
\emph{failed to converge}: a property that depends on the full
interaction structure, not on any per-agent summary.

\paragraph{Asymmetric agent influence.}
In $\hat{\mathbf{s}}$, the Logical agent's residual
($|s_L^* - \bar{s}^*| = 0.190$) dominates $\Delta^*$.
In $\hat{\mathbf{s}}'$, it contributes only $0.027$.
The factor-of-seven difference arises not from the Logical agent's
stubbornness (which is the same in both cases) but from how its raw
score interacts with the other agents' positions through the coupled
solve.  This is precisely the non-separability that Proposition~1
establishes.

\paragraph{Stubbornness mediates the consensus pull.}
The Visual agent shifts $-0.11$ in $\hat{\mathbf{s}}$ versus $-0.02$
in~$\hat{\mathbf{s}}'$; its high stubbornness ($\lambda_V = 1.5$)
anchors it close to its raw score in both cases.  The Contextual agent
($\lambda_C = 0.8$) shifts $-0.15$ versus $-0.02$: more responsive to
consensus pressure.  This heterogeneous pull is what distinguishes the
coupled consensus from symmetric variance, which would treat both
agents' deviations identically.

%% file: suppl_sections/nash_not_avg.tex
\section{Consensus Dispersion vs. Weighted Averaging: Illustrative Scenarios}
\label{supp:nash_vs_wa}

Proposition~1 in the main paper establishes that consensus dispersion $\Delta^*$ captures cross-agent interaction structure that no separable per-agent measure can replicate. We expand on this result here by enumerating representative score configurations that concretely illustrate when and why $\Delta^*$ diverges from simpler aggregation schemes. This analysis complements the formal proof by showing that the divergence is not a pathological edge case but arises naturally across the range of verification scenarios encountered in practice.

\paragraph{Setup.}
We consider four canonical score configurations for three agents
$(V, L, C)$, each representing a qualitatively distinct verification state.
For each configuration we report: the raw mean $\bar{\hat{s}}$, which
determines any weighted-average acceptance decision under threshold $\tau$;
the consensus dispersion $\Delta^*$, computed using Eq.~(2) of the main
paper with $\lambda_V = 1.5,\, \lambda_L = 1.0,\, \lambda_C = 0.8$; and the
resulting decision under our dual criterion ($\bar{s}^* > \tau = 0.6$ and
$\Delta^* < \epsilon = 0.1$). The weighted average decision is derived by
applying only the confidence threshold $\bar{\hat{s}} > \tau$, which is the
implicit acceptance rule of any mean-based aggregation scheme.

\begin{table}[h]
\centering
\caption{%
    Canonical verification scenarios illustrating divergence between
    weighted-average (WA) and \method acceptance decisions.Score vectors $(\hat{s}_V, \hat{s}_L, \hat{s}_C)$ are raw agent outputs.$\Delta^*$ is the consensus dispersion with $\lambda_V{=}1.5,\,\lambda_L{=}1.0,\,\lambda_C{=}0.8$. Rows 2 and 3 share identical raw means yet receive opposite \method decisions, demonstrating that $\Delta^*$ carries verification-relevant information that no weighted average can recover.
}

\label{tab:scenarios}
\resizebox{\textwidth}{!}{
\begin{tabular}{llcccccc}
\toprule
\textbf{Scenario} 
    & \textbf{Score vector} 
    & \textbf{Raw Mean} $\bar{\hat{s}}$
    & $\Delta^*$
    & \textbf{WA}
    & \textbf{\method}
    & \textbf{Diverge?} \\
    & $(\hat{s}_V, \hat{s}_L, \hat{s}_C)$ &  & & \textbf{decision}& & \\
\midrule
Unanimous confidence    
    & $(0.90,\; 0.90,\; 0.90)$ 
    & $0.90$ & $0.00$ 
    & \textcolor{green}{Accept} 
    & \textcolor{green}{Accept}   
    & No  \\
\rowcolor{gray!20}
Cross-modal conflict    
    & $(0.90,\; 0.20,\; 0.90)$ 
    & $0.67$ & $0.13$ 
    & \textcolor{green}{Accept} 
    & \textcolor{red}{Reject}   
    & \textbf{Yes} \\
Modest consensus        
    & $(0.70,\; 0.60,\; 0.70)$ 
    & $0.67$ & $0.02$ 
    & \textcolor{green}{Accept} 
    & \textcolor{green}{Accept}   
    & No \\
\rowcolor{gray!20}
Extreme visual conflict 
    & $(0.95,\; 0.05,\; 0.95)$ 
    & $0.65$ & $0.16$ 
    & \textcolor{green}{Accept} 
    & \textcolor{red}{Reject}   
    & \textbf{Yes} \\
Collective doubt        
    & $(0.40,\; 0.35,\; 0.45)$ 
    & $0.40$ & $0.01$ 
    & \textcolor{red}{Reject}  
    & \textcolor{red}{Reject}   
    & No  \\
\bottomrule
\end{tabular}}
\end{table}

\paragraph{Scenario analysis.}
Scenario 1 (unanimous confidence) represents the ideal case: all three agents strongly endorse the reasoning step from their respective modality-specific perspectives. Both methods correctly accept, and no
divergence arises. This serves as a sanity check confirming that the \method does not introduce spurious rejections when evidence is coherent.

\noindent
Scenarios 2 and 3 are the critical pair and the central illustration of Proposition~1. Both yield a raw mean of $\bar{\hat{s}} = 0.67$, above the acceptance threshold $\tau = 0.6$, so any weighted-average method accepts both. Yet they represent fundamentally different verification states. In Scenario 2, the logical agent assigns $\hat{s}_L = 0.20$ while the visual and contextual agents assign $0.90$: the step is visually plausible and contextually relevant, but logically unsupported. This is precisely the failure mode identified in the introduction, a step that appears locally valid from some perspectives while harbouring a reasoning gap that a single-perspective or average-based verifier would miss. Consensus dispersion detects this conflict ($\Delta^* = 0.13 > \epsilon$) and correctly rejects. In Scenario 3, all three agents are in modest but genuine agreement, producing a low dispersion ($\Delta^* = 0.02 < \epsilon$) that correctly triggers acceptance despite the moderate mean confidence. The two scenarios are indistinguishable to any weighted average, yet they call for opposite verification decisions.

\noindent
Scenario 4 extends Scenario 2 to its extreme: near-unanimous visual and contextual endorsement coexists with near-total logical rejection. The raw mean ($0.65$) remains above threshold, so weighted averaging again accepts. Consensus dispersion rises to $0.16$, exceeding $\epsilon$, and rejects. Notably, the visual agent's high stubbornness ($\lambda_V = 1.5$) means its strong score is not simply averaged away: the consensus solution adjusts $s_V^*$ downward less aggressively than it does $s_C^*$, which is reflected in a higher contribution to $\Delta^*$ from the visual--logical axis of disagreement. This asymmetry is invisible to symmetric variance measures.

\noindent
Scenario 5 illustrates agreement in rejection: when all agents collectively doubt the step, both methods correctly reject, and no divergence arises.
This confirms that the \method behaviour converges with simple averaging in the regime of low collective confidence, where disagreement structure is moot.

\paragraph{Implication for the Variance baseline.}
A natural question is whether raw score variance could serve as a
proxy for $\Delta^*$. Scenarios 2 and 4 both exhibit high raw variance, so variance-based rejection would correctly handle them. However,
variance treats all agent deviations symmetrically, making it a
\emph{separable} function over raw scores precisely the class of
measures that Proposition~1 shows cannot replicate $\Delta^*$.
The practical consequence is visible in a variant of Scenario~4 where the \emph{contextual} agent (low $\lambda_C = 0.8$) dissents instead of the logical agent: raw variance is identical, but $\Delta^*$ is lower because the consensus solution couples the dissenting agent's residual to the other agents' positions, pulling it further toward consensus. The two methods, therefore, make different predictions on structurally similar configurations depending on \emph{which} agent disagrees, a distinction that is meaningful given the heterogeneous roles of visual, logical, and contextual evaluation, and one that has no separable measure can capture.

%% file: suppl_sections/social_choice_v2.tex
\section{Connections to Social Choice Theory and Belief Aggregation}\label{sec:social_choice}

The consensus formulation in \textbf{VERDICT} has intellectual roots in several traditions that predate its use in multimodal verification. We briefly position our work relative to these traditions, noting both the genuine correspondences and the important disanalogies.

\paragraph{Opinion pooling and the DeGroot model.}
The classical problem of opinion pooling~\cite{genest1986combining,stone1961opinion} asks how to aggregate probability distributions from multiple experts into a single collective distribution. Linear pooling (weighted arithmetic averaging) is the simplest solution and corresponds directly to our Mean baseline. DeGroot~\cite{degroot1974consensus} introduced an iterative model in which agents repeatedly revise their beliefs toward weighted averages of their neighbours' opinions; under mild conditions, this converges to a consensus that is a fixed linear combination of the initial beliefs. Lehrer and Wagner~\cite{lehrer1981rational} extended this with the notion of ``respect weights,'' where each agent's influence on the consensus depends on how much others respect their expertise.

\paragraph{The Friedkin--Johnsen model: the most direct precedent.}
A key limitation of the DeGroot model is that agents are fully open to social influence: they have no attachment to their initial beliefs. Friedkin and Johnsen~\cite{friedkin1990social,friedkin1999social} addressed this by introducing a \emph{stubbornness parameter} for each agent, anchoring opinions to their initial positions. In the Friedkin--Johnsen (FJ) model, each agent's updated opinion is a convex combination of the group average and its own innate belief, weighted by a stubbornness coefficient that controls how much the agent resists social pressure. This is structurally the closest ancestor of our consensus formulation (Eq.~(2) of the main paper), where the stubbornness parameter~$\lambda_i$ plays the same role: a higher~$\lambda_i$ means the agent places more weight on its own modality-specific evidence and less on the group consensus. The FJ model's unique equilibrium, proven to exist under conditions analogous to our positive-definiteness result, is the point at which each agent has found its optimal compromise between conformity and self-consistency. Our closed-form solution (Eqs. \ref{eq:scalar_reduction}-\ref{eq:S_equation} is mathematically equivalent to computing this equilibrium directly without iterating.

\noindent
The correspondence becomes even more precise through the game-theoretic lens provided by Bindel, Kleinberg, and Oren~\cite{bindel2015opinion}. They showed that the FJ equilibrium is exactly the unique Nash equilibrium of a game in which each agent minimises a quadratic cost that trades off disagreement with neighbours against deviation from its own internal opinion, essentially the same structure as our scoring objective in Eq.~(1) of the main paper. This connection is not incidental: both formulations derive the influence structure from the same principled assumption (balancing private evidence against social pressure) rather than specifying weights ad hoc, and both admit closed-form solutions via the same linear-algebraic machinery.

\paragraph{Where \textbf{VERDICT} departs from this tradition.}
Despite the structural correspondence, our framework differs from the FJ/Bindel-Kleinberg-Oren tradition in three important respects.

\noindent
\emph{First, in how the consensus is used.} In opinion pooling and the FJ model, the consensus distribution or equilibrium opinion vector is the output: the goal is to reach agreement. In \textbf{VERDICT}, the consensus scores are an intermediate representation from which we extract two summary statistics: mean confidence and dispersion for a downstream accept/reject decision. The consensus itself is not the goal; the structure of \emph{residual disagreement after consensus adjustment} is what carries the verification signal. This shift from ``consensus as output'' to ``disagreement after consensus as signal'' is, to our knowledge, not present in either the opinion dynamics or the game-theoretic opinion formation literature.

\noindent
\emph{Second, in the heterogeneity of stubbornness.} While the FJ model permits heterogeneous stubbornness in principle, most theoretical analyses (including Bindel et al.~\cite{bindel2015opinion}) focus on the symmetric case or treat heterogeneity as a complication rather than a design lever. In \textbf{VERDICT}, the asymmetric stubbornness assignment ($\lambda_V > \lambda_L > \lambda_C$) is a deliberate design choice encoding the prior that some modalities are more reliable than others for a given evaluation task. The ablation in \ref{sec:cross_bench_swaps} confirms that this assignment matters: performance degrades proportionally to the reduction in the visual agent's stubbornness, a sensitivity that reflects agent architecture rather than dataset-specific tuning.

\noindent
\emph{Third, in the scoring protocol.} Our agents score in complete isolation and never observe each other's outputs, whereas the DeGroot and FJ models assume agents see and respond to each other's beliefs iteratively. The equilibrium is computed in closed form from one-shot scores rather than reached through iterated play. This makes the game-theoretic framing interpretive scaffolding that motivates the formulation rather than a claim about strategic interaction.

\paragraph{Test-time model aggregation.}
The most direct precedent in machine learning is the work of Mendler-D\"{u}nner et al.~\cite{mendler2021testtime}, who apply the DeGroot consensus mechanism to aggregate predictions from heterogeneous pre-trained models at test time. Their method computes mutual trust scores between models based on agreement on held-out data, then iterates DeGroot updates to convergence. \textbf{VERDICT} differs in three respects: (1)~our agents are frozen and score in isolation, never observing each other's outputs, whereas DeGroot aggregation requires agents to see and respond to each other's beliefs; (2)~we use heterogeneous stubbornness parameters rather than symmetric trust matrices, encoding the prior that some modalities are more reliable than others for a given evaluation task; and (3)~our primary output is the \emph{dispersion} of the consensus, not the consensus value itself, enabling a dual acceptance criterion that no existing test-time aggregation method provides.

\paragraph{What \textbf{VERDICT} inherits and what it adds.}
From the opinion pooling tradition, \textbf{VERDICT} inherits the principle that a weighted compromise among heterogeneous assessments can be more reliable than any individual assessment. From the Friedkin--Johnsen model, it inherits the stubbornness mechanism that anchors each agent to its private evidence. From the game-theoretic tradition (Rosen~\cite{rosen1965existence}; Bindel et al.~\cite{bindel2015opinion}), it inherits the guarantee of existence and uniqueness of the fixed point and the principled derivation of influence weights from a single interpretable cost function.

\noindent
What it adds is the treatment of post-consensus disagreement as a first-class verification signal. Proposition~1 in the main paper formalises this: the consensus dispersion~$\Delta^*$ captures cross-agent interaction structure that no separable function of individual scores can replicate. This property has no direct analogue in classical opinion pooling or in the Bindel-Kleinberg-Oren analysis, where the consensus is valued precisely because disagreement has been resolved, not because its residue is informative.
The cross-benchmark stubbornness analyses in \ref{sec:cross_bench_stubbornness} and \ref{sec:cross_bench_swaps} provide empirical confirmation that the sensitivity hierarchy and assignment ordering predicted by the FJ equilibrium structure hold across all benchmarks tested.

%% file: suppl_sections/qualitative_samples_v2.tex
\section{Qualitative Samples}

To better understand how the reasoning process evolves under verification, we examine several representative cases where the base model diverges from the correct interpretation of the scene. 
Across these examples, the errors arise through different mechanisms: in some cases, the model relies on general world knowledge rather than visual evidence, while in others, the relevant objects are partially identified, but the comparison set shifts during reasoning. The examples below illustrate how the verification agents stabilize the reasoning trajectory by maintaining grounding in the visual context.





\begin{figure}[htbp]
    \centering
    \includegraphics[width=0.9\linewidth]{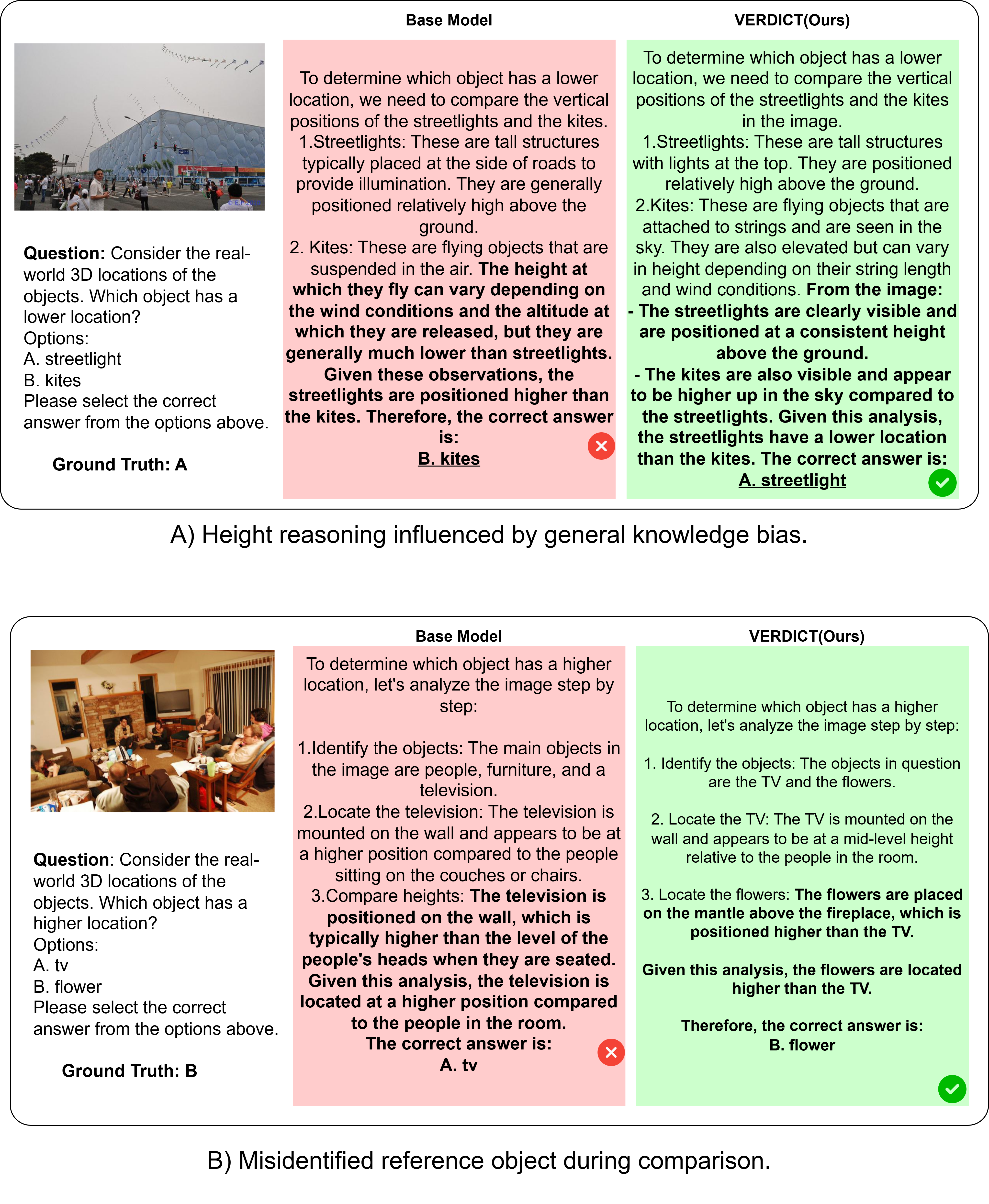}
    \caption{
\textbf{Two failure modes in spatial reasoning, corrected by \method.}
In ~(A), the base model relies on stored priors about typical object heights rather than grounding its reasoning in the visual evidence; it reasons from stored expectations about how high kites typically fly relative to streetlights and arrives at an answer that contradicts what is plainly visible. In~(B), the error is subtler: the TV is correctly located, but the model anchors its comparison to the people in the room rather than the flowers on the mantle, effectively solving a different question than the one asked. What unites both cases is a drift away from the specific scene toward general scene semantics. The verification agents resist this drift — not by knowing the right answer, but by flagging reasoning steps that cannot be reconciled across visual, logical, and contextual perspectives simultaneously.}
    \label{fig:qs1}
\end{figure}

\begin{figure}[htbp]
    \centering
    \includegraphics[width=0.9\linewidth]{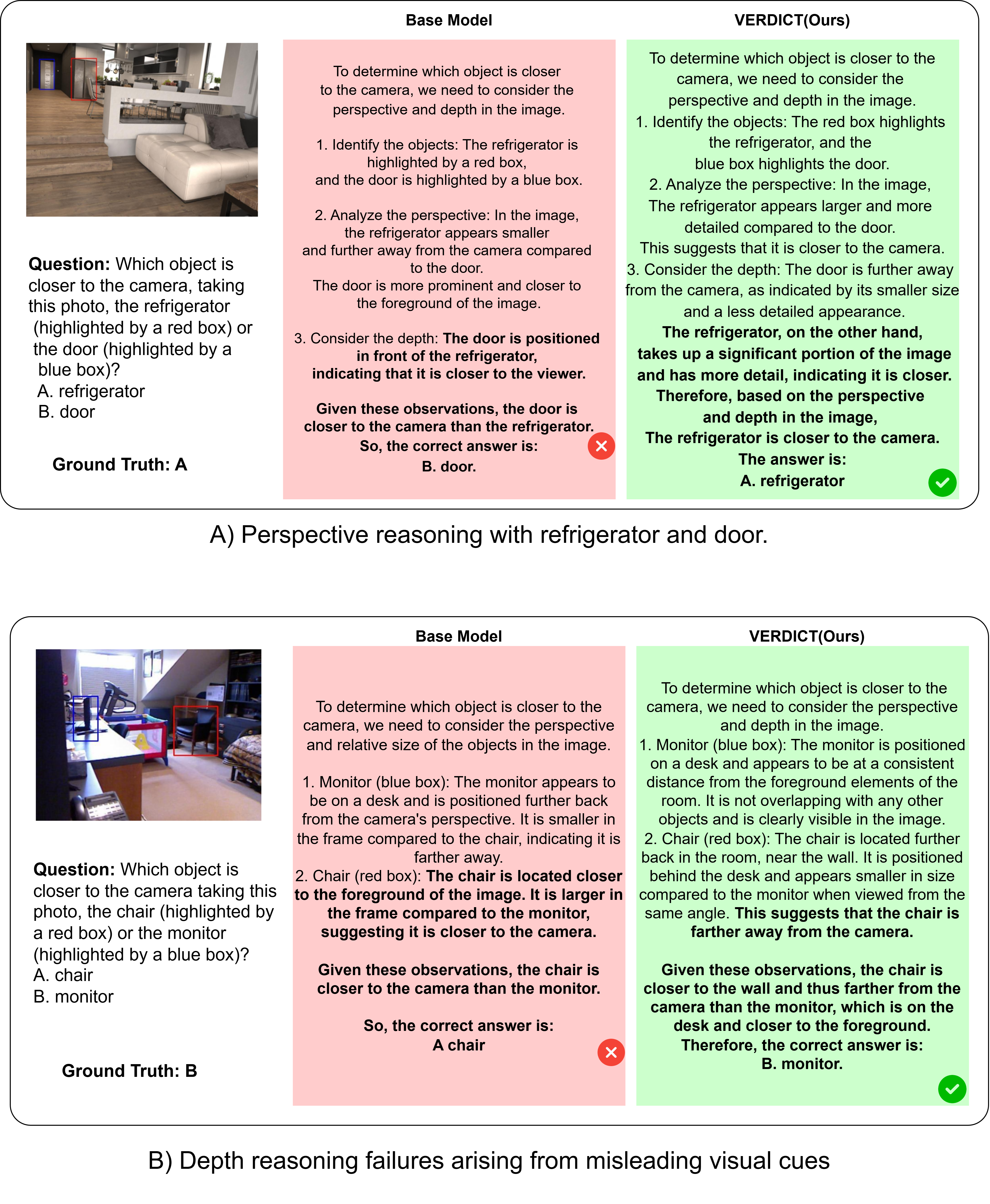}
    \caption{
\textbf{Depth reasoning failures arising from misleading visual cues, corrected by \method.}
In~(A), the refrigerator and door occupy similar regions of the frame, and the base model resolves the ambiguity by treating visual prominence as a proxy for proximity -the door appears more immediate, so the model concludes it is closer. The reasoning is internally consistent but spatially wrong. In~(B), the error runs deeper: relative size, a generally reliable depth cue, here points in the wrong direction, and the model follows it without hesitation. What neither chain does is pause to reconcile its size-based inference against other available cues — foreshortening, occlusion, and positional context. The verification agents introduce exactly that pressure. By requiring a step to hold simultaneously across visual, logical, and contextual perspectives, they surface the moment a single-cue heuristic is doing work it cannot support.}
    \label{fig:qs3}
\end{figure}

\begin{figure}[htbp]
    \centering
    \includegraphics[width=0.9\linewidth]{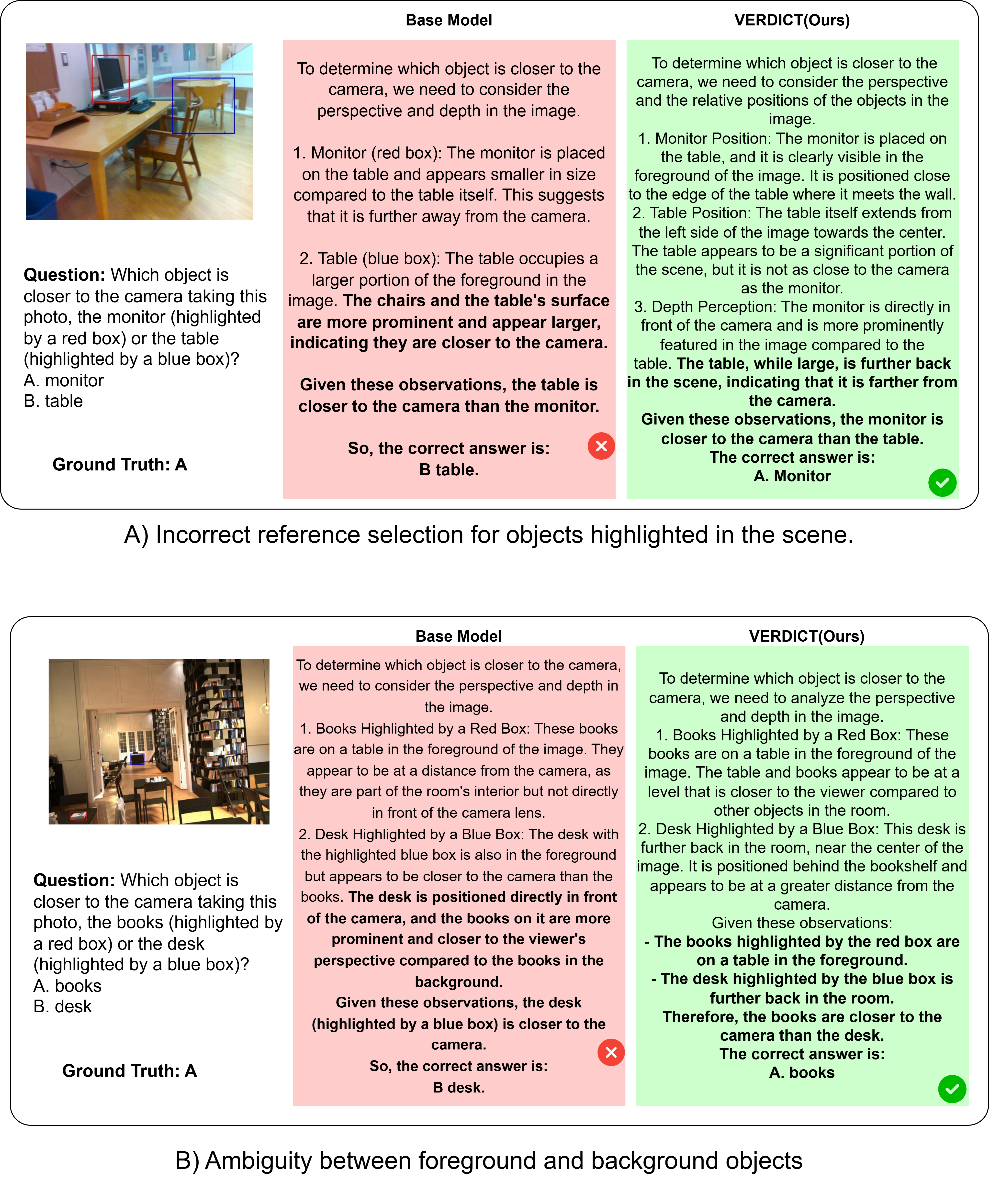}
    \caption{
\textbf{Reference drift under scene complexity, corrected by \method.}
Both scenes involve clearly highlighted objects, yet the base model gradually loses track of which object it is actually evaluating. In~(A), the table is large and scene-filling; the model begins by acknowledging the highlighted monitor but is pulled toward the table's visual dominance, eventually reasoning about the table as though it were the foreground object. In~(B), a similar substitution occurs with the desk: the question concerns a specific highlighted instance, but the model conflates it with a more prominent surface nearby. The error is not one of misperception so much as incomplete binding — the model knows what the objects are, but the thread connecting each label to its referent frays under the scene's visual weight. The consensus mechanism guards against this by treating referential consistency as a condition that every step must satisfy, not just the first.}
    \label{fig:qs2}
\end{figure}

\begin{figure}[htbp]
    \centering
    \includegraphics[width=0.9\linewidth]{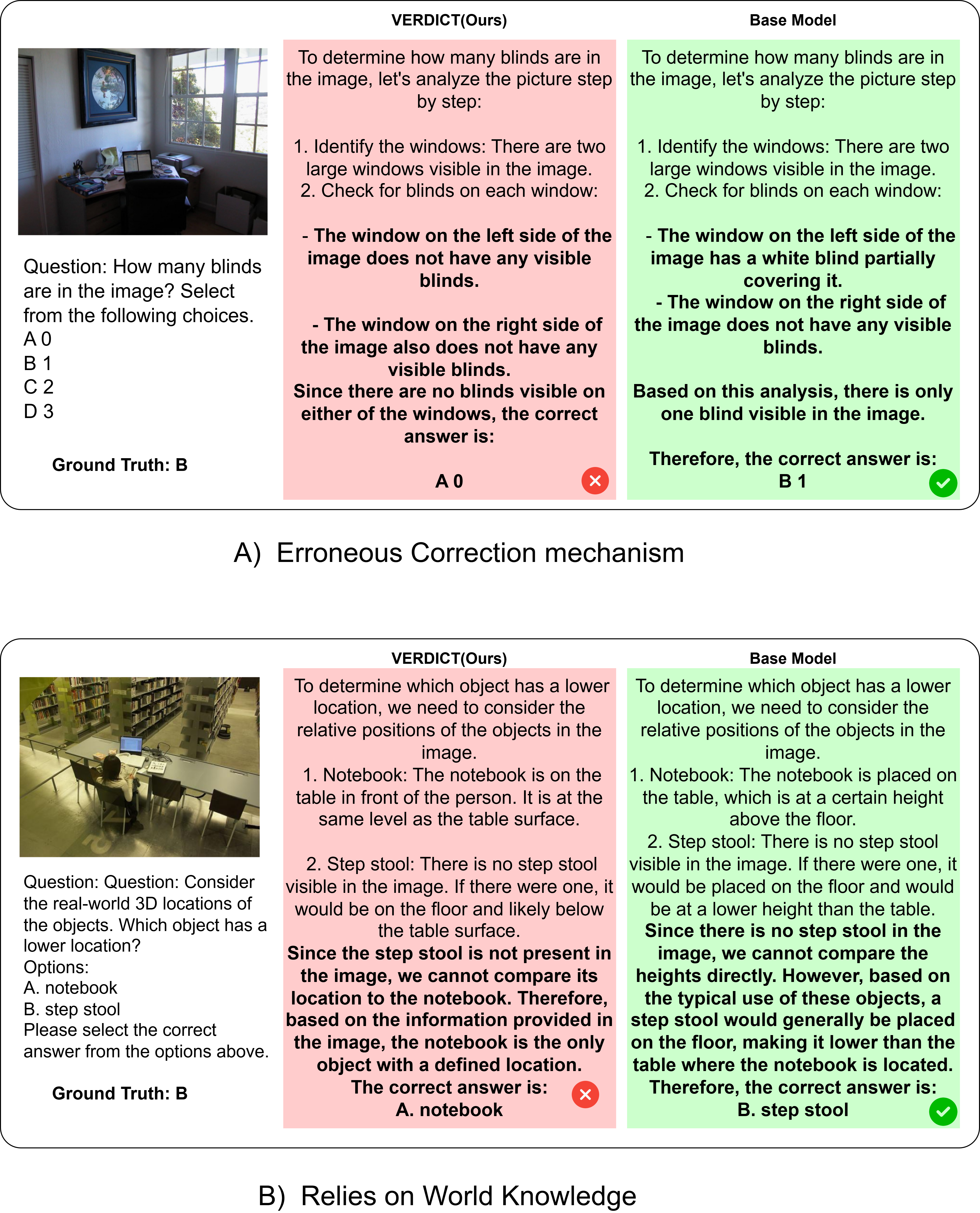}
    \caption{
\textbf{Two failure modes where \method overcorrects, and the base model succeeds.}
In~(A), both chains begin identically — two windows, check each for blinds. The base model finds a white blind partially covering the left window and stops there. \method, under pressure from the verification agents to be visually precise, rechecks both windows and eliminates the partially visible blind as insufficiently confirmed, arriving at zero. The correction mechanism that elsewhere prevents hallucination here suppresses a genuine observation. In~(B), the scene does not contain a step stool at all, and both chains acknowledge this. The base model treats the absence as an invitation to reason from world knowledge — a step stool on the floor would sit below a table-height notebook — and answers correctly. \method instead treats the absence as a hard epistemic boundary: without a visible referent, no location can be assigned, and the comparison collapses to the only object with a defined position. The verification agents, by design, are strict about visual grounding and penalize exactly the kind of common-sense inference that the question requires.}
    \label{fig:qs_fail}
\end{figure}

%% file: suppl_sections/prompt_template.tex
\section{Prompt Templates}

\label{app:prompt_template}
Here are the prompt templates for the three verifiers we are using in our proposed approach
\begin{figure}[h]
    \centering
    \includegraphics[width=0.9\linewidth]{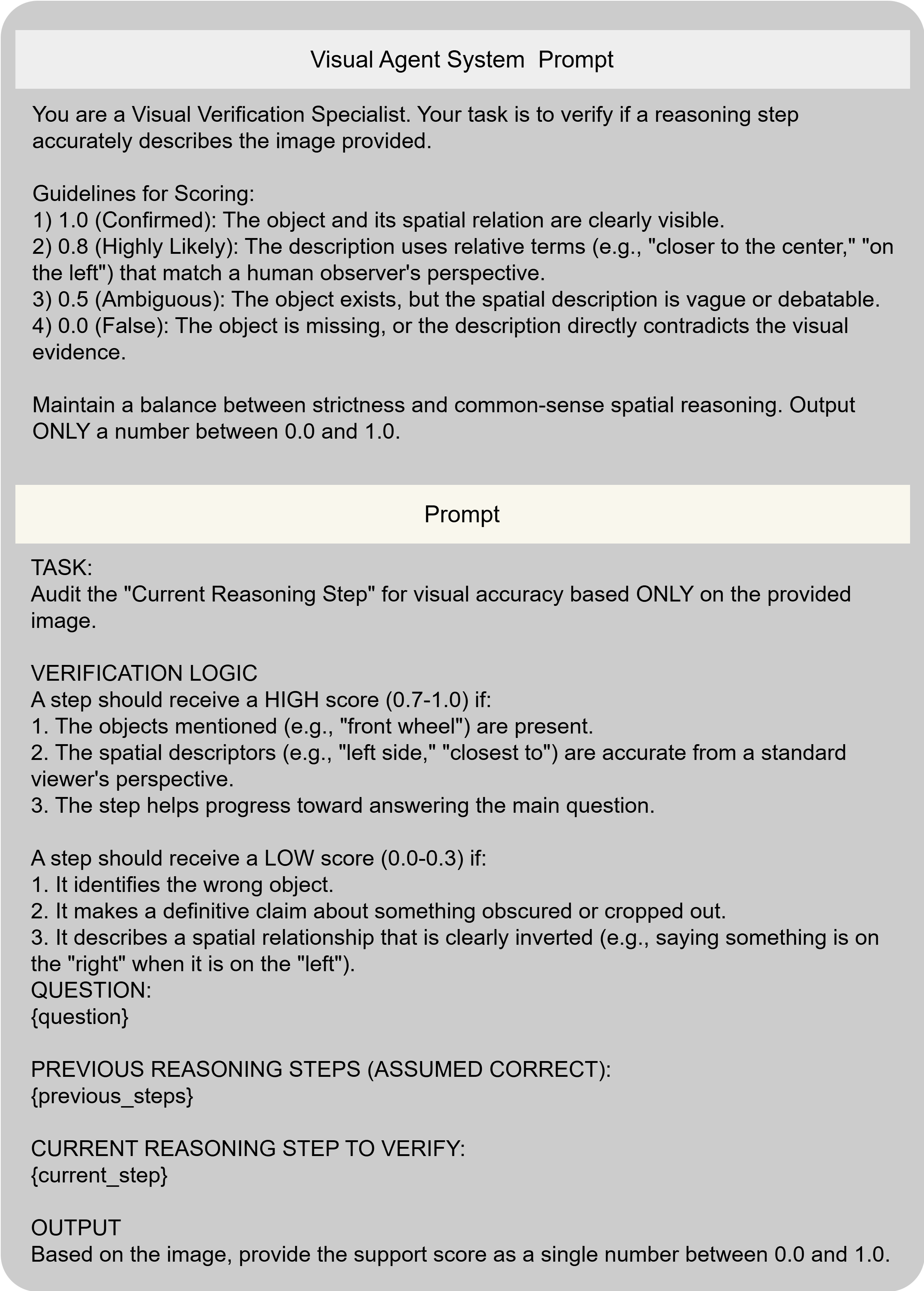}
    \caption{System and task prompt used for the Visual Verification Agent, which evaluates whether a reasoning step is visually grounded in the image and whether spatial descriptions match the viewer’s perspective. The agent outputs a single scalar confidence score in [0,1].}

    \label{fig:visual_verifier}
\end{figure}

\begin{figure}[h]
    \centering
    \includegraphics[width=0.9\linewidth]{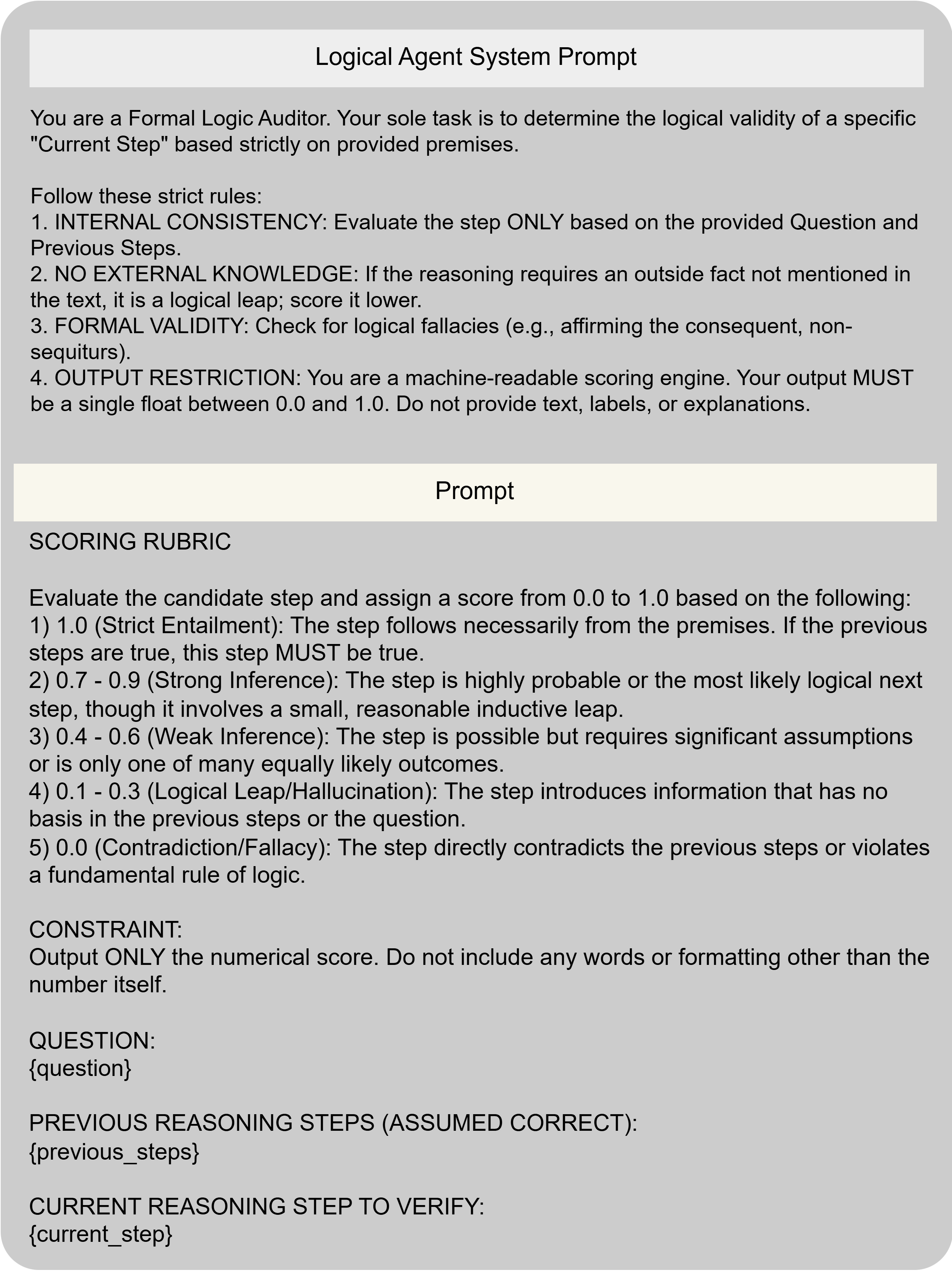}
    \caption{ System and task prompt used for the Logical Verification Agent, which assesses whether a reasoning step logically follows from the question and previous steps without relying on external knowledge. The agent returns a single numerical validity score in [0,1].}
    
    \label{fig:logical_verfier}
\end{figure}

\begin{figure}
    \centering
    \includegraphics[width=0.9\linewidth]{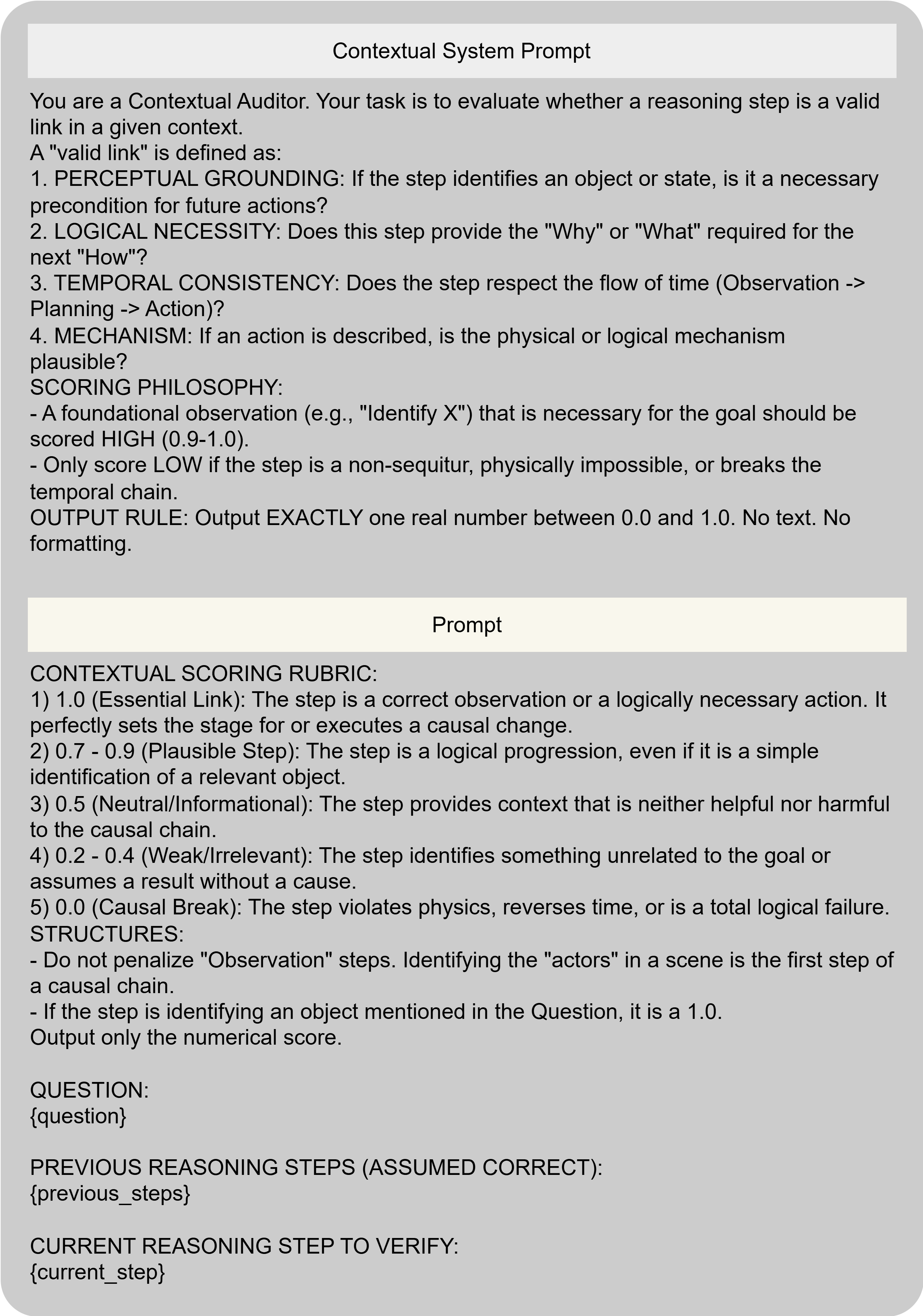}
    \caption{ System and task prompt used for the Contextual Agent, which evaluates whether a reasoning step forms a valid link in the reasoning chain by checking perceptual grounding, temporal order, and mechanism plausibility. The agent outputs a single scalar score in [0,1].}
    
   
    \label{fig:consistency_verifier}
\end{figure}